\documentclass[12pt]{iopart}

\usepackage{amssymb, hyperref, amsmath, color, soul, mathtools, algorithm, algpseudocode, float}

\usepackage{cleveref}

\usepackage{amsthm}

\newtheorem{proposition}{Proposition}
\newtheorem{theorem}{Theorem}[section]
\newtheorem{definition}{Definition}[section]

\newtheorem{claim}[theorem]{Claim}

\newcommand{\E}[2]{\mathbb{E}_{#1}\left[    #2   \right]}
\renewcommand{\O}[1]{O \left(  #1  \right)}
\newcommand{\knorm}[1]{\left | \left | #1 \right | \right |}
\newcommand{\floor}[1]{\lfloor #1 \rfloor}

\newcommand{\cD}{{\mathcal D}}

\newcommand{\cB}{{\mathcal B}}

\newcommand{\cY}{{\mathcal Y}}
\newcommand{\cX}{{\mathcal X}}

\newcommand{\cH}{{\mathcal H}}

\newcommand{\cW}{{\mathcal W}}

\DeclareMathOperator*{\argmin}{arg\,min}

\newcommand\bV{\boldsymbol V}

\newcommand\bx{\boldsymbol x}

\newcommand\bX{\boldsymbol X}

\newcommand\RR{\mathbb R}
\newcommand\EE{\mathbb E}

\newcommand\lgl{\langle}
\newcommand\rgl{\rangle}

\newcommand{\cZ}{\mathcal{Z}}
\newcommand{\bU}{{\mathbf{U}}}
\newcommand{\bM}{{\mathbf{M}}}
\newcommand{\bw}{{\mathbf{w}}}
\newcommand{\by}{{\mathbf{y}}}
\newcommand{\beeta}{{\boldsymbol{\beta}}}
\newcommand{\bnu}{{\boldsymbol{\nu}}}

\theoremstyle{plain}

\theoremstyle{plain}

\theoremstyle{plain}

\theoremstyle{remark}

\theoremstyle{definition}

\DeclarePairedDelimiter\abs{\lvert}{\rvert}%
\DeclarePairedDelimiter\norm{\lVert}{\rVert}%

\begin{document}

\title[Applying statistical learning theory for deep learning]{Applying statistical learning theory to deep learning}

\author{C\'edric Gerbelot$^{1}$, Avetik Karagulyan$^{2}$, Stefani Karp$^{3}$, Kavya Ravichandran$^{4}$, Menachem Stern$^{5}$ and Nathan Srebro$^{4}$}

\address{$^{1}$ Courant Institute of Mathematical Sciences, New York, NY 10012, USA \\
$^{2}$ King Abdullah University of Science and Technology, Thuwal 23955, Saudi Arabia \\
$^{3}$ Carnegie Mellon University, Pittsburgh, PA, and Google Research, NY, USA \\
$^{4}$ Toyota Technological Institute at Chicago, Chicago, Illinois 60637, USA \\
$^{5}$ Department of Physics \& Astronomy, University of Pennsylvania, Philadelphia, PA 19104-6396, USA}
\ead{cedric.gerbelot@cims.nyu.edu, avetik.karagulyan@kaust.edu.sa, shkarp@cs.cmu.edu, kavya@ttic.edu, nachi@sas.upenn.edu, nati@ttic.edu}
\vspace{10pt}
\begin{indented}
\item[]
\end{indented}

\begin{abstract}
    Although statistical learning theory provides a robust framework to understand supervised learning, many theoretical aspects of deep learning 
    remain unclear, in particular how different architectures may lead to inductive bias when trained using gradient based methods. The goal of these lectures is to provide an 
    overview of some of the main questions that arise when attempting to understand deep learning from a learning theory perspective. After a brief reminder on 
    statistical learning theory and stochastic optimization, we discuss implicit bias in the context of benign overfitting. We then move to a general description 
    of the mirror descent algorithm, showing how we may go back and forth between a parameter space and the corresponding function space for a given learning problem, as well as how the geometry of the learning problem may be represented by a metric tensor. Building on this framework, 
    we provide a detailed study of the implicit bias of gradient descent on linear diagonal networks for various regression tasks, showing how the loss function, scale of parameters at initialization and depth of the network may lead 
    to various forms of implicit bias, in particular transitioning between kernel and feature learning regimes.
\end{abstract}
\newpage
\tableofcontents
\newpage
\section{Lecture 1: Applying statistical learning theory to deep learning}

\subsection{Preamble}

The purpose of this first lecture is to provide a mathematical framework that allows 
to introduce the notion of inductive (or implicit) bias in a clear way. We will see that 
this notion is already present in the usual formulation of supervised learning, as it implicitly appears when preferring certain 
classes of model over others in empirical risk minimization procedures. While the statistical aspect of supervised 
learning can be understood by measuring the complexity of classes of predictors and introducing an associated notion of inductive bias, the computational aspect 
also requires an assumption of implementability. This is especially relevant in the context of deep learning, as the expressivity of neural networks 
can seemingly lead to the representation and learning of arbitrary functions. This will naturally lead us 
to the notion of computationally efficient learning rule, and to turn our attention to the implicit bias associated with the optimization aspect of 
supervised learning. 

\subsection{Inductive bias in supervised learning}

The main idea of supervised learning is to find a predictor, i.e. a mapping $h$ belonging to some function space $\mathcal{H}$, from inputs or instances $\mathcal{X}$ (e.g. images, sentences) to labels $\mathcal{Y}$ (e.g. classes) in order to predict the labels of new instances. In the simplest setting we'll just think of $y\in\{\pm 1\}$. To do so, we introduce a loss function $loss : \mathcal{Y} \times \mathcal{Y} \to \mathbb{R}_{+}$ which quantifies the error 
made by our predictor for a given labeled example $(x,y) \in \mathcal{X} \times \mathcal{Y}$. The goal is to find a predictor that has a small generalization loss $L(h)$ defined as the expected value of our loss function $loss$ with respect to a source distribution $\mathcal{D}$. Intuitively, this amounts to find $h$ such that our predictions tend to reproduce the hidden joint distribution. This is 
formalized in the following definition :

\begin{align} 
\mbox{\textul{Supervised learning}:\quad } &\mbox{find } h: \mathcal{X}\rightarrow \mathcal{Y} \mbox{ with small \textit{generalization error}, defined as},  \nonumber
\\
&L(h) = \mathbb{E} _{(x,y)\sim\mathcal{D}} [ loss(h(x);y) ] \nonumber
\end{align}

A central aspect of machine learning is to design this predictor $h$, not based on knowledge of the population $\mathcal{D}$, but rather based on an IID sample $\mathcal{S}$ (to which we have access to through an experiment, a dataset, ...) from that population. We attempt to find a good \textit{learning rule}, i.e. a mapping $A : \mathcal{S} \to \mathcal{H}$, that produces a predictor with small error for any population. 

\begin{align} 
	\mbox{\textul{Learning rule}: } &\mbox{(based on sample $S$)} \nonumber
	\\
 &A:\ S\rightarrow h\quad \mbox{(i.e. $A:(\mathcal{X}\times\mathcal{Y}) \rightarrow \mathcal{Y}^\mathcal{X}$)} \nonumber
\end{align}

Unfortunately, this is impossible. The `no free lunch' theorem , a terminology originally introduced in \cite{wolpert1997no}, tells us that small generalization error requires knowledge about the population. For any learning rule, there exists some distribution $\mathcal{D}$ (that is, some reality) for which the learning rule yields an expected error that is tantamount to randomly guessing the answer (e.g. $1/2$ for binary classification). More formally, for any learning rule $A$ based on an IID sample $\mathcal{S}$ of size $m$, there exists a distribution $\mathcal{D}$ such that there exists a predictor $h^{*}$ verifying $L(h^{*}) = 0$, but

$$
\mathbb{E}_{S\sim\mathcal{D}^m}[L(A(S))]\geq \frac{1}{2}-\frac{m}{2|\mathcal{X}|},
$$
where $\vert \mathcal{X} \vert$ designates the cardinality of the input space $\mathcal{X}$, which may be infinite.
This is true not only for independent $x,y$, but also if there exists a deterministic relation $y(x)$, so that a predictor does exist. The supposed improvement over $1/2$ is proportional to the size of the dataset $m/|\mathcal{X}|$ and is due to memorization of that dataset, and vanishes when the population size is large.
  
Thus, learning is impossible without assuming anything. This is where inductive bias becomes an essential part of learning. We assume that some realities (populations $\mathcal{D}$) are unlikely, and design the learning rule $A$ to work for the more likely realities, e.g. by preferring certain models $h(x)$ over others. More practically, we assume that the reality $\mathcal{D}$ has a certain property which ensures the learning rule $A$ will have a good generalization error. Typically, we assume that there exist models $h(x)$ in some class $\mathcal{H}$, or with low ``complexity'', denoted $c(h)$, such that it has low generalization error $L(h)$. An example are models where the output $y$ changes smoothly with the input $x$, where the complexity of the model can be captured by its total variation or an appropriate Sobolev norm. Another example is ridge regression, that prefers linear models, in which case the complexity will be measured with the norm of the weights of the predictor.

A flat inductive bias embodies the assumption that some realities are possible and others are not : $\exists {h^*} \in \mathcal{H}$ with low $L(h^*)$. If we make this assumption, we know what is the best learning rule for supervised learning, which is \textit{empirical risk minimization}:

$$
ERM_\mathcal{H} (S) = \hat{h} = \arg \min_{h\in\mathcal{H}} L_S(h),
$$

with $L_S$ the \emph{empirical} loss, or \textit{training} loss over our sample $S$ of size $m$:

$$
L_S(h) = \frac{1}{m}\sum_{i=1}^m loss(h(x_i);y_i).
$$

For this learning rule, we can guarantee an upper bound on the generalization error based on the capacity of the hypothesis class $\mathcal{H}$, which quantifies the complexity of our model. If the best model in our class is $h^*$, the error of the predictor achieved by the ERM learning rule is

$$
L(ERM_\mathcal{H}(S))\leq L(h^*) + \mathcal{R}_m(\mathcal{H}) \approx L(h^*) + \sqrt{\frac{\O{\text{capacity}(\mathcal{H})}}{m}}\quad (*)
$$
where the quantity $\mathcal{R}_{m}(\mathcal{H})$ represents the Rademacher complexity of the class $\mathcal{H}$ and will be studied more carefully in Lecture 3. For now, let us simply note that 
the error is larger than the best error in our class $\mathcal{H}$ by a term that scales with a measure of its capacity (colloquially, the `amount' of models in the class), divided by the number of samples. In other words, the number of training examples required to learn $h^*$ scales with the capacity of the class $\mathcal{H}$.

Let us study this notion of capacity more carefully. For binary classification, the capacity is the Vapnik-Chervonenkis (VC) \cite{vapnik1974theory,vapnik1999nature} dimension of the class, $capacity(\mathcal{H}) = VCdim(\mathcal{H})$. The VC dimension is the largest number of points $D$ that can be labeled by models $h\in \mathcal{H}$ in every possible way. Thus it quantifies the ability of the models in our hypothesis class $\mathcal{H}$ to fit an arbitrarily labeled dataset. This is quite natural. A model class with high VC dimension (i.e. that contains predictors allowing \textit{any} possible labelling of the set), does not have any inductive bias, so that learning is impossible (no free lunch). Learning becomes possible when the model class can be falsified, and the number of samples needed for learning is the number required to falsify this assumption on the model class $\mathcal{H}$. For linear classifiers over $d$ features, $VCdim(\mathcal{H})=d$. In fact, if the model class $\mathcal{H}$ can be parameterized with $d$ parameters, the VC dimension is usually $VCdim(\mathcal{H})=\tilde{O}(d)$. It is always true that the VC dimension is bounded by the logarithm of the cardinality of $\mathcal{H}$:

$$
VCdim(\mathcal{H}) \leq \log{|\mathcal{H}|}\leq \# \text{bits} = \# \text{params} \cdot \frac{\# \text{bits}}{\# \text{params}}.
$$

Thus we expect that if we encode the parameters of our model with a fixed number of bits, the VC dimension of the model scales with the number of parameters. Another way to produce model classes with finite capacity is to employ regularizers in our learning rule, explicitly penalizing models with high complexity. For example, it can be shown that for linear predictors with norm $\vert\vert w \vert\vert_2 \leq B$ (with logistic loss and normalized data), $capacity(\mathcal{H})=B^2$. More detail on VC dimension 
and related notions can be found in e.g. \cite{vapnik1999nature,bousquet2003introduction}.

Looking back at $(*)$, we see that learning, and machine learning in particular, requires model (hypothesis) classes $\mathcal{H}$ that are expressive enough to approximate reality well (contain $h^*$ with low generalization error), but also have a small enough capacity to allow for good generalization. The approximation error is defined by the error of the best model in our class $h^*$, and the estimation error is the excess over it, as the learning rule can choose a different, worse model given the empirical data.

Usually however, our learning protocols do not represent a flat inductive bias over some model class. We often think in terms of a complexity measure $c:\ \mathcal{Y}^\mathcal{X} \rightarrow [0,\infty]$, which is formally any ordering of predictors $h$. Some measures of complexity include for instance the degree of polynomials is our model class, the cardinality of active weights of the predictor (i.e. an assumption of sparsity) or a low given norm of the weights of our predictor $\vert \vert w \vert \vert$. The associated inductive bias is that $\exists h^*$ with low complexity $c(h^*)$ and low error $L(h^*)$. This inductive bias suggests another learning rule, \textit{structural risk minimization}:

$$
SRM_\mathcal{H} (S) = \arg \min_{h\in\mathcal{H}} L_S(h), c(h).
$$

This learning rule attempts to minimize two functions, which naturally introduces a trade-off between them. At best, the learning rule achieves a predictor that sits on the Pareto frontier that trades off generalization error and complexity. Any predictor on that line cannot improve either the error or complexity without worsening the other. It is possible to achieve this frontier by considering a regularization path, i.e. minimizing $L_S(h)+\lambda c(h)$, and varying the regularization amplitude $\lambda$ in the range $[0,\infty]$. Equivalently, one can attempt to minimize the error $L_S$ such that $c(h)\leq B$ (in which case the aforementioned $\lambda$ can be understood as a Lagrange multiplier). Note that this learning rule retrieves a multitude of candidates along the Pareto frontier. We can choose the best of them, for instance, according to their performance on a cross validation sample.

For the SRM learning rule, we get a similar guarantee to $(*)$ on the generalization error, although our hypothesis class $\mathcal{H}$ is now reduced to 
functions with low complexity, and in particular for the optimal predictor $h^{*}$, we have the bound

$$
L(SRM_\mathcal{H}(S))\leq L(h^*) + \sqrt{\frac{\O{capacity\left(\mathcal{H}_{c(h^*)}\right)}}{m}}.
$$

Another way to think of this complexity measure as opposed to a flat hypothesis class is that it gives rise to a hierarchy of hypothesis classes which are sub-level sets of the complexity measure. The guarantee for SRM gives us an upper bound on the loss based on the best model in the class $h^*$, and its complexity measure in that class; better predictors are obtained if $h^*$ lives in a small level set of the complexity measure. A good model class $\mathcal{H}$ in this approach not only contains a model that approximates reality $h^*$, but does so at lower level-set of a complexity measure.

\subsection{Inductive bias in deep learning}

\begin{figure}
\centering
\includegraphics[width=0.7\linewidth]{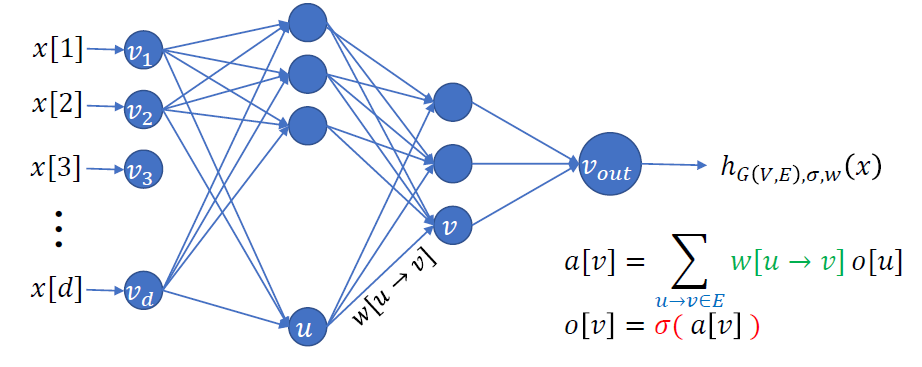}
\caption{Feed-forward neural network.}
\label{fig:FFNN}
\end{figure}

Deep learning is learning with a particular inductive bias, a flat hypothesis class in the form of a feed-forward neural network. For more information on the history of neural networks and their applications, the interested reader may consult the textbook \cite{lecun2015deep}. A feed-forward neural network (Fig.~\ref{fig:FFNN}) is described by a directed graph $G(V,E)$ with nodes (neurons) indexed by vertices $V$. These nodes are subdivided into three types: 
\begin{itemize}
\item \textit{Input nodes} $v_1,...,v_d\in V$ with no incoming edges, whose is $o[v_i]=x_i$.
\item \textit{Output node} $v_{out}\in V$, whose output is the model function $h_w(x)=o[v_{out}]$.
\item \textit{Hidden nodes} are all the rest of the nodes, which receive inputs from incoming edges (from a previous layer) and produce outputs to outgoing edges (to the next layer).
\end{itemize}

The network also has an \textit{activation function} $\sigma:\mathbb{R}\rightarrow\mathbb{R}$ that describes the non-linearity of the neural network; a popular choice is the rectified linear unit (ReLU), $\sigma_{ReLU}(z)=[z]_+ = \max(z,0)$. Finally, the edges of the network, each connecting $2$ nodes $u,v$, are called weights $w[u\rightarrow v]$ for each edge $u\rightarrow v \in E$. A choice of architecture, weights and activation function uniquely describe a predictor function $h_w(x)$. These models were historically developed by McCulloch and Pitts to describe logical calculus related to nervous activity \cite{mcculloch1943logical}. They were able to show that these models can perform complex computations, with the complexity 
measure directly related to the cardinality of the net.

In deep learning, we fix the architecture and activation function $\sigma$, and learn the weights from data. Thus the model class is given by $\mathcal{H}_{G(V,E),\sigma}=\{f_w(x) = $outputs of a network with weights $w\}$. We want to understand the capacity of these models, as well as their \textit{expressivity} (how well they represent reality). 

As noted before, the capacity is roughly given by the number of learned parameters, which here is the number of edges $\vert E \vert$. The VC dimension of these networks with threshold activation (ReLu) is $VCdim(\mathcal{H}_{G(V,E),sign})=\O{\vert E\vert \log \vert E\vert}$. However, for other activation functions, it is actually possible to get a capacity which is much higher than the number of parameters. For example, the VC dimension of these networks with sine activation is infinite, even for a single hidden node \cite{gaynier1995sinusoidal}. See also the more recent \cite{bartlett2019nearly} for complexity bounds on neural networks with piecewise lienar activations. We do not use these kind of activation functions. More useful activation functions are, for example the sigmoid function $\sigma(z)=(1+\exp(z))^{-1}$, whose VC dimensions is bounded by $VCdim(\mathcal{H}_{G(V,E),sigmoid})\leq\O{\vert E\vert^4}$, or ReLU function, for which $VCdim(\mathcal{H}_{G(V,E),ReLU})\leq\O{\vert E\vert \log \vert E\vert l}$, with $l$ the network depth. One can limit the capacity by discretizing the weights, e.g. if $w\in[-B,...,B]$, $VCdim\leq 2\vert E\vert B$. As we've seen, the fact that these network models can have a large capacity is not necessarily good, because capacity comes at the expense of inductive bias. 

What about expressivity? Feed-forward neural networks can represent any logical gate, this can be seen as a consequence of \cite{mcculloch1943logical}, and thus any function over $\mathcal{X}=\{\pm 1\}^d$ (as proved by Turing \cite{turing1939systems}, see also \cite{furst1984parity}). Define the class $CIRCUIT_n [depth,size]$ as all functions $f:\{\pm 1\}^n\rightarrow \{0,1\}$ that can be implemented with at most $size$ AND, OR and NOT gates, and longest path from input to output at most $depth$. We know that circuits can represent any function, see e.g. \cite{arora2009computational}, but only if we are allowed to select an appropriate gate architecture. In neural networks, we keep the architecture fixed (number of nodes, activations, edges) and only vary the weights. 

\begin{claim} 
A neural network with fixed architecture can learn the function of any circuit:
$$
CIRCUIT_n [depth,size] \subseteq \mathcal{H}_{G(V,E), l=depth, k=size, \sigma=sign}
$$

Where we use a fully connected neural net with $l=depth$ layers and $k=size$ nodes in every layer.
\end{claim}

This can be done easily, if we choose the weights of each edge to be $\pm 1$ if the edge is connected in the circuit (with /o without a NOT gate in between), $0$ otherwise. The bias terms are chosen as $fan\_in-1$ for AND gates, $1-fan\_in$ for OR gates, for a given value of $fan\_in \in [0,1]$. The weights essentially describe which wires exist in the circuit. Thus neural networks can represent any binary function.

More generally, we have a \textit{universal representation theorem} \cite{cybenko1989approximation,hornik1989multilayer,barron1993universal,pinkus1999approximation}: Any continuous function $f:[0,1]^d\rightarrow \mathbb{R}$ can be approximated to within $\epsilon$ by a feed-forward network with sigmoidal (or almost any other) activation function and a single hidden layer. This shows that as a model class, feed-forward neural nets are extremely expressive and can represent any reality. However, representing functions may require huge networks, e.g. with layer widths exponential in $d$. The relevant question is not what a network of arbitrary size can represent, but what small networks can represent.

Small networks can represent intersections of half-spaces (using single hidden layer, each neuron corresponds to a half-space and the output neuron performs AND) and unions of intersections of half-spaces (with two hidden layers: half-spaces$\rightarrow$OR$\rightarrow$AND). However, the main compelling reason to use them is \textit{feature learning}: linear predictors over (small number of) features, in turn represented as linear predictors over more basic features, that in turn are also represented as linear predictors. In essence, the network builds up a hierarchy of predictors that progressively manage more abstract features of the data. In the case of image data, this is typically presented as early layers learning simple features (edges in images), and later layers building up on the simpler features to represent higher-level, semantic ideas (cars, birds, etc.). 
Feature learning is at the heart of the success of deep neural networks, see e.g. \cite{bengio2013representation} for a description of the benefits of these architectures for representation learning.

Interestingly, a feed-forward neural network can represent any time $T$ computable function with network of size $\tilde{O}(T)$. This is true since anything computable in time $T$ is also computable by a logical circuit of size $\tilde{O}(T)$. This realization has broad implications for machine learning.

Machine learning is an engineering paradigm (of being lazy): use data and examples, instead of expert knowledge and tedious programming, to automatically create efficient systems that solve complex tasks. Therefore, we only care about a model (predictor) $h$ if it can be implemented efficiently. A good learned model only needs to compete with a programmer, producing results that are at least as good as a programmed model in a competitive (in terms of model evaluation) time. 

In this case we have a free lunch: the model class $TIME_T$ - all functions computable by at most time $T$, has capacity $\O{T}$, and hence learnable with $\O{T}$ parameters (e.g. using ERM). Even better: the model class $PROG_T$, all functions programmable up to length of code $T$, also has capacity $\O{T}$. This is relatively clear, because the length $T$ bounds the number of bits needed to represent all these functions. Unfortunately, ERM with respect to $PROG_T$ is uncomputable. Modified ERM for $TIME_T$ (truncating exec. time) can be computed, but is NP-complete. If $P=NP$, we can have universal learning (free lunch). If $P\ne NP$, i.e. if there exist one-way functions that are easy to compute but hard to learn, there is no poly-time algorithm for ERM over $TIME_T$. For a more detailed discussion of the 
representation of functions with circuits, classes of programmable functions and related complexity measures, see e.g. \cite{arora2009computational}.

We thus unfortunately conclude that the free lunch is only possible if $P=NP$. This realization gives rise to the computational no free lunch theorem: For every \textbf{computationally efficient} learning rule $A$, there is some reality $\mathcal{D}$ such that there is some computationally efficient (poly-time) $h^*$ with $L_S(h^*)=0$, but $\mathbb{E}[L(A(S))]\approx 1/2$. In other words, our learning rule $A$ can find an efficient $h^*$, but there are no guarantees on its generalization. 

This leads us to revise our requirements of inductive bias; we have to assume that not only that reality $\mathcal{D}$ supports good generalization, but also that the learning algorithm $A$ runs efficiently. The capacity of $\mathcal{H}$ or the complexity measure $h(c)$ are not sufficient inductive bias if ERM / SRM are not efficiently implementable, or if implementation does not always work (i.e. runs quickly but does not achieve ERM / SRM). Note that we switched from discussing learning rules (arbitrary mappings from sample to model), to talking about \textit{learning algorithms}, an actual implementable process that chooses such a model. 

Going back to neural networks, we completely understand them from a statistical perspective (in terms of capacity and expressivity, see the previous discussion on VC dimension for neural netowrks and universal approximation). The problem with them relates to computation; computing the ERM for feed-forward neural nets is a non-convex optimization problem, and no known algorithm is guaranteed to work. We know that learning in neural nets, even in the simplest cases ($2$-hidden units in one hidden layer), is NP-hard \cite{blum1988training}. Even if reality is well-approximated by a small neural net, and we tried optimizing a larger neural net (which has more degrees of freedom), optimization is easy but ERM is still NP-hard. Unfortunately, there is nothing one can do to efficiently solve this computational problem, which is essential to neural nets precisely because of their expressive power : this is the computational no free lunch. Even if a function is exactly representable with single hidden layer and $\Theta (\log d)$ nodes, even with no noise, and even if we take a much larger network or use any other method when learning: no poly-time algorithm can ensure better-than-chance prediction, see e.g. \cite{kearns1994cryptographic, daniely2014average}.

And nevertheless, deep learning does work! We have seen that from a statistical and computational perspectives, performing ERM on short programs (or short runtime programs) and learning with deep networks is equivalent. Both approaches are universal and approximate reality with reasonable sample complexity. They are both NP-hard, and provably hard to learn with any learning rule (subject to cryptographic assumptions). However there is no practical way to optimize over short programs, as e.g. there is no practical local search over programs. In contrast, deep neural nets are often easy to optimize; they are continuous models, amenable to local search (gradient descent, SGD), and enjoy much empirical success. In the worst case, deep learning is provably impossible, and yet, we are constantly reminded that deep learning is possible. There is a certain property of reality that makes feed-forward neural networks work, especially from the optimization viewpoint, and we have just started to scratch the surface of what that property is. However, we know for sure what it isn't: it is not the property that reality is well approximated by neural networks.

\subsection{Deep learning in practice}

As we have seen, deep neural networks can represent any function, and indeed have been shown to fit random data with perfect (training) accuracy \cite{zhang2021understanding}. However, when trained on real data, these networks do successfully generalize, even when over-parameterized.

We thus have a learning rule $A(S)$ that is able to achieve perfect training accuracy for any data set, even with random labels $L_S(A(S))=0$. On the other hand, it is able to generalize for real data $S\sim\mathcal{D}^m$ sampled from a reasonable reality $\mathcal{D}$, achieving low $L(A(S))$.

Perhaps we should not be surprised about this, as other learning algorithms do show similar behavior. A 1-Nearest Neighbor classifier, if realizable by a continuous $h^*$ (i.e. $L_S(h^*)$), then for an infinite sample size ($\vert S\vert\rightarrow \infty$), it is consistent with zero generalization error $L(1-\text{NN}(S))\rightarrow 0$. Similarly, a Hard Margin Support Vector Machine (SVM) with a Gaussian kernel (or some other universal kernel), or more generally, minimization of a norm for consistent solutions, also tend to generalize despite having vanishing training error: $\arg\min\vert\vert h\vert\vert_K$ such that $L_S(h)=0$. Let us consider a linear case where 
$$
w=\arg\min\vert\vert w\vert\vert_2\quad s.t.\quad \langle w,\phi(x_i)\rangle = y_i
$$

In this case, our SVM model does not have a flat inductive bias, but the norm of the weights $w$ adapt to the level of complexity inherent in the data. If reality is represented by a solution with small norm, then the learning rule will achieve a solution with low complexity measure and therefore generalize. However, if we try to fit random labels, we can only fit a model with a high norm (high complexity measure), and it will fail to generalize. We can always train SVMs with zero training error $L_S(h)=0$. If $\exists h^*$ with zero generalization error $L_\mathcal{D}(h^*)=0$, it will be achieved with sample complexity $\vert S\vert=\O{\vert\vert h \vert\vert_K^2}$. Another example for this generalization is found in Minimum Description Length (MDL): A program optimized for its length ($\arg\min\vert \text{program}\vert $) with $L_S(\text{program})=0$, is able to achieve a generalization error $L(MDL(S))\leq \O{\frac{\vert \text{program}\vert}{\vert S\vert}}$. That is, a short program only requires a sample complexity proportional to its length.

These examples and the ability of deep nets to generalize implies that the size of the network is not a good measure of model complexity. This is not a new idea; as it was already realized in the 1990s that kernel regression works for infinitely many features, because we rely on norm for complexity control (assuming the hypothesis class is a ball with fixed radius in the corresponding Reproducing Kernel Hilbert Space (RKHS)) rather than the dimensionality. It was shown in 1996 \cite{bartlett1996valid} that the complexity of a neural network is not controlled by the number of weights but by their magnitude. In fact, neural networks have many solutions for weights $w$ verifying $L_S(w)=0$, many of which have high generalization errors. These solutions tend to have high $w$-norms. However, the solutions found in practice for neural networks using gradient descent do generalize well, and tend to have small norms, even without explicitly regularizing for low norm solutions.

Where is this implicit regularization coming from? We will try to understand this in the simplest model possible - linear regression - in the next lecture. Consider an under-constrained least squares problem with $(n>m)$:

$$
\min_{w\in\mathbb{R}^n} \vert\vert Aw - b \vert\vert^2 \quad,\quad A\in \mathbb{R}^{m\times n}.
$$

In under-constrained cases there are many choices of $w$ for which the sum of squares vanishes. Imagine solving this problem with gradient descent, initialized at $w=0$. Gradient descent will definitely succeed, as this is a convex problem, and find a vanishing error solution, but which solution?

\begin{claim} 
Gradient descent (or SGD, conjugate gradient descent, BFGS) will converge to the least $\ell_{2}$-norm solution $\min_{Aw=b} \vert\vert w\vert\vert_2$. The proof follows from the iterates always being spanned by the rows of $A$ (more details will be given in the next lecture).
\end{claim}

While we did not explicitly design the algorithm to prefer solutions with small norm, it does in fact find the solution that minimizes this norm. This implicit regularization comes directly through the optimization process. In general we find that optimization algorithms minimize some norm or complexity measure, but which complexity measure?

\newpage
% \documentclass{article}
% \usepackage[margin=0.5in]{geometry}
% \usepackage{xcolor}
% \setlength{\parindent}{0pt}
% \usepackage{amsfonts}
% \usepackage{amssymb, hyperref, amsmath, color, soul, mathtools, algorithm, algpseudocode, float}

% \newcommand{\red}[1]{\textcolor{red}{#1}}
% \newcommand{\blue}[1]{\textcolor{blue}{#1}}

% \begin{document}

\section{Lecture 2 : Implicit bias and benign overfitting}

\subsection{Choosing the right complexity measure}

The discussion iniated in the previous lecture leads to the following main questions :
\begin{enumerate}
    \item How much of what we've seen so far fits within our classic understanding of statistical learning?
    \item What questions do we need to answer to put it within our standard understanding?
    \item And what goes \emph{beyond} our standard understanding?
\end{enumerate}  

One thing we've seen is that huge models (so large that they \textit{could} fit even random labels) can still generalize well. This does not contradict the classical understanding of supervised learning; as it's the same type of behavior obtained with something like a hard-margin SVM or a minimum-norm predictor. The reason for this is that what actually governs the ability of the model to generalize is not actually its \emph{capacity}. The real \textit{measure of complexity} is some kind of norm.

What is this norm? First of all, we are not using the word ``norm'' in an appropriate mathematical sense; what we \emph{really} mean is some \emph{measure of scale} that might not satisfy the rigorous definition of a norm. We can get implicit complexity control just from the optimization algorithm. For example, when we optimize an underdetermined least squares problem using gradient descent, we get the minimum-norm solution; that just comes from the optimization algorithm. So, we need to ask: what \textit{complexity measure} is being minimized, and how does gradient descent minimize it?

We ended the previous lecture by saying that if we change our optimization algorithm without changing the objective function, we're actually implicitly minimizing some \textit{other complexity measure}, which will change the inductive bias and thus the generalization properties of our model. As some examples, we can compare the test performance of SGD and Path-SGD as in \cite{path_sgd_2015} and SGD with the Adam optimizer as in \cite{marginal_value_adaptive_grad_2017}. In all cases, they reach the same final training error, but they have different final test performance values. In other words, they're reaching different global minima of the training objective. Thus what we are observing is that the inductive bias is determined by the bias of the optimization algorithm. In other words, if we have a training loss landscape with \textit{many global minima}, and we start optimizing on some ``hill'', \textit{different} optimization algorithms will move down the ``hill'' \textit{differently} and reach \textit{different} ``beaches'' (0 loss). 

An illustrative example of different optimization algorithms inducing different regularizers can be seen by studying gradient descent vs. coordinate descent. Gradient descent will get to the minimum $\ell_2$ norm, whereas coordinate descent will get to an approximately minimum $\ell_1$ norm solution. In a high-dimensional system, these two norms are extremely different; $\ell_1$ induces sparsity, and $\ell_2$ does not. This difference is significant, especially when we think about deep learning in terms of feature learning. \textit{Feature selection} (i.e., from a long list of features, select the relevant ones) is just sparsity, and $\ell_1$ regularization can achieve it. In deep learning, we want to do feature \emph{learning}, not just feature selection. But what \emph{is} finding new features? We can think of a continuous set of possible features, and we want to select good features from that infinite feature set. So as long as we're not too worried about infinities, there's not much difference between feature selection and feature learning. This can be quantified by establishing how a givne algorithm implicitly minimized an $\ell_1$ norm that will induce sparsity; as opposed to an $\ell_2$ norm that does not. Thus, this can be a huge difference, and all the inductive bias is coming from the algorithm here. In fact, here is the perspective on deep learning we shall take here: 

Deep networks can approximate any function. We somewhat dismissed our universal approximation results in the previous lecture because they require huge networks with seemingly unrealistic capacity. But perhaps we actually \emph{are} using networks that are large enough to capture all functions (since we are, at least on the available data, able to capture all functions). So maybe we \emph{are} essentially optimizing over the space of all functions. In that case, minimizing empirical error with respect to all functions doesn't make any sense; it's really easy to optimize the empirical error with respect to all functions, since we can just memorize the training examples and not do anything anywhere else in the domain. But in deep learning, we optimize over all functions \emph{with particular search dynamics}, and although we do get to a function that has 0 training error, we don't get to just \emph{any} 0 error solution. \emph{How} we optimize over the space of all functions determines which directions we like to take. Roughly speaking, we'll get to a 0-error solution that's ``close" to the initialization point in some sense with respect to our geometry.

\subsection{Examples}

With this perspective in mind, let us now discuss a few examples in which we are seemingly optimizing over the entire function class of a given problem but where the choice of optimization dynamics 
implicitly imposes a form of regularity on the solution, leading to good generalization.

\subsubsection{Matrix completion}

In matrix completion, we have some observations from a matrix. We want to complete the matrix and uncover the remaining entries. You should think of the matrix as having some structure (e.g., some low-rank structure). Formally, the matrix completion problem is:
\begin{equation}
\min_{X \in \mathbb{R}^{n \times n}}\|\text{observed}(X)-y\|_2^2.
\end{equation}

In some sense, it's very easy to solve this optimization problem. We can complete the observed entries and put 0 everywhere else, but it won't help us recover the unobserved entries. The problem is underdetermined. So what do we do? We can run gradient descent directly on $X$, or - alternatively - we can replace $X$ with $UV^\top$ ($U,V$ full rank, therefore no rank constraint on $X=UV^\top$) and run gradient descent on $U,V$. Figure~\ref{fig:matrix_completion} compares the results of these two procedures when the ground-truth matrix $X^*$ has rank 2. We also see how slight variations in gradient descent on $U,V$ change which solution we converge to. But the bigger effect is coming from the reparameterization (from $X$ to $UV^\top$).

So how can the implicit bias be understood here? In this case, we have a good understanding, though not complete. The initial proposal in \cite{gwbns_2017} was that gradient descent is converging to a low nuclear norm solution - i.e., nuclear norm is the relevant complexity measure that gradient descent is minimizing. We know that minimizing the nuclear norm can give you good generalization when you have an approximate low-rank matrix. Minimization of the nuclear norm is proved rigorously in some cases (e.g., under restricted isometry property (RIP) in \cite{overp_matrix_sensing_li_ma_zhang_2018}) and turns out to not always be the case (e.g., see counterexample in Example 5.9 in \cite{implicit_bias_gd_matrix_factorization_2021}). 

%\begin{equation}y = \mathcal{A}(X^*) + \mathcal{N}(0,10^{-3}), y_\text{test} = \mathcal{A}_\text{test}(X^*) + \mathcal{N}(0,10^{-3})\end{equation}
\begin{figure}
\begin{minipage}{0.5\textwidth}
\center
\includegraphics[scale = 0.19]{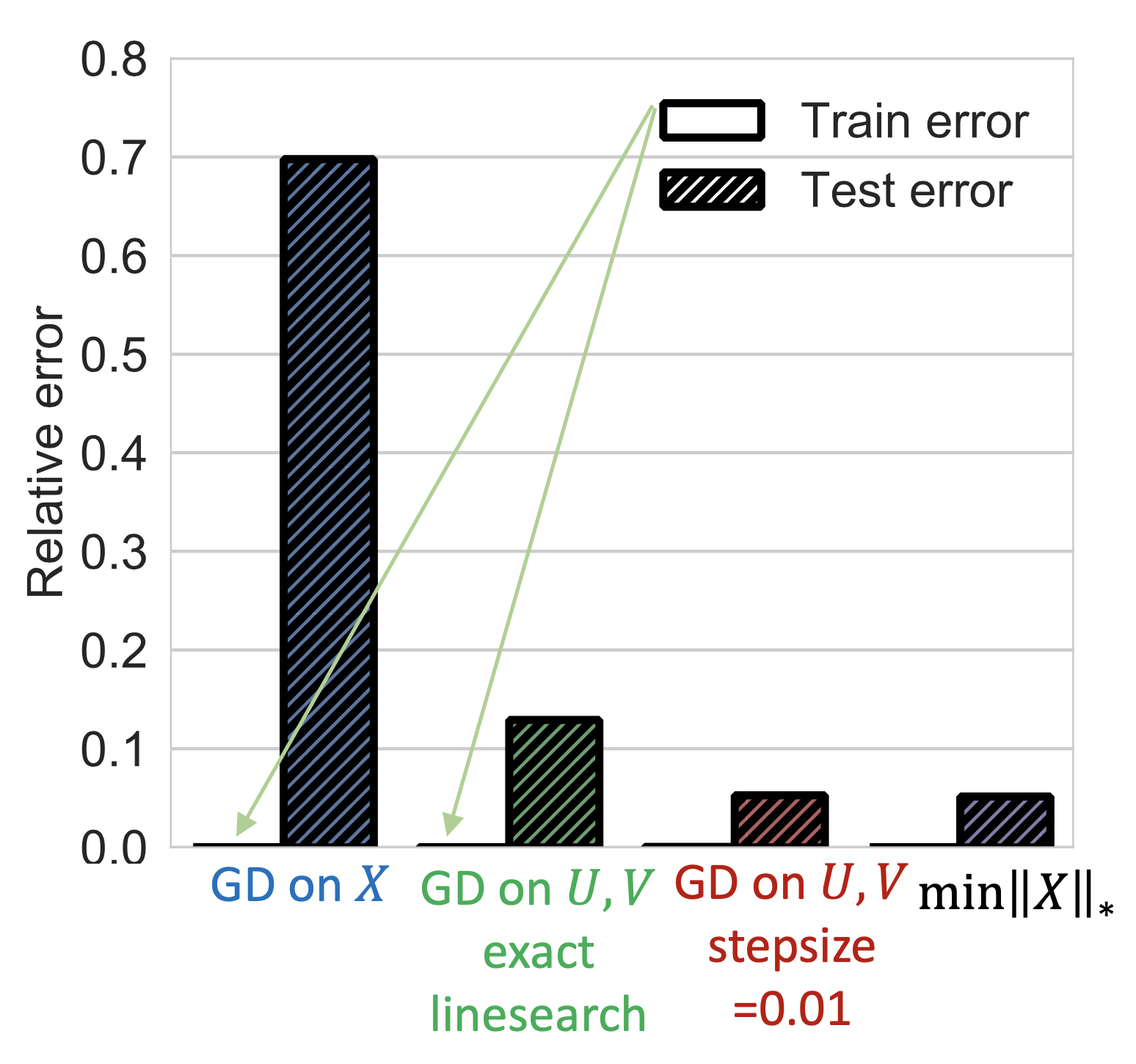}
\caption{Matrix completion. Relative error is the squared error compared to the squared error of the null predictor.}
\label{fig:matrix_completion}
\end{minipage}
\begin{minipage}{0.5\textwidth}
\center
\includegraphics[scale = 0.25]{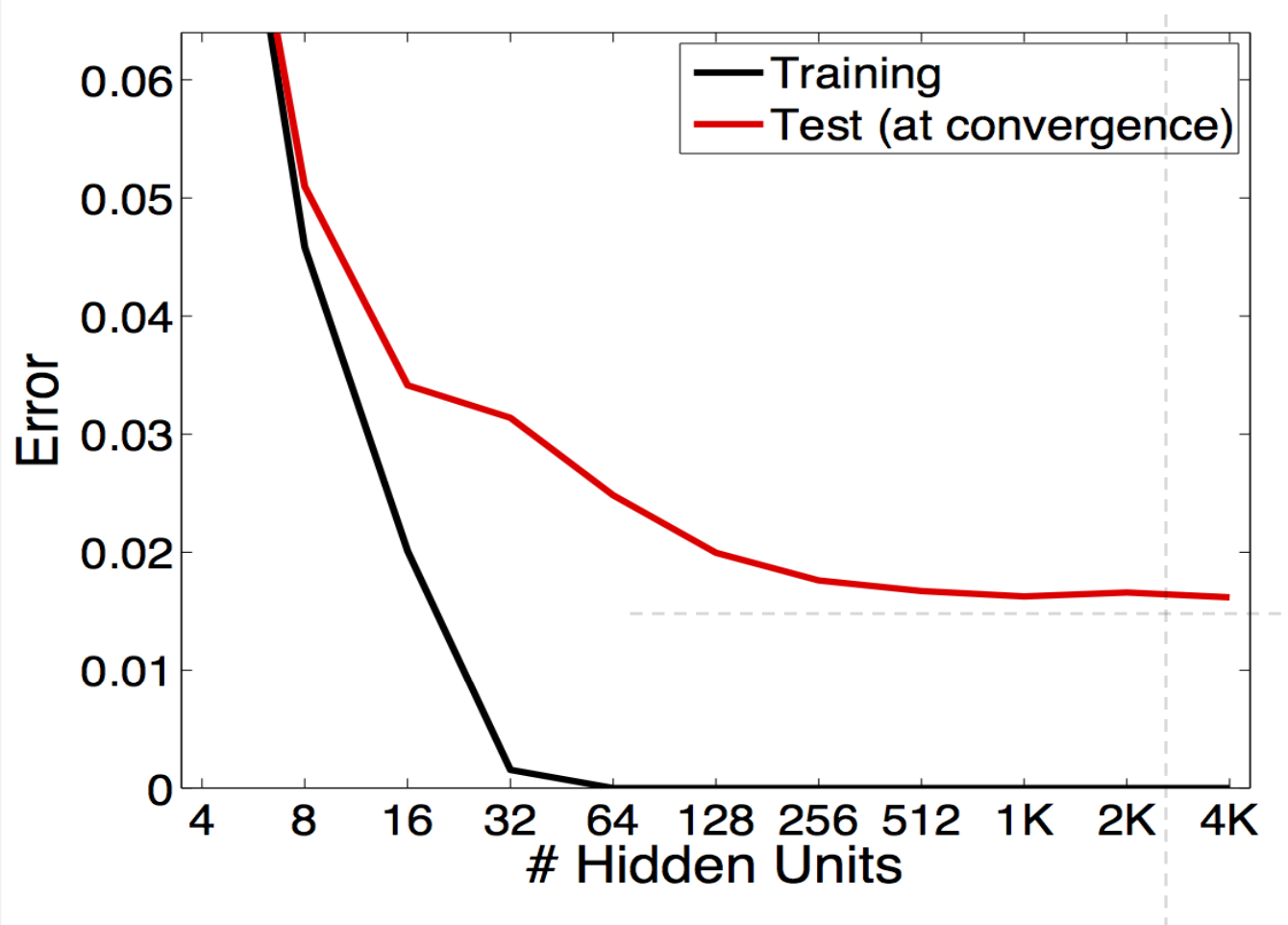}
\caption{Error vs. number of hidden units.}
\label{fig:hidden_units_plot}
\end{minipage}
\end{figure}

\subsubsection{Single overparameterized linear unit: $\mathbb{R}^d \to \mathbb{R}$}

If we train a single unit with gradient descent using the logistic (``cross entropy'') loss, we converge to the max-margin separator (the hard-margin SVM predictor) \cite{soudry2018implicit}, which involves an implicit $\ell_2$ norm regularization :
\begin{equation}
w(\infty) \propto \arg\min \|w\|_2 ~ \text{ s.t. } ~ \forall i ~~ y_i \langle w, x_i \rangle \ge 1.
\end{equation}
This holds regardless of initialization. We will go over this result in detail later on.

\subsubsection{Deep linear network: $\mathbb{R}^d \to \mathbb{R}$}
Now, let us consider what happens in a deeper network (with only linear activations). Let $w$ denote the weights of all the layers. Then our deep linear network implements the same linear mapping as above:
\begin{equation}
f_w(x) = \langle \beta_w, x \rangle.
\end{equation}

When we run gradient descent on $w$ (as opposed to $\beta$, as we did above), one might think that this reparameterization could affect the search geometry. However, in this case, the inductive bias is actually the same as above:
\begin{equation}
\beta_{w(\infty)} \propto \arg\min \|\beta\|_2 ~ \text{ s.t. } ~ \forall i ~~ y_i \langle \beta, x_i \rangle \ge 1.
\end{equation}

%\begin{equation}\min_w \mathcal{L}(f_w) \equiv \min_\beta \mathcal{L}(x \mapsto \langle \beta, x \rangle)\end{equation}

\subsubsection{Linear convolutional networks}

Things get more interesting in the linear convolutional case, though. Let's consider the following linear convolutional network:
\begin{equation}
h_l[d] = \sum_{k=0}^{D-1} w_l[k] h_{l-1}[d + k \operatorname{mod} D], \qquad h_\text{out} = \langle w_L, h_{L-1} \rangle,
\end{equation}
with $0 \leq d \leq D-1$, which is still just a reparameterization of our original linear function from $\mathbb{R}^d \to \mathbb{R}$. Now we can ask what happens when we train this model using gradient descent.

\paragraph{Single layer ($L=2$).} With a single hidden layer, training the weights with gradient descent implicitly minimizes the $\ell_1$ norm in the frequency domain:
\begin{equation}
\arg\min \|\operatorname{DFT}(\beta)\|_1 ~ \text{ s.t. } ~ \forall i ~~ y_i \langle \beta, x_i \rangle \ge 1,
\end{equation}
where DFT denotes the discrete Fourier transform. In other words, we obtain sparsity in the frequency domain.

\paragraph{Multiple layers.} With $L$ layers, training the weights with gradient descent converges to a critical point of 
\begin{equation}
\|\operatorname{DFT}(\beta)\|_{2/L} ~ \text{ s.t. } ~ \forall i ~~ y_i \langle \beta, x_i \rangle \ge 1,
\end{equation}
where $\|\cdot\|_{2/L}$ denotes the $2/L$ quasinorm. It is not technically a norm, but it \emph{is} formally defined, and it's even more sparsity-inducing than $\ell_1$. Thus, increasing the depth induces more and more sparsity in the frequency domain. See \cite{implicit_bias_gd_conv_2018} for more details.

\subsubsection{Infinite-width ReLU network}

Now let us look at all functions (not just linear) from $\mathbb{R}^d \to \mathbb{R}$. In order to represent them with a neural network, we have to introduce nonlinearities (e.g., ReLU). If we let a single-hidden-layer ReLU network be wide enough, we can approximate all functions. So let us learn using infinite-width ReLU networks.

\paragraph{Functions $h$ from $\mathbb{R} \to \mathbb{R}$.}

Gradient descent on the weights implicitly minimizes
\begin{equation}
\max\left( \int |h''|dx, |h'(-\infty) + h'(+\infty)| \right).
\end{equation}
This would be a very sensible penalty to choose, since it is kind of a smoothness-inducing penalty. The interesting thing is that we didn't explicitly choose it; it appeared solely from gradient descent on this parameterization.

\paragraph{Functions $h$ from $\mathbb{R}^d \to \mathbb{R}$.}

Gradient descent on the weights implicitly minimizes
\begin{equation}
\int | \partial_b^{d+1} \operatorname{Radon}(h) |,
\end{equation}
where $Radon(h)$ designates the Radon transform of $h$ with parameter $b$.
Once again, we obtain a form of sensible smoothness penalty. This result is rigorous for logistic loss (doesn't depend on initialization or learning rate). For squared loss, we don't know how to analyze it exactly, although we expect something similar. See \cite{inf_width_in_function_space_2019}, \cite{function_space_bounded_norm_multivariate_2020}, \cite{implicit_bias_gd_wide_two_layer_2020} for more details.

\subsubsection{Takeaways from these examples}

\paragraph{Main contributors to implicit bias} The game here is that we want to understand what happens in the space of functions. The inductive bias in parameter space is relatively simple ($\ell_2$ or something similar, often). But what we really care about is what happens in function space, which can be very rich. A large part of the optimization problem is dictated by the architecture. The classical view is that the architecture is important because it limits what functions may be obtained; however, that's not the case here. The architecture is important because it determines the mapping from parameter space to function space and is the \emph{biggest} contributor to the geometry by which we are searching in function space.

The next most significant thing that affects the inductive bias is the geometry of the local search in parameter space (e.g., $\ell_2$ vs. $\ell_1$ in parameter space).

And the least significant thing (though still relevant) that affects the inductive bias is the set of optimization choices (e.g., initialization, batch size, step size, etc.).

\paragraph{Does gradient descent always minimize the $\ell_2$ norm in parameter space?}

In all of these examples, we can get the same thing as gradient descent's implicit regularization using a minimum $\ell_2$ norm on the weights (subject to fitting the data). In all of these examples, the complexity control in parameter space is very simple (it is just $\ell_2$), and everything is coming from the parameterization. So is all this discussion of the implicit bias of gradient descent just $\ell_2$ in parameter space, with everything coming just from the parameterization?

The answer is both yes and no. It is possible to obtain some asymptotic results showing that, for some restricted class of models with the logistic loss, everything boils down to an $\ell_2$ norm. However, we'll also see that in many cases, this is not true, and we'll obtain something very different (e.g., under squared loss, or even under logistic loss non-asymptotically).

\subsection{Beyond mitigating overfitting : better generaliztion through reparameterization} %44:30

We now understand that, in Figure~\ref{fig:hidden_units_plot}, we have complexity control coming from the algorithm, and this is what stops the model from overfitting. The main question this led us to ask is: what is this complexity control? (which we just studied, in several examples.)

But there is \emph{another} thing that is going on here. What we have studied so far explains why we might be able to generalize well even at a large number of hidden units (i.e., we have complexity control coming from the algorithm, even though we're optimizing in the space of all functions). However, we haven't yet explained why the red curve in Figure \ref{fig:hidden_units_plot} (test error) actually goes \emph{down} as the number of hidden units increases. Recall, the $x$-axis is not complexity (complexity is something else - some norm-based complexity, probably).

\paragraph{Gaussian kernel.}

We can actually see similar behavior even with kernel methods. Let us consider the Gaussian kernel, which corresponds to an infinite-dimensional feature space, the RKHS corresponding to this kernel, denoted $\mathcal{H}$. Let us think of what happens if we use a finite approximation to the Gaussian kernel. Concretely, we have the Gaussian kernel $\langle \phi_\infty(x), \phi_\infty(x')\rangle = e^{-\|x-x'\|^2}$ and the finite-dimensional random feature mapping \[\phi_d(x)[i] = \frac{1}{\sqrt{d}}\cos(\langle \omega_i,x\rangle+\theta_i)\] 
approximating $\mathcal{H}$ \cite{rahimi2007random}. According to Bochner's theorem, see e.g. \cite{berg1984harmonic}, the Fourier transform $p(\omega)$ of any translation invariant kernel 
is a probability distribution. The features $\phi_{d}(x)[i]$ then correspond to the discretization of the representation of the Gaussian kernel as the inverse Fourier 
transform of $p(\omega)$. Thus the parameters $\omega_{i},\theta_{i}$ are not learned but drawn randomly. We refer the reader to \cite{rahimi2007random} and references therein for details of implementation 
of this method and, to e.g. \cite{rudi2017generalization} for related theoretical guarantees. 
For a given empirical loss $L_{S}$ based on an i.i.d. sample $S$, the algorithm returns 
\begin{align*}
&A(S) = \arg\min \|w\|_{\mathcal{H}}\\
&\mbox{s.t.} \quad L_S\left(x \mapsto \langle w, \phi_d(x)\rangle\right) = 0\\
&\mbox{i.e.} \quad \forall (x_i,y_i) \in S, \quad y_i = \langle w, \phi_d(x_i)\rangle.
\end{align*}
As $d \to \infty$, we approach the Gaussian kernel. Once we have more features than data points, we can already get 0 training error. But we are not doing it with the Gaussian kernel yet; we are doing it with an approximation of the Gaussian kernel obtained from the finite-dimensional features. If the RKHS norm induced by the Gaussian kernel is our ``correct'' complexity measure, then as $d \to \infty$, we are approximating it better and better. So we are minimizing a complexity measure that's a better and better approximation of the complexity measure we want. So as the dimensionality increases, the test error improves.
\paragraph{Matrix completion.} We can also see this in matrix completion, a finite dimensional problem. Suppose $X$ is $n \times n$, we observe $nk$ entries, and we parameterize $X$ as $UV^\top$, where $U,V \in \mathbb{R}^{n \times d}$ (thereby adding a rank-$d$ constraint). As above, $d$ captures the quality of our approximation. Suppose the ``right'' complexity measure is nuclear norm. Letting $\hat{L}$ be the empirical loss, we have two different regimes:\\
- If $d < k$, our algorithm returns $\arg\min \hat{L}(X)$ s.t. $\operatorname{rank}(X) \le d$.\\
- If $d > k$, our algorithm returns $\arg\min \|X\|_*$ s.t. $\hat{L}(X) = 0, ~ \operatorname{rank}(X) \le d$.\\
So as we increase the rank constraint, the test error becomes better and better. As we increase the rank, we're getting closer to what we really want, which is a full-rank low nuclear norm matrix. So $d$ is not our complexity measure; rather, it's the dimensionality of the approximation to the full-rank system.

\paragraph{Remark.}
We claim that this is what's happening in neural networks too. The real object we should be learning with is an infinite-size network. But we cannot really represent an infinite-size network? Instead, we represent a truncated, finite-dimensional representation of it. As our truncation becomes finer and finer, our representation becomes better and better. We want the size to be so big that we essentially have a good approximation to the infinite-size model. Our methods should not rely on the size of the approximation being small. So we should start by understanding networks of infinite size and then worry about the question: ``how large do we need our model to be in order for it to be considered infinite?''.

\subsection{A more recent form of implicit bias : benign overfitting}

There is still one case that doesn't really fit the classical understanding of supervised learning, which was recently pointed out in \cite{kernel_learning_belkin_2018,double_descent_belkin_2019}. The observation is the following : we are getting good generalization even though we are insisting on a 0-training error solution in \emph{noisy} situations (situations where the approximation error is nonzero). In particular, in Figure~\ref{fig:c_vs_L_plot}, we want to balance complexity and training error, so we know we want to be somewhere on this regularization frontier. But the solutions we are finding are the minimum-complexity 0-error solutions - an extreme point of the frontier, and we are seeing this even in fairly noisy cases, where we would expect to be somewhere on the frontier that strikes more of a balance between complexity and training error.

To understand this a bit better, let us return to fitting noisy data with polynomials, where complexity is degree of the polynomial. As we increase the degree of the polynomial, we can decrease the training error. In this case, as seen in Figure~\ref{fig:underfitting_vs_overfitting}, we can get 0 training error with a degree-9 polynomial. But we are fitting the noise and getting bad generalization. This is captured by the classic U-shaped red curve in Figure~\ref{fig:underfitting_vs_overfitting}. At some point, we start overfitting - which we will define as fitting the noise. At that point, the test error starts to become worse because the estimation error begins to dominate. The conventional wisdom is that we should never insist on 0 training error, because it will fit the noise and generalize poorly.

\paragraph{Connection with the so-called \emph{double descent}.} Arguably, the first paper to discuss the above phenomenon was \cite{double_descent_belkin_2019}, which introduced the notion of ``double descent'' (Figure~\ref{fig:double_descent_belkin}). At this point, we probably understand about 95\% of double descent, which has little to do with the question we just asked about fitting the noise. So let us briefly discuss the double descent phenomenon, before reaching the remaining 5\% that we do not yet understand.

Why are we getting double descent? Let us think of a least squares problem in dimensionality $d$ and a fixed number of training examples. The $x$-axis is the dimensionality. The $y$-axis is the error of the ERM solution - but not just any ERM solution: if the problem is overdetermined, find the solution that minimizes the reconstruction error; if the problem is underdetermined, find the minimum Euclidean norm 0-training-error solution. Until the interpolation threshold, the $x$-axis really is complexity control. After the interpolation threshold, the $x$-axis is no longer complexity control; rather, it is the degree of approximation. As we just discussed, it is not surprising that the test error improves as the approximation becomes better.

The fact that we get an increase followed by a decrease is not surprising. What \emph{is} surprising is that we get good generalization even though we insist on 0 training error (a solution that interpolates the training set) in a fairly noisy situation. The main experiment showcasing this behaviour from \cite{kernel_learning_belkin_2018}, in which high levels of noise were added to synthetic data. When perfectly fitting this noisy data using different methods (all methods get 0 training error), the test error is substantially below the null risk. In particular, the phenomenon we are seeing is that we \emph{are} perfectly fitting the noise, but this overfitting is not harmful, as in Figure~\ref{fig:underfitting_vs_overfitting}. Rather, the overfitting is \emph{benign}. We are fitting the noise in a way that has a kind of measure-0 effect; fitting the noise does not ruin the fit in other places. We are maybe not \emph{gaining} anything from fitting this noise, but it doesn't hurt us either (it's benign).

\paragraph{Open questions regarding benign overfitting} In what situations is overfitting harmful, and in what situations is it benign? In least squares, we can get both kinds of behaviors in a now-predictable way (the result of research in the past 2-3 years, see e.g. \cite{hastie2022surprises,bartlett2020benign}). We know that it relates to a measure of effective rank at the level of the covariance matrix of the design matrix. Characterizing more generally (beyond least squares) when overfitting is harmful or benign is a big challenge that is yet to be solved, and is fairly different from how we used to think about overfitting.

%In these double descent curves, the $x$-axis is in some sense wrong. The right way to look at it is to put the complexity on the $x$-axis. 

% Let's suppose we're minimizing the loss plus an $\ell_2$ regularizer, scaled by $\lambda$. 
In some sense, the most practical implication of benign overfitting is that we don't have to worry too much about selecting the right value of the regularization parameter $\lambda$; we have a whole regime of good values of $\lambda$, even in noisy cases. In many cases in practice, though, we see something in between benign and harmful overfitting, although much closer to benign. This matches what we see empirically: that adding a bit of regularization \emph{can} be helpful by a bit, but we can still have good performance without explicit regularization.

\begin{figure}
\begin{minipage}{0.5\textwidth}
\center
\includegraphics[scale = 0.17]{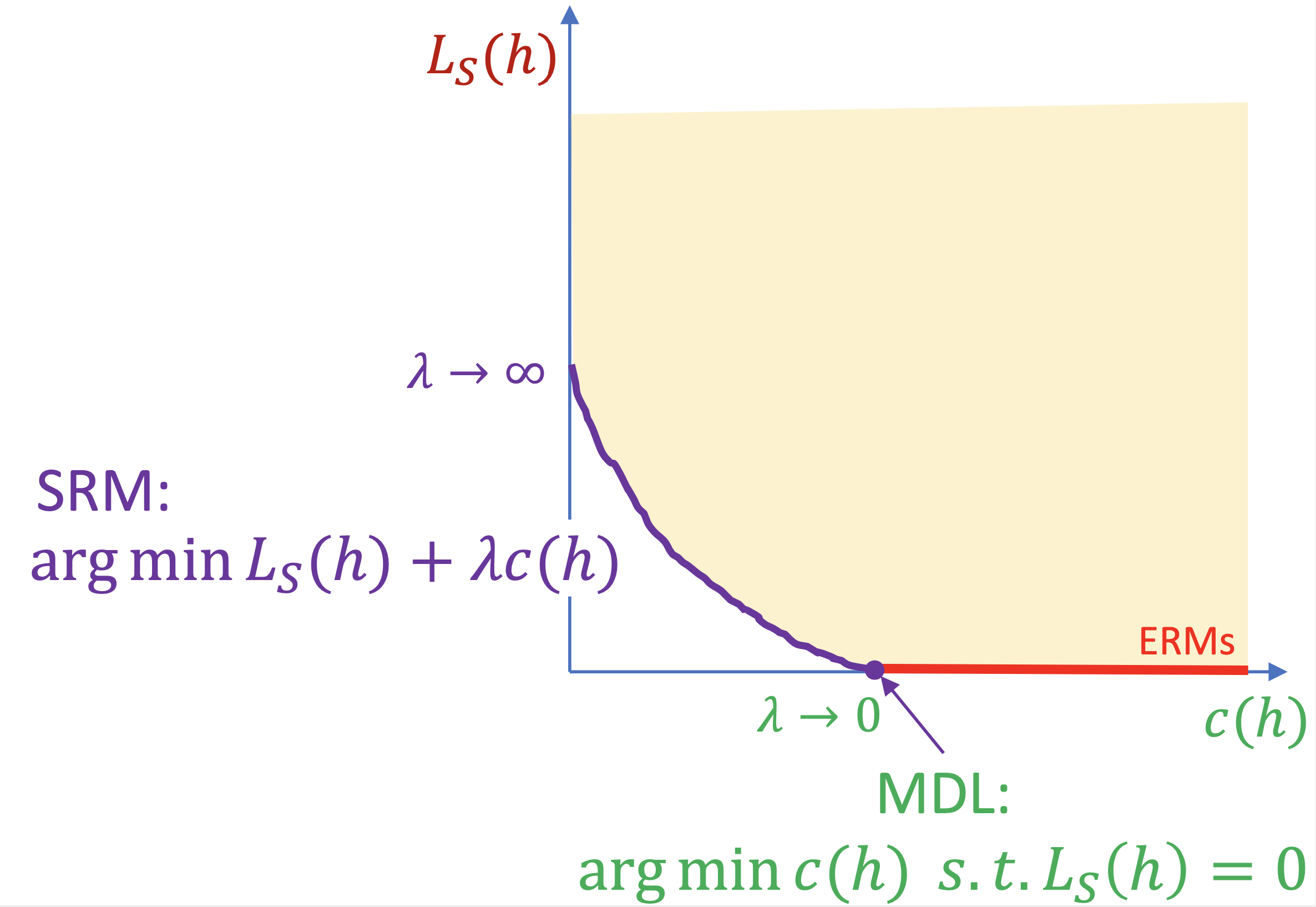}
\caption{Trading off training loss and complexity.}
\label{fig:c_vs_L_plot}
\end{minipage}
\begin{minipage}{0.5\textwidth}
\center
\includegraphics[scale = 0.16]{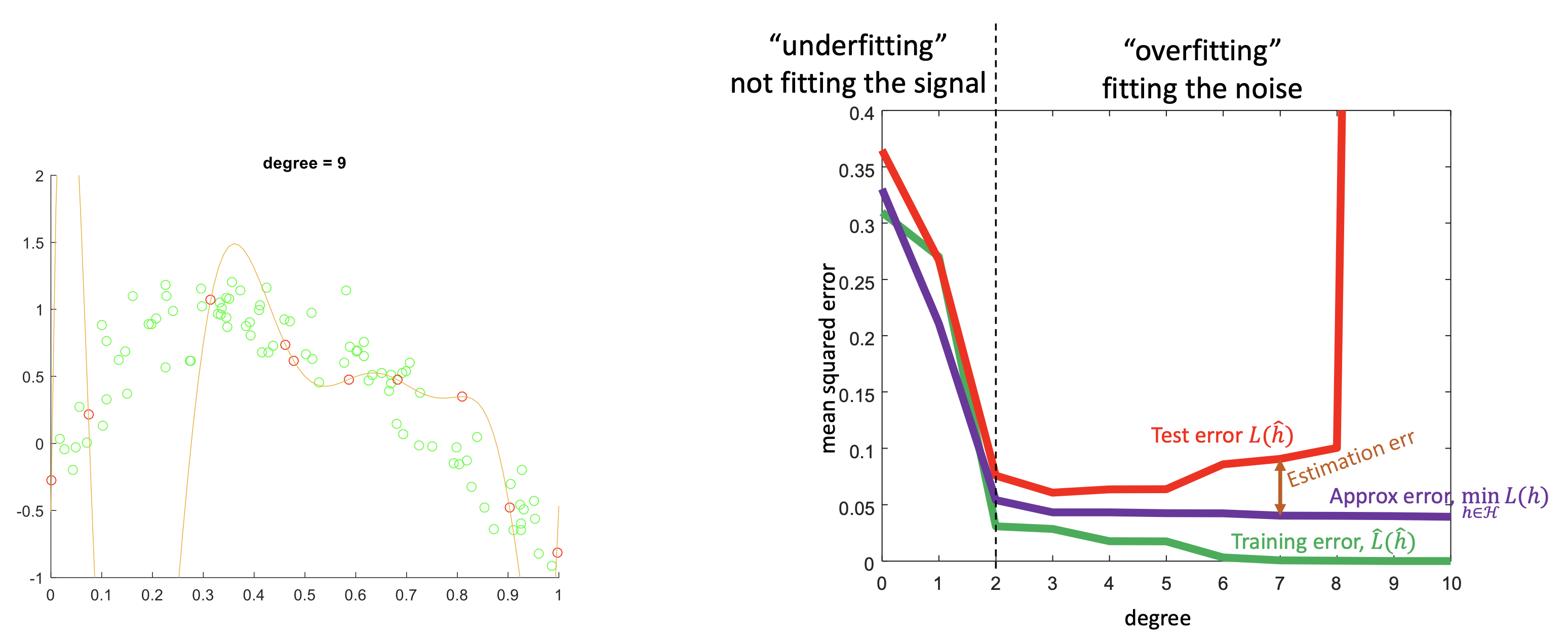}
\caption{Our classical understanding of overfitting.}
\label{fig:underfitting_vs_overfitting}
\end{minipage}
\end{figure}

\begin{figure}
\center
\includegraphics[scale = 0.30]{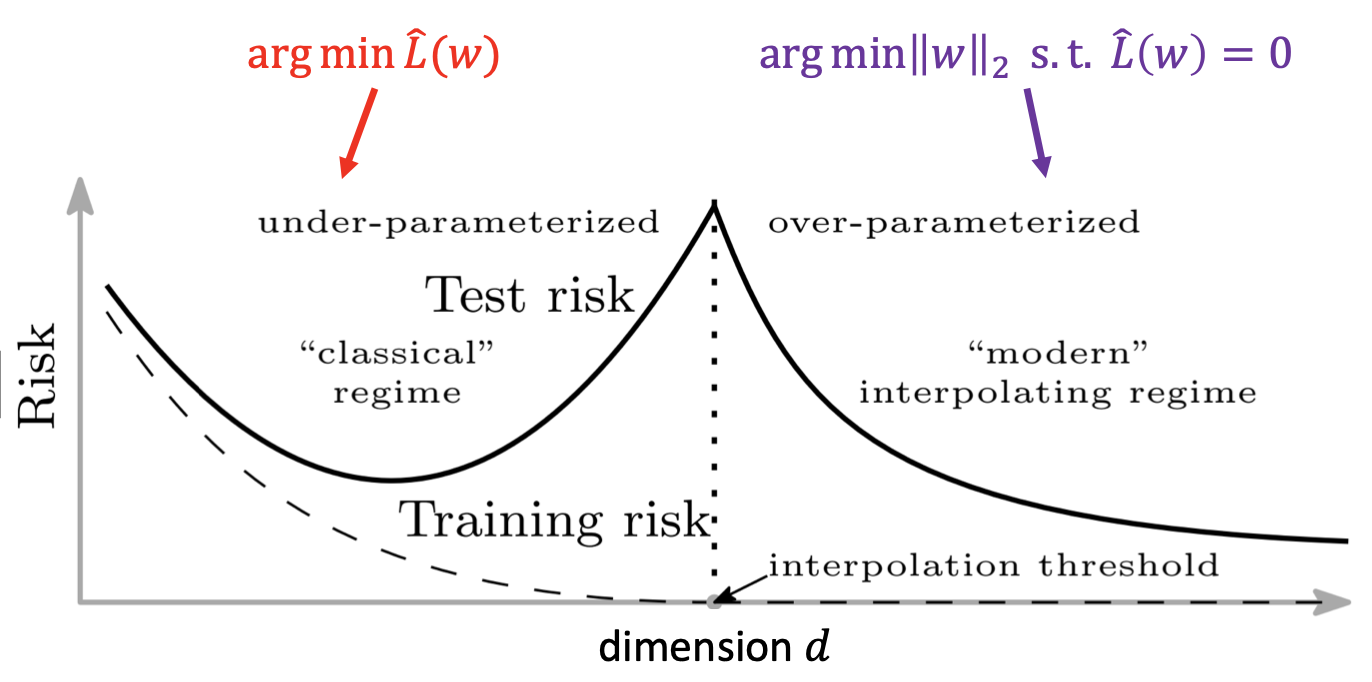}
\caption{The double descent phenomenon.}
\label{fig:double_descent_belkin}
\end{figure}

%\begin{equation}L(\hat{h}) \quad \le \underbrace{\inf_{h \in \mathcal{H}} L(h)}_{\text{approximation error}} + \quad \underbrace{\sqrt{\frac{\log |\mathcal{H}| + 2 \log(2/\delta)}{n}}}_{\text{estimation error}}\end{equation}
\subsection{A more general look at the implicit bias in optimization}
From now on and in the remaining lectures, we will focus on understanding the implicit bias coming from optimization and 
leave aside the benign overfitting phenomenon. The examples above allow us to establish the following guidelines to understand implicit bias 
in optimization : 
\begin{itemize}
  \item what complexity measure is actually being minimized by a given algorithm? This will depend on the choice of the algorithm but also, as we have seen, on the chosen architecture for our model.
  \item how do the parameters of the algorithm (initialization scale, step size, etc ...) affect the minimization of this complexity measure?
  \item How do low values of this complexity measure ensure good generalization?
\end{itemize}
The goal of the remaining lectures will be to provide a framework and method to answer these questions in simple cases. In particular, 
we will rely on the generic formulation of the \emph{mirror descent} algorithm, which will allow us to understand what 
cost function is implicitly being minimized for a given choice of architecture and algorithm parameters, as well as capture the geometry of the 
function space in which our algorithm is implicitly searching. Before doing so, let us provide some reminders on optimization in the context 
of supervised learning, which will be the subject of the next lecture.
%In the next lectures, we'll study various settings and ask: What is this true complexity measure? What does this ``complexity measure'' approach mean? 
%There's further subtlety here. It's not enough to say that the optimization algorithm biases us toward low complexity. We also have to say that the algorithm succeeds in optimizing it.
%Summing up, there are:\\
%1. mathematical questions\\
%2. questions about reality

% \end{document}
\newpage
\section{Lecture 3 : Statistical learning and stochastic convex optimization} 
    
    In this lecture, we will explore the connections between optimization geometry and generalization in the well-understood convex case. Specifically, we will derive generalization guarantees based on this geometry. Many of the concepts and proofs reproduced here can be found in classical references on statistical learning \cite{vapnik1999nature,shalev2014understanding,mohri2018foundations} and optimization \cite{bubeck2015convex,nesterov2018lectures}.

    \subsection{Learning and optimization for convex problems}
    Recall that the objective of supervised learning is, for given input and output spaces $\mathcal{X},\mathcal{Y}$, to find a predictor function $h_w :\cX \rightarrow \cY$, parametrized by a vector $w$, with low population error defined by:
    \begin{equation}\label{eq:pop-risk}
      L(h_w) = \mathbb{E}_{x, y}\left[l\left(h_{w}(x) ; y\right)\right],
    \end{equation}
    for some hidden joint density $p(x,y)$ and a chosen loss $l$ function measuring the prediction error for a given pair $(x,y)$. In what follows, we will denote $\mathcal{H}$ the chosen hypothesis class of functions that can be represented with the parameter $w$, and will consider the same notation for optimization on $\mathcal{H}$ or the corresponding parameter space. 
    Since we do not have direct access to the joint distribution $p(x,y)$, we collect a dataset $S = \{(x_1,y_1),\ldots,(x_m,y_m)\}$ of $m$ i.i.d. samples from $p$, and use it to estimate $L(h_w)$ (equivalently denoted $L(w)$) with the corresponding empirical distribution. 
    This leads to the \emph{empirical risk minimization} (ERM) problem to estimate a parameter vector $\hat{w}$:
    \begin{equation}\label{eq:def-w}
      \hat{w} = \arg\min_{h \in \mathcal{H}} \hat{L}(w) := \arg\min _{h \in \mathcal{H}} \frac{1}{m} 
      \sum_{i} l\left(h_{w}\left(x_{i}\right) ; y_{i}\right)
      + \lambda\Psi(w).
    \end{equation}
    Here, the term  $\Psi(w)$ is the regularizer or penalty, while the parameter $\lambda >0$ tunes the regularization strength.
    The purpose of this term is mainly to prevent overfitting. 
    Intuitively, a well chosen $\Psi$ will increase if a corresponding complexity measure of the model increases. Typical examples include the $\ell_{2}$ norm, enforcing regularity at the functional level (RKHS ball, etc ...), or the $\ell_{1}$ norm, inducing sparsity at the level of the parameters $w$.
    Equivalently, \eqref{eq:def-w} can be written as
     \begin{align}\label{eq:erm-b}
      \hat{w}=&\arg \min \frac{1}{m} 
      \sum_{i} l\left(h_{w}\left(x_{i}\right) ; y_{i}\right) \\
      &\mbox{such that} \quad \Psi(w) \leq B
    \end{align}
    where $B$ depends on the regularization coefficient $\lambda$. 
    The main two sources of error that need to be controlled in empirical risk minimization are the optimization and generalization error. 
    The latter is handled using uniform convergence tools to control the convergence rate of the empirical risk towards the population one, when the parameters are constrained to the sublevel sets defined by the penalty. Note that we will 
    not consider problems related to approximation error in this lecture.
    \begin{figure}[H]
      \center
      \includegraphics[scale = 0.2]{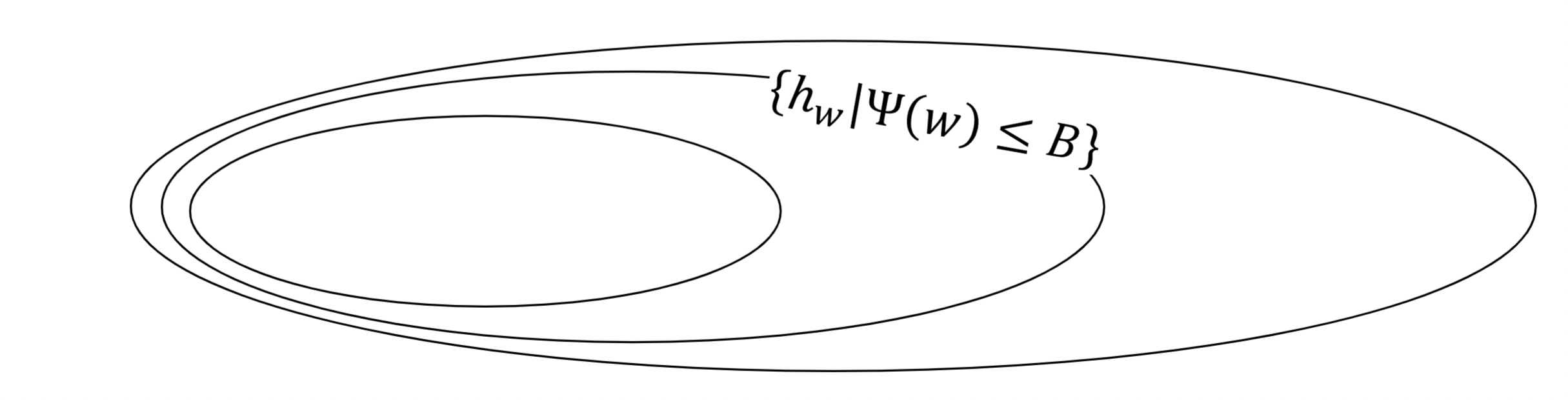}
      \caption{Graphical representation of sublevel sets of the complexity measure.}
      \label{fig-level-curves}
    \end{figure}
  By uniform convergence we mean the following. Let us look back at the ERM learning rule \eqref{eq:erm-b}.
  To ensure that it performs well with respect to the population error, we can bound the difference between the population and empirical errors uniformly over all predictors in the class. That is, for any $\epsilon >0$, we need to quantify how many samples $m$ are needed to ensure that $\sup_{\|\Psi(w)\| \leq B} \big\lvert \hat{L}(w) - L(w)\big\rvert < \epsilon$. For a given 
  learning problem, a dataset that ensures the latter inequality is said to be \emph{$\epsilon-$representative}. 
  It is straightforward to show that an $\frac{\epsilon}{2}$ representative training set ensures that 
  \begin{equation}
    L(\hat{w}) \leq \min_{w}L(w)+\epsilon,
  \end{equation}
  ensuring that the predictor $\hat{w}$ is a good proxy for the true minimizer. The standard way to achieve uniform control of the deviation between $\hat{L}(w)$, based on an available dataset $\mathcal{S}$, and $L(w)$ is 
  to quantify the complexity of the hypothesis class $\mathcal{H}$ and regularity of the loss function, and to relate these quantities to the required sample complexity. A useful 
  complexity measure for hypothesis classes is the \emph{Rademacher complexity}, defined in its empirical form as 
  \begin{equation}
    \mathcal{R}(\mathcal{H}\circ S) = \frac{1}{m}\mathbb{E}_{\boldsymbol{\sigma} \sim \{\pm 1\}^{m}} \left[\sup_{h \in \mathcal{H}}\sum_{i=1}^{m}\sigma_{i}h(z_{i})\right]
  \end{equation}
  where $\boldsymbol{\sigma}$ is a random vector with i.i.d. Rademacher entries. 
  The Rademacher complexity gives a distribution dependent alternative to the VC dimension discussed in Lecture 1, defined for any class of real-valued functions.
  We then have uniform convergence bounds similar to the ones presented using the VC dimension in Lecture 1, e.g., assuming that $\abs{l(h,z)} < c$ for some positive constant $c$, for any $h \in \mathcal{H}$
  \begin{equation}
      \sup_{h \in \mathcal{H}} \vert L(h)-\hat{L}(h) \vert \leq 2\mathcal{R}(l \circ \mathcal{H}\circ S)+c\sqrt{\frac{2/\delta}{m}}
  \end{equation}
  with probability at least $1-\delta$. The interested reader can find more details on the 
  Rademacher complexity along with examples in Chapter 26 of \cite{shalev2014understanding}.
  In this lecture, we will directly derive optimization guarantees for the learning problem formulated as a stochastic optimization problem. \\
\quad \\
    \emph{Example:} Let us assume that the reality is captured by a low-norm linear predictor. Mathematically, this goes as follows:
    \begin{equation}
      \cH = \{h_w(x) \rightarrow \lgl w,x\rgl \mid \|w\|_2 \leq B \}. 
    \end{equation}
    For simplicity, let us assume that the data and the derivative of the loss function are bounded: $\|x\|_2 \leq 1$ and $\|\nabla l\| < 1$, and that $l$ is convex. In that case,
    it can be shown straightforwardly that the empirical Rademacher complexity of the corresponding ERM problem verifies 
    \begin{equation}
      \mathcal{R}(l \circ \mathcal{H}\circ S) \leq \sqrt{\frac{B^{2}}{m}}.
    \end{equation}
    If we denote by $\hat{w}$ the $\argmin$ of the empirical loss $\hat{L}$, we then reach the following generalization result:
    \begin{equation}\label{eq:gener}
        L(\hat{w}) \leq \inf_{\|w\|\leq B} L(w) + O\left(\sqrt{\frac{B^2}{m}}\right).
    \end{equation}
    In order to compute $\hat{w}$ we perform gradient descent on $\hat{L} (w) = \frac{1}{m} \sum_i l(\lgl w,x_i\rgl,y_i)$, where $l(\cdot,\cdot)$ is the loss function. The iteration of the GD goes as follows
     \begin{equation*}
      w^{k+1} = w^k - \eta \nabla \hat{L} (w^k).
    \end{equation*}
    The convergence rate of this algorithm with the optimal step-size, see e.g. \cite{nesterov2018lectures}, is described as:
    \begin{equation*}
      \hat{L} (\bar{w}^T) \leq \inf_{\|w\|\leq B} \hat{L}(w) + O\left(\sqrt{\frac{B^2}{T}}\right)
    \end{equation*}
    However, in each iteration of the gradient descent, we need $m$ gradient computations. Depending on the data this may be computationally costly. Instead, one may use stochastic gradient descent (SGD). In this case, we uniformly pick an example $(x_i,y_i)$ and only calculate its corresponding gradient term. SGD iteration is written as
    \begin{equation*}
      \bar{w}^{k+1} = \bar{w}^k - \eta \nabla l (\lgl \bar{w}^k,x_i\rgl,y_i).
    \end{equation*}
    Thus at each iteration we subtract an unbiased estimator of the full gradient. Indeed, at one may check that
    \begin{equation*}
      \EE [\nabla l (\langle \bar{w}^k,x_i \rangle ,y_i)] = \nabla L (\bar{w}^k).
    \end{equation*}
    For the SGD algorithm we have the same convergence guarantee:
    \begin{equation*}
      \hat{L}(\bar{w}^T) \leq  \inf_{\|w\|\leq B} \hat{L}(w) + O\left(\sqrt{\frac{B^2}{T}}\right).
    \end{equation*}
    Combining this bound with \eqref{eq:gener} we have the following:
    \begin{equation*}
      L(\bar{w}^T) \leq  \inf_{\|w\|\leq B} \hat{L}(w) + O\left(\sqrt{\frac{B^2}{m}}\right)+ O\left(\sqrt{\frac{B^2}{T}}\right).
    \end{equation*}
    Thus the error magnitude of the SGD and approximation error is the same.
    This means that we need to do at most $m$ iteration of SGD because when $m < T$, the dominant term in the previous bound becomes  $O(\sqrt{\frac{B^2}{m}})$.\\
    The one pass SGD can also be viewed as an algorithm to minimize the population risk $L(w) = \EE[l((w,x),y)]$. 
    Indeed, the gradient term satisfies the following:
    \begin{equation*}
      \nabla L(w^{k}) = \EE_{x_i,y_i} \big[\nabla l(\lgl w^k,x_i\rgl,y_i)\big].
    \end{equation*}
    The latter means that instead of this two-step scheme, we can analyze the generalization using the optimization guarantee directly for the population risk. Therefore we obtain the following
    \begin{equation*}
      L(\bar{w}^T) \leq  \inf_{\|w\|\leq B} {L}(w) + O\left(\sqrt{\frac{B^2}{T}}\right).
    \end{equation*} 
    We cannot do more iterations than the number of data points, as we need to have independent samples from the population. 
    Thus the number of iterations is again bounded by the sample size and therefore we get the same bound.
    
    \subsection{Stochastic optimization}

    A stochastic optimization problem is written as 
    \begin{equation}\label{eq:st-opt}
      \min_{ w \in \cW} F(w) = \min_{ w \in \cW} \EE_{z \sim \cD} [f(w,z)]
    \end{equation}
    based on i.i.d. samples $z_1,z_2,\ldots, z_m\sim \cD$. Here the distribution $\cD$ is unknown and we do not have access to $F(w)$. 
    But using the samples we can have estimates of $F$ and $\nabla F$.
    An instance of this problem is the general learning problem. It can be formulated as 
    \begin{equation*}
      \min_{h} F(h) = \min_{h} \EE_{z \sim \cD} [f (h,z)],
    \end{equation*}
    using the data samples  $z_i \sim \cD$, for $i=1,\ldots,m$, where $h$ is a function adapted to the particular model. Here is a short list of examples.
    
    \begin{itemize}
    \item In \emph{supervised} learning we have $\cZ = \cX \times \cY = \{ z = (x,y) \mid x \in \cX, y\in \cY\}$. 
    The function $h:\cX \rightarrow \cY$ and $f(h,z)= l(h(x),y)$, where $l$ is the loss function.
    \item In the \emph{unsupervised} $k$-means clustering problem we have 
    $z =  x \in \RR^d$ and $h = (\mu[1],\ldots,\mu[n]) \in \RR^{d\times k}$. Here $h[i]$ is the center of $i$-th cluster. The objective function for this problem is defined as
      \begin{equation*}
        f((\mu[1], \mu[2], \ldots, \mu[k]), x)=\min_{i}\|\mu[i]-x\|^{2}.
      \end{equation*}
    \item The problem of \emph{density estimation} can also seen as a stochastic optimization. Consider $z = x$ in some measurable space $\cZ$ (e.g. $\RR^d$). Then for each $h$, we define the probability density $p_h(z)$ and the objective function $f(h,z) = -\log p_h(z)$. The function $F$ in this case is the $\sf KL$ divergence.
    \end{itemize}
    
    The fields of stochastic optimization and statistical learning have been developed in parallel in the 60s and 70s \cite{vapnik1974theory,nemirovsky1979problem}.
    
    Let us get back to the stochastic optimization problem \eqref{eq:st-opt}. 
    We saw two ways of solving this problem. 
    The first is based on Sample Average Approximation (SAA) or the ERM. It essentially consists of collecting data $z_1, z_2,\ldots,z_m$ and estimating the expectation term with the empirical mean
    \begin{equation*}
      \hat{F}_m(w) = \frac{1}{m} \sum_i f(w,z_i).
    \end{equation*}
    The other method is the Stochastic Approximation (SA) e.g. SGD. Here, we update $w^i$ using $f(w^i,z_i),\nabla f(w^i,z_i)$ and previous iterates.
    In particular in SGD we have: $w^{i+1} = w^{i} - \eta \nabla f(w^i,z_i)$.
    
    As mentioned previously, in machine (supervised) learning the objective function is the population risk $L(w)$. 
    In this setting, the SGD can be applied in two ways. 
    The first follows the direct SA approach (one-pass SGD).
    \begin{algorithm}[H]
        \begin{algorithmic}[1]
          \State Initialize $w^{(0)} = 0$
          \State At iteration $t = 0, 1,\ldots, T$  
          \item \qquad Draw $(x_t,y_t) \sim \cD$
          \item \qquad $w^{(t+1)} \leftarrow w^{(t)} - \eta_t \nabla l(\lgl w^{(t)},x_t\rgl,y_t) $
          \item Return $\bar{w}^{(T)} = \frac{1}{T} \sum_{t=1}^{T} w^{(t)}$.
      \end{algorithmic}
    \end{algorithm}
    For this algorithm we may obtain the following convergence guarantee\footnote{All the guarantees are satisfied up to a constant.}:
    \begin{equation*}
     {L} (\bar{w}^{(T)}) \leq {L}(w^*) + 2\sqrt{\frac{B^2}{m}} + \sqrt{\frac{B^2}{T}}
    \end{equation*}.
    However, we may perform SGD on ERM. That is our optimization problem has the following form:
    \begin{equation}
      \min_{\|w\|_2 \leq B} \hat{L} (w) = 
      \min_{\|w\|_2 \leq B} \frac{1}{m} \sum_{i=1}^{m} l(\lgl w,x_i\rgl,y_i).
    \end{equation}
    The alternative approach suggests the  minimization scheme below.
    \begin{algorithm}
        \begin{algorithmic}[1]
      \State Draw $(x_1,y_1),\ldots,(x_m,y_m) \sim \cD$
      \State Initialize $w^{(0)} = 0$
      \State At iteration $t$, we pick randomly $i \in \{1,2,\ldots m\}$.
      \State $w^{(t+1)} \leftarrow w^{(t)} - \eta_t \nabla l(\lgl w^{(t)},x_i\rgl,y_i) $
      \State Then, we may perform the step $w^{t+1} \leftarrow {\rm proj}_{\|w\|\leq B} w^{t+1}$. although we may show that implicitly our iterate will converge to this set.
      \State Return $\bar{w}^{(T)} = \frac{1}{T} \sum_{t=1}^{T} w^{(t)}$.
         \end{algorithmic}
    \end{algorithm}
    
    For this algorithm we may obtain the following convergence guarantee:
    \begin{equation*}
     {L} (\bar{w}^{(T)}) \leq {L}(w^*) + 2\sqrt{\frac{B^2}{m}} + \sqrt{\frac{B^2}{T}}.
    \end{equation*}
    
    The several differences between these two schemes. 
    \begin{itemize}
      \item 
          Since we need independent samples for the first scheme, it has as many samples as iterations, ($m = T$). 
          This is not the case for the second method, as we fix initially the $m$ samples and then choose from them. Thus it can have $T>m$ iterations. 
      \item The direct SA approach does not require regularization and thus it does not have a projection step.
      \item The SGD on ERM method is explicitly regularized. 
      The regularization of the direct SA approach hides in the step-size. 
      Indeed, in order for the parameter $w$ to have larger norm, one needs to choose larger step-sizes. 
      In particular, if we choose $\eta_t = \sqrt{B^2t}$, we get the following:
      \begin{equation*}
         {L} (\bar{w}^{(T)}) \leq {L}(w^*) + \sqrt{\frac{B^2}{T}}
      \end{equation*}
      On the other hand, the SGD on ERM has the following generalization error bound:
       \begin{equation*}
         {L} (\bar{w}^{(T)}) \leq {L}(w^*) + 2\sqrt{\frac{B^2}{m}} + \sqrt{\frac{B^2}{T}}
      \end{equation*}
      In both cases $L(w^*)$ is the value of the cost function at the optimal point in the class: $L(w^*) = \min_{\|w\|_2 \leq B} L(w)$.
    \end{itemize}

    \paragraph{Where is the regularization?} Although we mentioned the effect of the step-size on the regularization in the direct approach, 
    it is still not clear why we observe this phenomenon. Let us look back at the standard GD.
    The gradient descent  minimizes the norm $\|w\|_2$.
    Indeed, at each step of the gradient descent we minimize its linear approximation given by the gradient $g^{(t)} = \nabla F(w^{(t)})$:
    \begin{equation*}
      w^{(t+1)}  
      =  \arg\min_{w} \left\{F(w^{(t)}) + \lgl g^{(t)},w - w^{(t)} \rgl \right\}.
    \end{equation*}
    The latter is linear function and thus its minimum is at infinity. 
    Also, on the other hand, the linear approximation is not valid when we get far from the iterate $w^{(t)}$.
   To mitigate this effect, we can add a square regularizer $\|w - w^t\|_2^2 / 2\eta$, leading to 
    \begin{equation}\label{eq:gd-argmin}
    \begin{aligned}
      w^{(t+1)} &\leftarrow  
      \arg\min_{w} \left\{ F(w^{(t)}) + \lgl g^{(t)},w - w^{(t)} \rgl + \frac{1}{2\eta } \| w - w^{(t)}\|_2^2 \right\} \\
      & = \arg\min_{w} \left\{\lgl g^{(t)},w - w^{(t)} \rgl + \frac{1}{2\eta } \| w - w^{(t)}\|_2^2 \right\} \\
      &= w^{(t)} - \eta g^{(t)}.
    \end{aligned}
    \end{equation}
    In order to better understand this in the context of SGD, let us first introduce the notion of stability.

    \subsection{Stability}
    Here we will be studying a notion of stability for loss functions, generically denoted $\hat{F}(w)$, defined as empirical sums of the form $\frac{1}{m}\sum_{i=1}^{m}f(w,z_{i})$, for a given dataset $(z_{i})_{1 \leq i \leq m}$. We start by defining the leave-one-out stability and replace-one-out stability, and then derive generalization bounds using these quantities, see e.g. \cite{shalev2010learnability}.
    \begin{definition}
      Let $\beta : \mathbb{N} \to \mathbb{R}$ be a monotonically decreasing function. A learning rule $\tilde{w}(z_1,\ldots,z_m)$ is \emph{leave one out  $\beta(m)$-stable} if 
      \begin{equation*}
        | f(\tilde{w}(z_1,\ldots,z_{m-1}),z_m) -  f(\tilde{w}(z_1,\ldots,z_{m}),z_m) | \leq \beta(m),
      \end{equation*}
      and \emph{replace one out $\beta(m)$-stable} if, 
      \begin{equation*}
        | f(\tilde{w}(z_1,\ldots,z_{m-1},z'),z_m) -  f(\tilde{w}(z_1,\ldots,z_{m}),z_m) | \leq \beta(m),
      \end{equation*}
      where $z'$ is a new independent sample from the hidden distribution $\mathcal{D}$.
    \end{definition}

    For simplicity we will assume that the learning rule is symmetric. That is $\tilde{w}(z_1,\ldots,z_m) = \tilde{w}(\sigma({z}_1),\ldots,\sigma(z_m))$, where $\sigma$ is any permutation defined on $\{1,2,\ldots,m\}$.
    \begin{theorem}\label{th-stab}
      Define $\widehat{F}(w) := \frac{1}{m} \sum_{i=1}^{m} f_i(w)$. If $\tilde{w}$ is symmetric and $\beta(m)$ stable then 
      \begin{equation*}
        \EE[f\left(\widetilde{w}\left(z_{1}, \ldots, z_{m-1}\right),z_m\right)] 
        \leq \EE[\widehat{F}\left(\widetilde{w}\left(z_{1}, \ldots, z_{m}\right),z_m\right)] + \beta(m)
      \end{equation*}
    \end{theorem}
    \begin{proof}
      By symmetry of $\widetilde{w}$ we have
      \begin{equation*}
      \begin{aligned}
        \mathbb{E}_{z_{1}, \ldots, z_{m}}&\left[f\left(\widetilde{w}\left(z_{1}, \ldots, z_{m-1}\right), z_{m}\right)\right] \\
        &=\frac{1}{m} \sum_{i=1}^{m} \mathbb{E}\left[f\left(\widetilde{w}\left(z_{1}, \ldots, z_{i-1}, z_{i+1}, \ldots, z_{m}\right), z_{i}\right)\right]
      \end{aligned}
      \end{equation*}
      Using the stability of the function $f$, we get the following
      \begin{equation*}
      \begin{aligned}
        \mathbb{E}_{z_{1}, \ldots, z_{m}}&\left[f\left(\widetilde{w}\left(z_{1}, \ldots, z_{m-1}\right), z_{m}\right)\right] \\
        & \leq \frac{1}{m} \sum_{i=1}^{m}\left(\mathbb{E}\left[f\left(\widetilde{w}\left(z_{1}, \ldots, z_{m}\right), z_{i}\right)\right]+\beta(m)\right) \\
        &=\mathbb{E}\left[\frac{1}{m} \sum_{i=1}^{m} f\left(\widetilde{w}\left(z_{1}, \ldots, z_{m}\right), z_{i}\right)\right]+\beta(m)\\
        &=\mathbb{E}\left[\widehat{F}\left(\widetilde{w}_{m}\right)\right]+\beta(m)
      \end{aligned}
      \end{equation*}
    \end{proof}
    This result yields generalization for stable learning rules. 
    However, one needs to take into account that stability may be tricky in the learning problem. 
    In the case, when the predictor interpolates the data, the empirical error is equal to zero. 
    But most interpolators hardly satisfy the stability condition with a small $\beta(m)$. 
    Thus, the right hand side will be very large and hence non-informative. 
    On the other hand, let us consider the zero predictor. 
    It is stable, as the rule does not depend on the data. 
    However, it has a very large empirical error. 
    We therefore need a different rule that is stable and has small empirical error, to guarantee generalization.

    \subsection{Strong convexity}
    Let us define strong convexity for real valued functions.
    \begin{definition}
      The function $\Psi : \cW \rightarrow \RR$ is $\alpha$-strongly convex w.r.t. to a norm $\|w\|$ if for 
      any $w,w' \in \mathcal{W}$:
      \begin{equation}\label{eq:str-conv}
        \Psi(w') \geq \Psi(w)  + \lgl \nabla \Psi(w), w' - w \rgl + \frac{\alpha}{2} \|w' - w\|^2. 
      \end{equation}
    \end{definition}
    Strong convexity essentially means that the function is bounded below by a quadratic function.
    This property depends on the norm unlike the convexity which can be defined on a vector space. 
    Let us apply Eq.\eqref{eq:str-conv} for the optimum point $w_0 := \argmin_{w} \Psi(w)$.
    Since $\nabla \Psi(w_0) = 0$ we get the following for every $w$:
    \begin{equation}
    \label{eq:strg_conv_prop}
        \Psi(w) - \Psi(w_0)  \geq  \frac{\alpha}{2} \|w - w_0\|^2. 
    \end{equation}
    The latter means that the difference between the function value at any point and the optimal one is proportional to the distance from the optimal point. 
    Let us now establish the connection between strong convexity and stability. For the dataset $S = \{z_1,\ldots,z_m\}$ we define the regularized ERM (RERM) as follows:
    \begin{equation*}
      {\sf RERM}_{\lambda\Psi}(S) = \arg\min_{w \in \cW} \left\{\widehat{F}(w) + \lambda \Psi(w)\right\}.
    \end{equation*}
    Here $\Psi$ is an $\alpha$-strongly convex function and $\widehat{F}$ is the empirical mean corresponding to the dataset $S$.
    We now recall the definition of Lipschitz continuity.
    \begin{definition}
      The function $f(w,z)$ is $G$-Lipschitz w.r.t. $\|\cdot\|$ if and only if for every $z \in \cZ$ and $w,w' \in \cW$ 
      \begin{equation*}
        \vert f(w,z) - f(w',z) \vert \leq G \|w' - w\|.
      \end{equation*}
    \end{definition}
    One may notice that Lipschitz continuity implies some kind of ``stability", as it essentially means that small perturbation of the argument
    will not change the function value drastically. 
    As for strong convexity, the Lipschitzness also depends on the norm of whatever normed space we are working on. 
    This yields that $\|\nabla_w f(w,z)\|_* \leq G$, where $\|\cdot\|_*$ is the dual norm and $\nabla_w f(w,z)$ is viewed as an element of the dual of $\mathcal{W}$.
    We will assume that the norm is the same for both properties. 
    \begin{proposition}
      If $f$ is $G$-Lipschitz  and $\Psi$ is $\alpha$-strongly convex then 
      ${\sf  RERM}_{\lambda \Psi} (S)$ is stable with a coefficient 
      \begin{equation}
        \beta(m)\leq \frac{2G^2}{m\lambda\alpha}.
      \end{equation}
    \end{proposition}
    \begin{proof}
        In the following proof we will abbreviate $RERM$ with $R$. It is straightforward to check that if $f$ is $G-$Lipschitz then so is $\hat{F}$ (with respect to $w$, conditionally on the $z_{i}$). Denote $S = (z_{1},...,z_{m})$ the full dataset and $S^{(-m)} = (z_{1},...,z_{m-1})$ the dataset in which $z_{m}$ is left out. Now define $h_{S}(w) = \hat{F}_{S}(w)+\lambda \Psi(w)$, where $\hat{F}_{S}(w) = \frac{1}{m}\sum_{i=1}^{m}f(w,z_{i})$. Then $h_{S}$ is $\lambda \alpha$ strongly convex and from equation \eqref{eq:strg_conv_prop} we have
        \begin{equation}
        \label{eq:str_cvx_dist}
            h_{S}(w)-h_{S}({\sf  R}_{\lambda \Psi} (S)) \geq \lambda \alpha \norm{w-{\sf  R}_{\lambda \Psi} (S)}_{2}^{2}.
        \end{equation}
        Now, denote ${\sf  R}_{\lambda \Psi} (S^{(-m)})$ the solution to the ERM problem with the dataset $S^{(-m)}$. Then
        \begin{align}
          \label{eq:inter_1}
            &h_{S}({\sf  R}_{\lambda \Psi} (S^{(-m)}))-h_{S}({\sf  R}_{\lambda \Psi} (S)) = h_{S^{(-m)}}({\sf  R}_{\lambda \Psi} (S^{(-m)}))-h_{S^{(-m)}}({\sf  R}_{\lambda \Psi} (S)) \notag \\
            &+\frac{1}{m}\left(f({\sf  R}_{\lambda \Psi} (S^{(-m)}),z_{m})-f({\sf  R}_{\lambda \Psi} (S),z_{m})\right).
        \end{align}
        Since $h_{S^{(-m)}}(w)$ is positive and strongly convex, 
        \begin{align}
            h_{S^{(-m)}}({\sf  R}_{\lambda \Psi} (S^{(-m)}))-h_{S^{(-m)}}({\sf  R}_{\lambda \Psi} (S)) \geq 0
        \end{align}
        so that 
        \begin{align}
            h_{S}({\sf  R}_{\lambda \Psi} (S^{(-m)}))-&h_{S}({\sf  R}_{\lambda \Psi} (S)) \notag \\
            &\leq \frac{1}{m}\left(f({\sf  R}_{\lambda \Psi} (S^{(-m)}),z_{m})-f({\sf  R}_{\lambda \Psi} (S^{(-m)}),z_{m})\right) \notag \\
            &\leq G\norm{{\sf  R}_{\lambda \Psi} (S^{(-m)})-{\sf  R}_{\lambda \Psi} (S)}_{2}
        \end{align}
        using the Lipschitz continuity of $f$. Combining this inequality with Eq.\eqref{eq:str_cvx_dist}, we reach 
        \begin{equation}
            \norm{{\sf  R}_{\lambda \Psi} (S^{(-m)})-{\sf  R}_{\lambda \Psi} (S)}_{2} \leq \frac{G}{\alpha m \lambda},
        \end{equation}
        and
        \begin{equation}
            f({\sf  R}_{\lambda \Psi} (S^{(-m)}),z_{m})-f({\sf  R}_{\lambda \Psi} (S),z_{m}) \leq \frac{G^{2}}{\alpha m \lambda}.
        \end{equation}
        In the case of replace-one-out stability, we define $S^{'} = (z_{1},...,z_{m-1},z')$ and the related estimator ${\sf  R}_{\lambda \Psi} (S^{'})$. Eq.\eqref{eq:inter_1} then becomes 
        \begin{align}
            &h_{S}({\sf  R}_{\lambda \Psi} (S^{'}))-h_{S}({\sf  R}_{\lambda \Psi} (S)) = h_{S^{'}}({\sf  R}_{\lambda \Psi} (S^{'}))-h_{S^{'}}({\sf  R}_{\lambda \Psi} (S)) \notag \\
            &+\frac{1}{m}\left(f({\sf  R}_{\lambda \Psi} (S^{'}),z_{m})-f({\sf  R}_{\lambda \Psi} (S),z_{m})\right)+\frac{1}{m}\left(f({\sf  R}_{\lambda \Psi} (S^{'}),z')-f({\sf  R}_{\lambda \Psi} (S),z')\right),
        \end{align}
        leading to 
        \begin{equation}
          \norm{{\sf  R}_{\lambda \Psi} (S^{'})-{\sf  R}_{\lambda \Psi} (S)}_{2} \leq \frac{2G}{\alpha m \lambda}.
      \end{equation}
    \end{proof}
    In the rest of the lecture, without loss of generality, we may assume that the regularization function $\Psi$ is $1$-strongly convex. 
    This is easily achieved by tuning the parameter $\lambda$ accordingly.
    \Cref{th-stab} then yields the following
    \begin{equation*}
      \begin{aligned}
        \mathbb{E}\left[F\left({\sf RERM}_{\lambda \Psi}(S)\right)\right]
        &\leq \mathbb{E}\left[\widehat{F}\left({\sf RERM}_{\lambda \Psi}(S)\right)\right]+\frac{2 G^{2}}{\lambda m}.
      \end{aligned}
    \end{equation*}
    Without loss of generality we may assume that $\Psi$ is a positive function. Then using the definition 
    of $\sf RERM$, for every $w \in \cW$ we obtain
    \begin{equation*}
      \begin{aligned}
        \mathbb{E}\left[F\left({\sf RERM}_{\lambda \Psi}(S)\right)\right]
        & \leq \mathbb{E}\left[\widehat{F}\left({\sf RERM}_{\lambda \Psi}(S)\right)+\lambda \Psi({\sf RERM}_{\lambda \Psi}(S))\right]
        +\frac{2 G^{2}}{\lambda  m}\\
        & \leq \mathbb{E}\left[\widehat{F}(w)+\lambda \Psi(w)\right]+\frac{2 G^{2}}{\lambda m}\\
        &=F(w)+\lambda \Psi(w)+\frac{2 G^{2}}{\lambda  m}\\
        &\leq \inf _{w \in \mathcal{W}} F(w) + \sqrt{\frac{8 G^{2}\sup_{w}\Psi(w)}{ m}},
      \end{aligned}
    \end{equation*}
    where the last line is obtained by choosing $\lambda = \sqrt{\frac{2G^{2}}{\alpha m \sup_{w} \Psi(w)}}$. 
    In particular, using the constraint $\Psi(w) \leq B$ and optimizing over the parameter $\lambda$, the last inequality can be rewritten as
    \begin{equation*}
        \mathbb{E}\left[F\left({\sf RERM}_{\lambda \Psi}(S)\right)\right] - F(w^*) \leq O\left(\sqrt{\frac{\sup \{\|\nabla_w f\|_*^{2}\}B}{m}}\right).
    \end{equation*}
    Let us now look back at the Risk minimization problem in the convex case:
    \begin{equation*}
      \min_{w\in\cW} \EE_{z\sim \cD} [f(w,z)] = 
      \min_{w\in\cW} \EE_{z\sim \cD} [\textsf{loss}(\lgl w, \phi(x)\rgl,y)].
    \end{equation*}
    $\cW$ is assumed to be convex.
    If ${\sf loss}(\hat{y},y) $ is convex in $\hat{y}$, then the problem is convex.
    For a non-trivial loss, the composition $\textsf{loss} (h_w(x),y)$ is convex in $w$ \textbf{only} when $h_w(x) = \lgl w,\phi(x) \rgl$.
    In this setting Lipschitz continuity may be established as follows. 
    Assume that the loss function $\sf loss(y,y')$ is $g$-Lipschitz continuous w.r.t. $y$. Then
    \begin{equation}
      | f(w,(x,y)) - f(w',(x,y)) | \leq g \|\phi(x)\|_* \cdot \|w - w'\|.
    \end{equation}
    In particular, the learning problem becomes $G = gR$ Lipschitz continuous, if we assume that $\|\phi(x)\|_* \leq R$, for some $R > 0$.
    Hence the generalization bound becomes
   \begin{equation*}
      \mathbb{E}\left[F\left({\sf RERM}_{\lambda \Psi}(S)\right)\right] - F(w^*) 
      \leq O\left(\sqrt{\frac{\Psi(w^*)\sup \|\phi(x)\|_*}{m}}\right).
    \end{equation*} 
    In order for the right-hand side to be small we need to have an appropriate sample size. 
    Below we derive the sample complexity for several examples.
    \begin{itemize}
      \item $\Psi(w) = \frac{1}{2} \|w\|_2^2$ is $1$-strongly convex w.r.t. $\|w\|_2$ then $m \propto \|w\|_2^2 \cdot \|\phi(x)\|_2^2$.
      \item $\Psi(w) = \frac{1}{2} w^{T} Q w$ is $1$-strongly convex w.r.t. $\|w\|_Q$ then   $m \propto  (w^{T} Q w) (x^{T} Q^{-1} x)$.
      Here we choose $Q$ to be small in some direction, then we pay for it in its dual $Q^{-1}$.
      \item $\Psi(w) = \frac{1}{2(p-1)} \|w\|_p^2$ is $1$-strongly convex w.r.t. $\|w\|_p$, 
      then   $m \propto  \frac{\|w\|_p^2 \cdot \|x\|_q^2}{p-1}$.
      Here, we would like the $q$ norm of the data to be small. Thus, we want $q$ to be large. 
      Hence $p$ must be close to $1$, which explodes the denominator of the sample complexity.
      \item $\Psi(w) = \sum_i w[i] \log \big(\frac{w[i]}{1/d}\big)$ is $1$-strongly convex w.r.t. $\|w\|_1$.
      This problem is called the entropic minimizer. Its sample complexity satisfies   
      $m \propto  \frac{\|w\|_1^2 \cdot \|x\|_{\infty}^2}{p-1}$. 
    \end{itemize}
    We see in this example that in order to have good sample complexity we need to have matching geometries for the data and the parameter.
    \paragraph{Online learning} 
    The online learning paradigm resembles the previously discussed stochastic optimization framework.
    The optimizer provides $w_i$. We give it to the adversary to compute $f(w_i,z_i)$
    and then use the value of $f(w_i,z_i)$ to compute $w_{i+1}$.

    The stability in online learning setting plays an important role. 
    Consider the Follow The Leader (FTL) rule. 
    It proposes to choose
    \begin{equation}
      \hat{w}_m (z_1,\ldots,z_{m-1}) = \arg\min_{w\in\cW} \sum_{i=1}^{m-1} f(w,z_i) 
    \end{equation}
    However, this is an unstable rule. 
    A better method called Be the Leader (BTL) suggests the following:
    \begin{equation}
      \hat{w}_m (z_1,\ldots,z_{m-1}) = \arg\min_{w\in\cW} \sum_{i=1}^{m} f(w,z_i).
    \end{equation}
    This has great convergence properties, but it is not implementable as we assume to have access to  $f(w,z_m)$. 
    Instead, we regularize the FTL. The Follow the Regularized Leader (FTRL) goes as 
    \begin{equation*}
      \hat{w}_m (z_1,\ldots,z_{m-1}) = \arg\min_{w\in\cW} \sum_{i=1}^{m-1} f(w,z_i) + \lambda_t \Psi(w).
    \end{equation*} 
    This algorithm however, does not resemble the one-pass SGD algorithm (see also equation \eqref{eq:gd-argmin}):
    \begin{equation}
      w_{t+1} = \argmin_{w} \lgl \nabla f(w_t,z_t),w \rgl + \lambda_t \|w - w_t\|_2^2/2.
    \end{equation} 
    With the increase of the iteration $m$, our FTRL needs to minimize a more complex sum-decomposable function. 
    This becomes costly for large $m$'s. 
    For convex objectives we can relax the problem. We minimize the linear approximation of the objective (Linearized FTRL):
    \begin{equation}\label{eq:lin-ftrl}
      \hat{w}_m^{\lambda} (z_1,\ldots,z_{m-1}) = \arg\min_{w\in\cW} \frac{1}{m} \left\lgl \sum_{i=1}^{m-1} 
      \nabla f(w_i,z_i),w\right\rgl + \lambda_t \Psi(w).
    \end{equation} 
    This problem is simpler than the FTRL, but it is still more complex than the one-pass SGD because of its dependence on the previous gradient evaluations. To fill this gap we introduce the mirror descent method.

    \subsection{Mirror descent}

    In this part of the lecture, we will present the mirror descent algorithm. We will get a better understanding of it in the upcoming lectures.
    Let us define the Bregman divergence. 
    For a given strictly convex, continuously differentiable function $\Psi$ defined on a convex set $\Omega$, the Bregman divergence between two points $x,y$ in $\Omega$ is given by 
    \begin{equation*}
      D_{\Psi} (x \mid y) =  \Psi(x) - (\Psi(y) + \lgl\nabla \Psi(y),x-y\rgl).
    \end{equation*}
    Intuitively, it corresponds to the distance, at point $x$, between the function $\Psi$ and its linearization at point $y$, see \Cref{fig:Bregman} for an illustration.
    \begin{figure}
      \center
      \includegraphics[scale = 0.15]{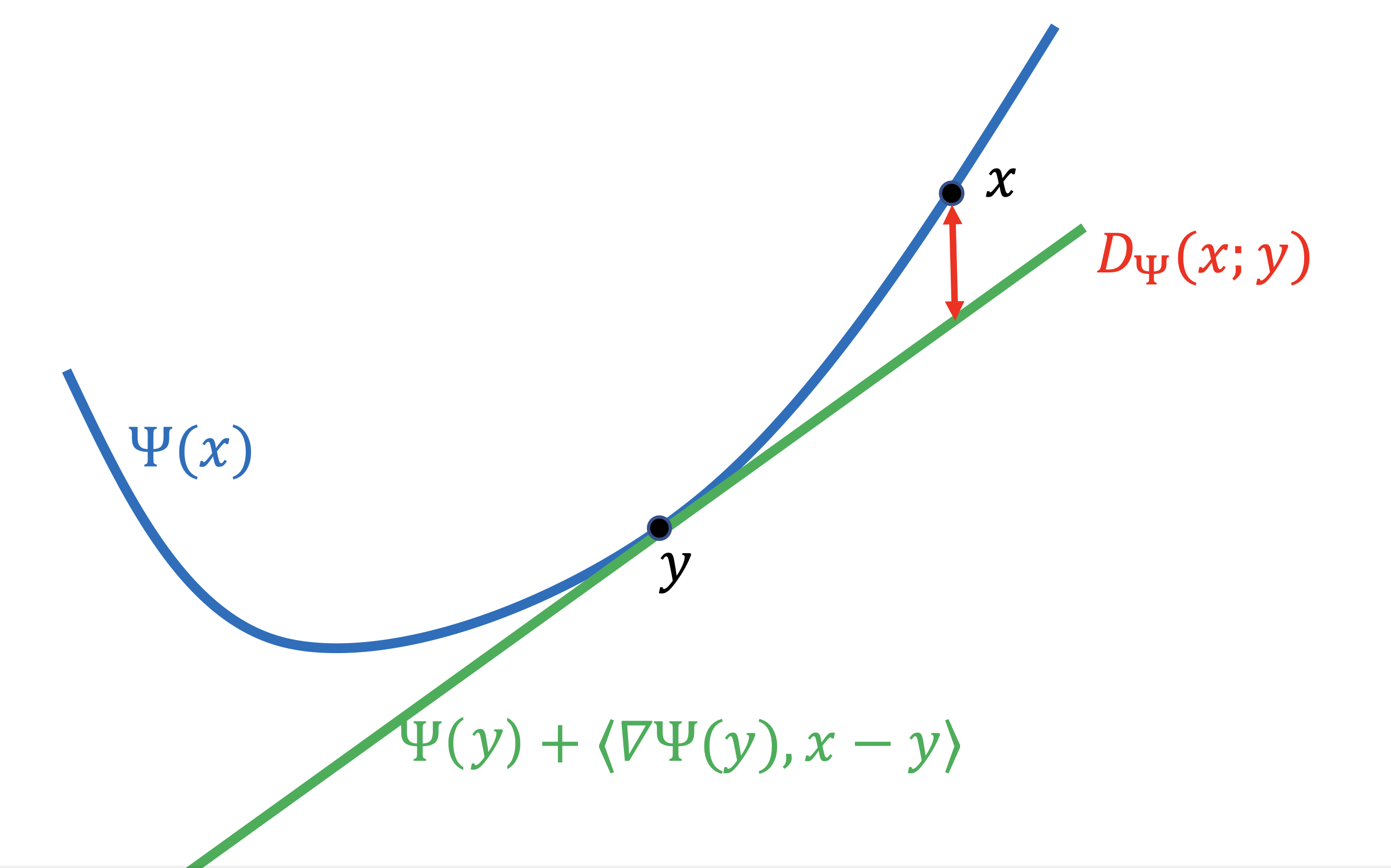}
      \caption{Bregman divergence.}
      \label{fig:Bregman}
    \end{figure}
    In particular, for $\alpha$-strongly convex functions, it holds that $D_{\Psi} (x \mid y) \geq \alpha\|x- y\|^2/2$.
    The mirror descent is then defined as the following iterative scheme:
    \begin{align}\label{eq:mirror-descent}
        w_{t+1} &= \arg\min_{w \in \cW} \lgl \nabla f(w_t,z_t) , w \rgl + \lambda_t D_{\Psi} (w \mid w_t)\\
                 &= \Pi_{\Psi}^{\cW} \left(\nabla \Psi^{-1} \left( \nabla \Psi (w_t) - \frac{1}{\lambda_t} \nabla f(w_t,z_t) \right)\right),
    \end{align}
    where $ \Pi_{\Psi}^{\cW}(w) = \min_{w' \in \cW} D_{\Psi} (w' \mid w)$ is a projection step on $\mathcal{W}$ with respect to the Bregman distance. One may easily verify that if $\Psi$ is a quadratic function, then its Bregman divergence is also quadratic, and we recover the standard gradient descent algorithm, assuming the constraint set $\mathcal{W}$ is the domain of definition of $\Psi$. Indeed, taking $\Psi(x) = \frac{1}{2}\norm{x}_{2}^{2}$, we obtain 
    \begin{equation}
        D_{\Psi}(x \vert y) = \frac{1}{2}\norm{x-y}_{2}^{2}
    \end{equation}
    so that 
    \begin{align}
        &\arg\min_{w \in \cW} \lgl \nabla f(w_t,z_t) , w \rgl + \lambda_t D_{\Psi} (w \mid w_t) \\
        &= w_{t}-\frac{1}{\lambda_{t}}\nabla f(w^{t},z^{t})
    \end{align}
    Hence the one-pass SGD is also an instance of the mirror descent.
    Let us now look at the minimization problem \eqref{eq:mirror-descent}. The optimality condition is the following:
    \begin{equation}
      \begin{aligned}
        0  &= \nabla_w \big(\lgl \nabla f(w_t,z_t) , w \rgl + \lambda_t D_{\Psi} (w \mid w_t)\big)\\
        &=  \nabla f(w_t,z_t) + \lambda_t \big(\nabla \Psi(w) -  \nabla \Psi(w_t)\big),\\
      \end{aligned}
     \end{equation} 
    which leads to
     \begin{equation*}
      \begin{aligned}
        \nabla \Psi(w)  &=   \nabla \Psi(w_t)  - \frac{1}{\lambda_t} \nabla f(w_t,z_t). \\
      \end{aligned}
     \end{equation*} 
     For differentiable, strongly convex functions the gradient is invertible, and we may write the following formula for $w_{t+1}$
     \begin{equation*}
       \begin{aligned}
         w_{t+1} = \Pi_{\Psi}^{\cW}\left( \big(\nabla \Psi\big)^{-1} 
         \Big(\nabla \Psi(w_t)  - \frac{1}{\lambda_t} \nabla f(w_t,z_t)\Big) \right).
       \end{aligned}
     \end{equation*}
    Here $\Pi_{\Psi}^{\cW}(w_0)$ is the projection of $w_0$ on $\cW$: 
    \begin{equation*}
      \Pi_{\Psi}^{\cW}(w_0) = \arg\min_{w\in \cW} D_{\Psi}(w\mid w_0).
    \end{equation*}
    See \Cref{fig:mirror} for illustration of the mirror descent algorithm.
     \begin{figure}
      \center
      \includegraphics[scale = 0.25]{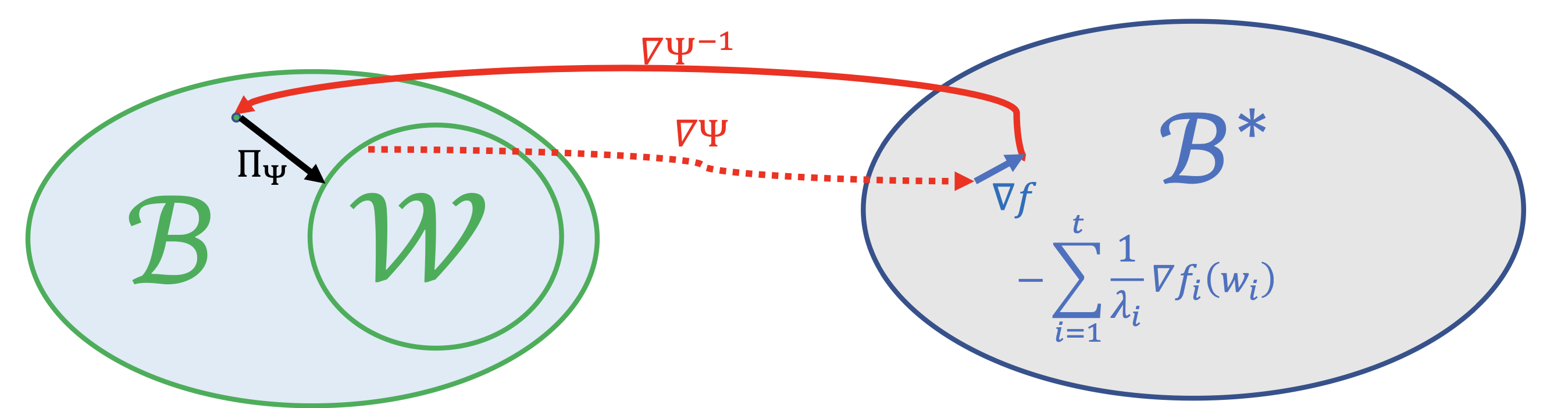}
      \caption{Suppose that the hypothesis class $\cW$ is a subset of $\cB$, where $\cB$ is a Banach space. As we have seen above, the gradients live in the dual space $\cB^*$. Then each iteration of the mirror descent consists of four steps. First we take the map
      $\nabla \Psi : \cB \rightarrow \cB^*$. Then we do a step  in $\cB^*$ in the direction of the gradient $\nabla f$. At the third step, we take $\Psi^{-1}$, which maps $\cB^*$ to $\cB$. The final step is the projection of this point to the hypothesis set $\cW$.}
      \label{fig:mirror}
    \end{figure}
    Suppose that $\cW$ is the entire space: $\cW = \cB$. Then in the dual space $\cB^*$ we get a sum of the gradients, as there is no projection in the primal space. 
    We take the initial point $w_0 = \argmin_w \Psi(w)$. This choice is intuitive as one would like to start with a model with lowest complexity. Thus the mirror descent iteration goes as follows:
     \begin{equation*}
       \begin{aligned}
         w_{t+1} &=  \big(\nabla \Psi\big)^{-1} 
         \Big(\nabla \Psi(w_0)  - \sum_{i=1}^{t} \frac{1}{\lambda_i} \nabla f(w_i,z_i)\Big)\\
         &= \argmin_w \Big(\sum_{i=1}^{t} \frac{1}{\lambda_i} \big\lgl\nabla f(w_i,z_i),w \big\rgl + \nabla \Psi(w)\Big),\\
       \end{aligned}
     \end{equation*}
     and we recover the linearized FTRL iteration \eqref{eq:lin-ftrl}.

\subsection{Mirror descent and implicit bias of optimization}
    We see that the generic formulation of the mirror descent algorithm allows to capture the geometry of an optimization problem 
    through the Bregman distance of a chosen potential and the associated projection operator. For a given function $F(w)$ to optimize (over weights $w$), mirror descent implicitly 
    minimizes an effective cost at each time step taking the form $\lgl \nabla f(w_t,z_t) , w \rgl + \lambda_t D_{\Psi} (w \mid w_t)$. We will see that, with some additional work,
    this framework can be used to understand implicit bias in optimization : for a given descent algorithm, we can find the potential that is being implicitly minimized and the optimality conditions 
    that the fixed point of the descent method corresponds to. This will be the core topic of the next lecture.
\newpage
\section{Lecture 4 : Mirror descent and implicit bias of descent algorithms}
We now have a good understanding of stochastic and online optimization for supervised learning. Consider an objective of the form:

$$
f(w, (x,y)) = loss(\langle w, \phi(x) \rangle, y)
$$
so that we have convexity with respect to $w\,.$ In what follows, we will abbreviate the pair $(x,y)$ as $z$. This is suitably generic because any data set is realizable in this form: We can choose $\phi$ as mapping $x$ to an indicator about its identity, and $w$ to select the appropriate label.

There are two reasons why studying convex optimization can be useful to understand deep learning: One is because it relates to mirror descent, and the second is that it allows us to discuss the geometry of the optimization problem and therefore the inductive bias.

Recall that if we are exploring the loss landscape with gradient descent, we are implicitly staying close in $\ell_2$ norm to the initialization.

\subsection{Mirror Descent}
We derived mirror descent as iteratively minimizing a regularized form of the first order approximation of the objective function using a convex potential. We obtained the following form:
\begin{align}
    w_{k+1} &= \arg \min_{w \in \mathcal{W}} \langle \nabla f(w_k, z_k), w \rangle + \lambda_k D_{\Psi}(w || w_k) \notag \\
    &=  \Pi_{\Psi}^{\cW} \left(\nabla \Psi^{-1} \left( \nabla \Psi (w_k) - \frac{1}{\lambda_k} \nabla f(w_k,z_k) \right)\right)
\end{align}

where

$$
D_{\Psi}(w || w') = \Psi(w) - \Psi(w') + \langle \nabla \Psi(w'), w-w' \rangle.
$$

Since the value of the step size $\eta_{k}$ is directly linked to the value of the regularization parammeter $\lambda_{k}$, we may rewrite the iteration as 
\begin{align}
  w_{k+1} = \Pi_{\Psi}^{\cW} \left(\nabla \Psi^{-1} \left( \nabla \Psi (w_k) - \eta_{k} \nabla f(w_k,z_k) \right)\right).
\end{align}
\begin{figure}
    \centering
    \includegraphics[width=0.9\textwidth]{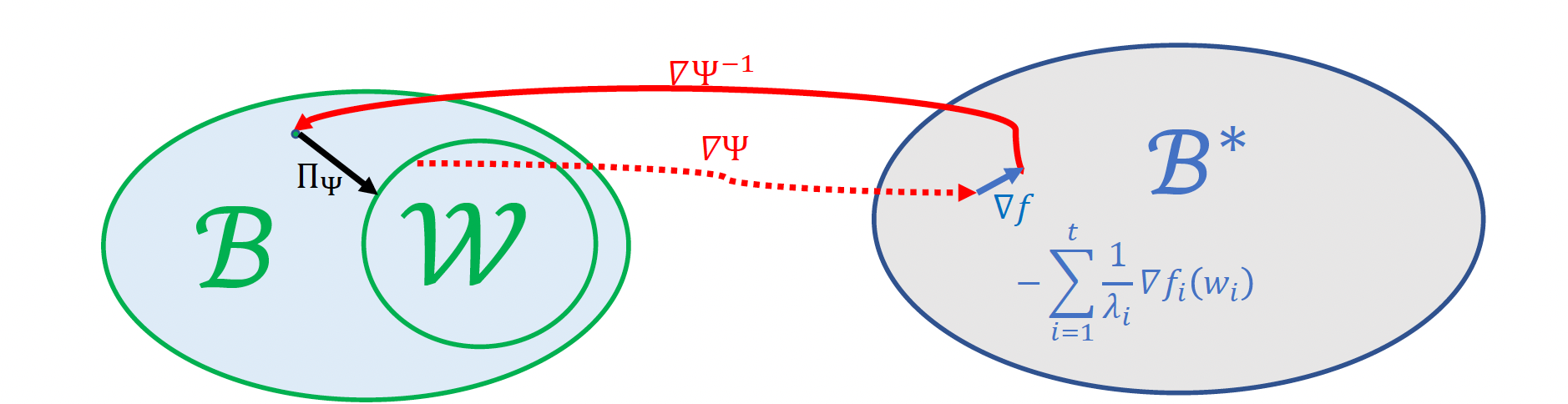}
    \caption{Graphical representation of mirror descent, see Figure \ref{fig:mirror}.}
    \label{fig:linkfunct}
\end{figure}

%\hl{include link function diagram}
When the step size is small, we can take a low-order approximation of the Bregman divergence:
$$
D_{\Psi}(w || w_k) \underset{w-w_k \text{small}}{ \approx } \langle \nabla \Psi(w_k), w-w_k \rangle - (w-w_k)^T \nabla^2 \Psi(w_k) (w-w_k) - \langle \nabla \Psi(w_k), w-w_k \rangle
$$
where we have ignored terms not dependent on $w\,,$ as they do not affect optimizing with respect to $w\,.$

We can then write the mirror descent problem approximately as:
\begin{align*}
w_{k+1} &= \arg \min_{w \in \mathcal{W}} \langle \nabla f(w_k, z_k), w \rangle + \frac{1}{\eta_k} (w-w_k)^T \nabla^2 \Psi(w_k) (w-w_k) \\
&= w_k - \eta_k \nabla^2 \Psi(w_k)^{-1} \nabla f(w_k, z_k)\,, \\
&\text{where} \quad \rho(w_k) \coloneqq \nabla^2 \Psi(w_k)^{-1}
\end{align*}

The latter version of the update rule is known as natural gradient descent, which was first popularized in the context of information geometry by S. Amari (see \cite{amari2012differential}). It is valid for any sufficiently regular potential $\Psi$ (or really for any metric tensor). For intuition, consider that there is a manifold with some local metric $\rho\,,$ which is the Hessian of $\Psi\,,$ and we optimize on this manifold. Thus, here we adopt the point of view that natural gradient descent is an approximation to mirror descent.

Next, we can take the step size to 0 and write the gradient flow equation:
$$
\dot{w} = - \nabla^2 \Psi(w(t))^{-1}  \nabla f(w(t), z)\,.
$$

As the step size goes to 0, the stochasticity goes away in the online learning framework: in any time interval, as the step size goes to 0, we take more steps and so we see more new samples, effectively minimizing the population loss. Thus, instead of writing the stochastic gradient, we can write the population gradient, denoting $F$ the population loss:
\begin{align}
\dot{w} = - \nabla^2 \Psi(w(t))^{-1}  \nabla F(w(t))\,. \label{eqn:naturalgd}
\end{align}
This is either very convenient because we do not have to worry about the stochasticity, or it is bothersome because we want to study the stochasticity but cannot.

There are two ways to interpret what is happening here: when the step size goes to 0, both mirror descent and natural gradient descent (which are the same in the infinitesimal step size limit) converge to the Riemannian gradient flow on the population objective. (Note: this is another way to think about what we are doing with SGD, because as the step size gets smaller and smaller, we {\em are} approaching optimizing the population loss.) The other way to think about natural gradient descent and mirror descent is: we endow the space with some local geometry and take steps based on that local geometry. Natural gradient descent corresponds to a forward Euler discretization of the natural gradient flow rule (Eqn.~\ref{eqn:naturalgd}):

\begin{align}
\dot{w}(t) &= - \nabla^2 \Psi(w(\lfloor t \rfloor_\eta))^{-1}
 \nabla F(w( \lfloor t \rfloor_\eta) )\\
 &= - \nabla^2 \Psi(w(\lfloor t \rfloor_\eta))^{-1}
 \nabla f(w( \lfloor t \rfloor_\eta), z_{\floor{t}_\eta)}
 \end{align}
 
%  The proof for this is by looking at the dual space and noticing that it is piecewise linear in the dual space. 
% \hl{guessing we will have to add the proof...}
 \begin{align}
 &w(k \eta + \Delta t) = \nabla \Phi^{-1}( \nabla \Psi(w(k\eta)) + \Delta t \nabla f(w(k\eta), z_k)) \\
 &\Leftrightarrow  \nabla \Psi(w(k \eta + \Delta t)) =  \nabla \Psi(w(k\eta)) + \Delta t \nabla f(w(k\eta), z_k)
 \end{align}
 Mirror descent represents a slightly more sophisticated discretization of this equation. Suppose we start with any arbitrary metric tensor
 \begin{align}
     \dot{w}(t) = - \rho(w(t))^{-1} \nabla F(w(t))\,.
 \end{align}
% \hl{I don't think I've picked up on the MD discretization}

%  In particular, suppose the metric tensor defining this manifold is not known a priori to be a Hessian for a potential. Unfortunately, m
What happens if we start with a manifold and seek a complexity measure to optimize over this manifold? In order to determine the complexity measure that underlies the geometry, we require the metric tensor to also be a Hessian map. Unfortunately, most metric tensors (smooth mapping from general vector space to $d \times d$ positive definite matrix) are not Hessian maps. In order to have it be a Hessian, we require that:
$$
\frac{\partial \rho_{ij}}{\partial w_k} = \frac{\partial \rho_{ik}}{\partial{w_j}}\,.
$$

This turns out to be almost sufficient as well. Suppose we define:
$$
\rho(w) = \mathbb{I} + w w^T\,.
$$
The manifold that we get from looking at the squared norm of $w$ appended to $w$ turns out to be the above. It turns out that this is not a Hessian map, as the above symmetry doesn't hold. Thus, starting from a metric tensor, it's quite special to actually be a Hessian map, and that is what allows us to determine the complexity measure that is underlying the geometry. To summarize, NGD is piecewise linear in the primal space, and MD is piecewise linear in the dual space.

\subsubsection{Examples of Mirror Descent}
Let us now look at how the dynamics look like for some examples of mirror descent. How should we choose the geometry? We know that:
$$
\E{S \sim D^m}{F(\bar{w}_k)} \le F(w^\star) + \O{\sqrt{\frac{\Psi(w^\star) \, \sup \knorm{\nabla f}_\star}{k}}}
$$
which holds as long as $\Psi$ is 1-strongly convex with respect to our choice of norm. Here, $w^\star$ is no longer the best in the norm bounded class, since we do not wish to limit the norm explicitly. The guarantee above means that we can compete with any $w^\star$, but we need to pay for its complexity in the second term.

\paragraph{Remark} When the metric tensor doesn't correspond to a Hessian, it is unclear wether we can identify a global complexity measure.

\paragraph{Example 1: Euclidean potential}

$$
\Psi(w) = \frac 12 \knorm{w}_2^2 \quad \text{strongly convex  wrt} \quad \knorm{w}_2\,.
$$
Here, $k \propto \knorm{w^\star}_2^2 \knorm{\phi(x)}_2^2\,.$ The metric tensor is just the identity, and the dynamics are the standard GD dynamics

\paragraph{Example 2: Mahalanobis potential}
% \textcolor{red}{(can I call it this?)}} 
$$
\Psi(w) = \frac12 w^T Q w\quad \text{strongly convex wrt} \knorm{w}_Q = \sqrt{w^T Q w}
$$
Here, $k \propto (w^{\star T} Q w^\star) \phi(x)^T Q^{-1} \phi(x)\,.$ The dynamics correspond to preconditioned gradient descent:
$$
\dot{w} = - Q^{-1} \nabla F(w)\,.
$$

\paragraph{Example 3: Entropic potential} 
$$
\Psi(w) = \sum_{i} w[i] \log \frac{w[i]}{i/d}
$$
here, $k \propto \knorm{w}_1^2 \log d \knorm{\phi(x)}_\infty^2\,.$ Here, the Hessian will be diagonal with $1/w$ on each component, and so the dynamics will look like:
$$
\dot{w} = - \nabla^2 \Psi(w)^{-1} \nabla F(w) = - \text{diag}(w) \nabla F(w) = - w_i \partial_i F(w)
$$

The local geometry is not constant, nor is it uniformly penalized in the same directions everywhere. It penalizes changing coordinates that are already small.

\subsubsection{Smoothness and Batching}

\begin{definition}
If $F(w') \le F(w) + \langle \nabla F(w), w'-w\rangle + \frac{H}{2} \knorm{w' - w}^2\,,$ we say that the convex function $F$ is $H$-smooth.
\end{definition}

For a smooth function, we can get a better convergence rate, which depends on $1/k$ rather than $1/\sqrt{k}\,:$
\begin{align} \label{eqn:smoothsogen}
    &\E{}{F(\bar{w}_k)} \le F(w^\star) + \O{\frac{H \Psi^\star(w)}{k} + \sqrt{\frac{\sigma^2 \Psi^\star(w)}{k}}}\,, \\
    &\text{where } \sigma \text{ s.t. } \E{}{\knorm{\nabla f(w,z) - \nabla F(z)}^2} \le \sigma^2 \notag
\end{align}
The first term above is the optimization term that only depends on the discretization (how well does linearization match the true function), and the second term is the stochasticity term, which depends on the distance between the stochastic approximation and the population objective. Accordingly, it makes sense to take more samples than discretization steps. We can do this via using batches. In particular, if instead we use:
$$
\frac 1b \sum_{i = 1}^b \nabla f(w, z_i)\,,
$$
then the variance scales with $1/b,$ giving us:
\begin{equation}
    \E{}{F(\bar{w}_k)} \le F(w^\star) + \O{\frac{H \Psi^\star(w)}{k} + \sqrt{\frac{\sigma^2 \Psi^\star(w)}{bk}}}\,.
\end{equation}

We can see here is that optimization is easier than the statistical aspect. It never really helps to take more steps than the number of samples. In the non-convex case, this is substantially different: in particular, the optimization might be harder than the statistical part, and so it might be worth viewing the samples multiple times. 

\paragraph{Remark} It is sufficient to look just at variance of the gradients at $w^\star\,.$ This is important because if the problem is realizable, then at the optimum, $w^\star$ is correct for any $z\,,$ and therefore the gradients are 0. So if $f \ge 0\,,$
$$
 \E{}{F(\bar{w}_k)} \le F(w^\star) + \O{\frac{H \Psi^\star(w)}{k} + \sqrt{\frac{H F(w^\star) \Psi^\star(w)}{bk}}}\,.
$$

% The last thing to say about this, related to the earlier question about manifolds that can't be written as Hessians. 
One more observation: this analysis relies not only on the choice of potential function but also on the choice of norm: the variance and smoothness depend on the norm, but the algorithm itself doesn't rely on the norm. Recent analyses alleviate this by going through relative smoothness \cite{bauschke2017-relativesmooth, lu2018-relativesmooth}. 

\begin{definition}
A function $F$ is relatively smooth to $\Psi$ up to smoothness parameter $H$ if:
$$
\nabla^2 F(w) \preceq H \nabla^2 \Psi(w)\,.
$$
\end{definition}

The guarantee in Eqn.~\ref{eqn:smoothsogen} holds when $H$ is the relative smoothness parameter.
% ; as of the lecture, the latter is a bit unclear.

\subsection{General Steepest Descent}

We can also think more generally about steepest descent methods that are not related to a metric,

$$
w_{k + 1} = \arg \min \langle \nabla f(w_k, z_k), w \rangle + \frac{1}{\eta}\delta(w, w')\,.
$$

We can take $\delta(w, w') = \knorm{w-w'}_1\,,$ for example, and though it appears to correspond to using $\ell_1$ geometry, 
it actually will correspond to coordinate descent. Using $\ell_\infty$ will take a step corresponding to the sign of the gradient.

\subsection{Implicit bias of descent methods}

So far, we have discussed the convex setting, and we obtained bounds combining optimization guarantees and generalization guarantees. We get generalization guarantees with respect to $w^\star$ considering only the dynamics of optimization. We studied the connection between the geometry of searching the space and the corresponding complexity. However, there are some limitations to what we have seen so far. First of all, this is only the convex case. Second of all, this isn't what we're seeing in deep learning. Here, we are not driving training error to 0. The generalization seems to rely on the fact that we are using stochasticity. That is different from what we see in the examples from the previous lectures where we train with full batch, multi-pass gradient methods but still get good generalization. 

Let us compare two approaches. First, suppose we actually select:

% 15:18 -- get this part from video
$$
w_{\lambda} = \arg \min \langle \nabla f(w_k, z_k), w \rangle + \lambda \Psi(w)\,.
$$

And secondly, consider $\bar{w}_k$, the solution achieved after $k$ iterations of mirror descent. In either of these cases, we get the same generalization guarantee, but for different solutions. In the Lipschitz case with optimally chosen $\lambda$, the distance (suboptimality, really) between $w_\lambda$ and $w^\star$ is $1/\sqrt{k}\,.$ The distance between $w^\star$ and $\bar{w}_k$ is also $1/\sqrt{k}\,,$ but the distance between $w_\lambda, \bar{w}_{k}$ is also $1/\sqrt{k}$. We know this because we know that we can view one-pass of this as optimizing the training objective, which in $k$ steps gets suboptimality $1/\sqrt{k}\,.$
% , and if I only do one pass, this is the same  \hl{get from video at 1:16:00. transcript unclear.} 
Thus, we are not saying that implicit regularization gets us to the same solution as explicit regularization, but rather we are saying that the generalization guarantees hold.
If we keep repeating passes, we might get to minimizer of the training error, but it's unclear if this is beneficial.

\subsubsection{Deriving Implicit Regularization for Gradient Descent}

We will analyze gradient descent on the unregularized training 
% \hl{the transcript at 1:20:41 says population?} 
objective:
$$
\hat{F}(w) = \frac 1m \sum_{i = 1}^m loss(\langle w, x_i \rangle, y_i )\,.
$$
We've dropped $\phi$ for ease of notation, but we should think of $x$ as being the output of some feature map. Let $w \in \mathbb{R}^d\,, \mbox{with} \thickspace d >> m\,.$ Let us use gradient descent:
$$
w_{k +1} = w_k - \eta_{k} \nabla \hat{F}(w_k)\, ; \quad w_0 = 0\,.
$$

We know from previous discussion that this will result in implicit $\ell_2$ regularization, but let us formally derive this. This will also allow us to understand what happens when moving to mirror descent later on.

For the first step, we are going to argue that the iterates $w_k$ all lie in a linear manifold given by the span of the data, i.e.:

$$
w_k \in \mathcal{M} = \text{span}(X) = \left \{ w = X^T s \bigg \rvert s \in \mathbb{R}^m \right \}
$$
where $X \in \mathbb{R}^{m \times d}$.
This simply comes from the fact that, for the model above, the gradient reads 
\begin{align}
  \nabla \hat{F}(w) &= \frac{1}{m}\sum_{i=1}^{m}loss'\left(\langle w,x_{i} \rangle, y_{i}\right)x_{i} \notag \\
  &= X^{\top}s,
\end{align}
where we introduced the vector of \emph{residuals} $s \in \mathbb{R}^{m}$ containing the entries $\frac{1}{m}loss'\left(\langle w,x_{i} \rangle, y_{i}\right)$. 
This already tells us a lot about what we will get: even though we have $d$ parameters, we only will ever be in an $m$-dimensional subspace. We will also assume that we get to a global optimum. This is not always easy to show but are going to assume it. Since we are in the overparametrized setting, we know that $Xw = y\,,$ where $w$ is the predictor to which we finally converge. We claim that the fixed point of this iteration is exactly the minimizer of the following optimization problem:

%Let us also consider the optimization problem:
\begin{equation} \label{eqn:implicitoptprob}
\min \frac 12 \knorm{x}_2^2 \quad \text{ such that } Xw = y\,.
\end{equation}

% \hl{figure out what the green problem is here: 1:24:00 ish in the video}

In general, these arguments are going to proceed by claiming that solving the first optimization problem actually implicitly solves the second optimization problem. To show this, we use the method of Karush-Kuhn-Tucker (KKT) conditions for optimality of a solution to a constrained optimization problem. The optimum for a constrained optimization problem is uniquely characterized by its KKT conditions.
Let us look at the KKT conditions. We introduce dual variables $\nu$ for the constraints, giving us that the Lagrangian is:
$$
L(w, \nu) = \frac 12 \knorm{w}_{2}^2 + \nu^T (y - Xw)\,.
$$

The KKT conditions are:
\begin{itemize}
    \item \textbf{Stationarity} $0 = \nabla_w L = w - X^T \nu$
    \item \textbf{Primal feasibility} $y = Xw\,.$ %$w = X^T \nu$
\end{itemize}

We observe that the stationarity condition shows exactly what we claimed earlier, that $w_k$ are in the span of the data, and the primal feasibility condition finds the interpolator. Since a point that satisfies these two conditions is optimal for the optimization problem in Eqn.~\ref{eqn:implicitoptprob}, and since gradient descent that arrives at a 0 training error predictor satisfies these two conditions, gradient descent finds an optimal solution to the problem in Eqn.~\ref{eqn:implicitoptprob}, i.e., a minimum Euclidean norm solution.

%\hl{should mention dual variables and primal variables: timestamps: 1:26:00-1:28:00}
% \hl{still don't quite understand if we are proving that iterates lie in the span of the data or showing that iterates lying in the span of the data are consistent with L2 implicit regularization.}

\subsubsection{Similar Argument for Mirror Descent}
With this argument in mind, we can apply a similar procedure to analyze the output of mirror descent. We are going to show that with mirror descent, we obtain a solution that implicitly minimizes the Bregman distance of the chosen potential with the starting point $w_{0}$. To show this, we will use essentially the same proof. The optimization problem is still the same as before: we are minimizing the training objective, but the optimization algorithm will differ. Now, we optimize $\hat{F}$ using mirror descent with respect to some potential function $\Psi$. Let us reproduce the iterates:
%\hl{TODO -- not written in lecture}
\begin{equation*}
w_{k + 1} = \arg \min_{w \in \mathcal{W}} \langle \nabla f(w_t, z_t), w \rangle + \lambda_t D_{\Psi}(w || w_t)
\end{equation*}

In mirror descent, we accumulate the gradients in the dual space. 
$$
w_{k+1} = \nabla \Psi^{-1} \bigg( \nabla \Psi(w_0) - \sum_{i = 1}^k \frac{1}{\lambda_{i}}\nabla \hat{F}(w_k)\bigg)
$$
That is, the $k+1$ iterate is the mapping back into the primal space of the accumulation of the gradients in the dual space.
As before, we can see that the gradients lie in the span of the data (in the dual space). Thus, again, $w_k \in \mathcal{M} = \{ \nabla \Psi^{-1}(\nabla \Psi(w_0) + X^T s) \, \forall \, x \in \mathbb{R}^m  \}\,.$ This is not a flat manifold (zero curvature) in the primal space, but it is in the dual space. To see what point we converge to, we use the assumption that in the end, we converge to a global optimum. This imposes $m$ linear constraints, which when intersected with an $m$-dimensional manifold gives us a unique point. To determine which unique point that is, we write an optimization problem whose KKT conditions match the two sets of constraints (global optimality and lying in the manifold):
% The following is such an optimization problem:
% \hl{green optimization problem:}
$$
\min D_{\Psi}(w || w_0) \quad \text{ such that } Xw = y\,.
$$
Then the Lagrangian is:
$$
L(w, \nu) = D_{\Psi}(w || w_0) + \nu(y - X\mu)
$$
To see this, we write the KKT conditions:
\begin{itemize}
    \item \textbf{Stationarity} $ 0 = \nabla \Psi(w) - \Psi(w_0) - X^{\top} \nu$
    \item \textbf{Primal Feasibility} Xw = y
\end{itemize}

This is exactly what was specified earlier.

\subsubsection{A General Method} We started from the optimization problem of minimizing the training objective and then saw trajectory stays within a manifold. The second set of constraints we imposed came from the global optimality of the solution reached. We then matched these two sets of constraints to KKT conditions for some other optimization problem. 
The goal of the next lecture will be to use this method to understand the implicit bias of optimization in various problems, and in particular attempt to highlight what parameters or architectures may lead to 
implicity $\ell_{2}$ regularization or $\ell_{1}$, i.e. feature learning.
\newpage
\section{Lecture 5: Implicit bias with linear functionals and the square loss}
The main topic of this lecture is the derivation of analytical evidence to the lazy \cite{chizat2019lazy} and rich regimes in learning problems with gradient descent, by employing the general method described in the previous lecture. We will consider simple model for which the dynamics of gradient descent may be solved exactly and we will study these trajectories 
as a function of the magnitude of the initialization and model architectures, leading to various implicit biases.
\subsection{Setting}
We consider models parametrized by a weight vector $\mathbf{w} \in \mathbb{R}^{p}$ acting on an input space $\mathcal{X}$ and denote $F(\mathbf{w}) \in \{f : \mathcal{X} \to \mathbb{R}\}$
the predictor implemented by $\mathbf{w}$, such that $F(\mathbf{w})(\mathbf{x}) = f(\mathbf{w},\mathbf{x})$. We will focus on models linear in $\mathbf{x}$ represented by a linear functional in the dual 
space of $\mathcal{X}$, denoted $\mathcal{X}^{*}$ and represented by a vector $\boldsymbol{\beta}_{\mathbf{w}}$ such that $f(\mathbf{w},\mathbf{x}) = \langle \boldsymbol{\beta}_{\mathbf{w}}, \mathbf{x} \rangle$, i.e. $\boldsymbol{\beta}_{\mathbf{w}}$ is some 
transform, potentially non-linear of $\mathbf{w}$. We consider the supervised learning problem with $n$ sample pairs $(\mathbf{x}_{i},y_{i})$ where $y_{i} \in \mathbb{R}$ is a response vector. The parameters are learned by minimizing the empirical loss 
\begin{equation}
    \frac{1}{n}\sum_{i=1}^{n}\mbox{loss}(f(\mathbf{w},\mathbf{x}_{i}),y_{i})
\end{equation}
with the corresponding population loss 
\begin{equation}
    \mathbb{E}_{\mathbf{x},y}\left[\mbox{loss}(f(\mathbf{w},y),\mathbf{x})\right].
\end{equation}
We assume the model is homogeneous of order $D$, i.e., for any constant $c>0$ $F(c\mathbf{w}) = c^{D}h_{\mathbf{w}}$. The order $D$ is related to
the depth of networks with homogeneous activations (e.g. a linear or Relu). A linear model is homogeneous of order $1$, a factorization model of order $2$.
Our objective is three-fold :
\begin{itemize}
    \item first, we want to characterize the implicit bias of gradient descent in this setup. In regards to the previous lectures, is optimizing in the parameters space using gradient descent equivalent to 
    optimizing in the functional space w.r.t. some metric tensor and potential? If this is the case, can we characterize this potential analytically?
    \item secondly, is this implicit optimization problem in function space equivalent to explicit regularization in parameter space? For instance, we saw that GD on a least-squares problem converges towards the minimum $\ell_{2}$
    norm solution, equivalent to explicitly penalizing the $\ell_{2}$ norm. Can we extend this picture to more general models, in particular models closer to deep learning architectures.
    \item finally, we would like to study the transition between kernel (a.k.a. lazy) regime and rich feature learning regime. What parameters govern this behaviour? We have seen in the previous lectures 
    that in many cases we obtain implicit biases that cannot be reached with kernel methods (sparsity, nuclear norm, ...). Can we get a more precise picture on simple models?
\end{itemize}

\subsection{Reminder on the kernel regime}
Consider the function computed by the model $f(\mathbf{w},\mathbf{x})$. We can take its first order approximation around the initalization of GD $\mathbf{w}_{0}$.
\begin{align}
    f(\mathbf{w},\mathbf{x}) = f(\mathbf{w}_{0},\mathbf{x})+\langle \mathbf{w}-\mathbf{w}_{0},\nabla_{\mathbf{w}}f(\mathbf{w}_{0},\mathbf{x})\rangle+\mathcal{O}\left(\norm{\mathbf{w}-\mathbf{w}_{0}}_{2}^{2}\right)
\end{align}
In what follows, we will sometimes write $\phi_{0}(\mathbf{x}) = \nabla_{\mathbf{w}}f(\mathbf{w}_{0},\mathbf{x})$. In certain regimes, this linear 
approximation is always valid across training and the model behaves as an affine model $f(\mathbf{w},\mathbf{x}) = f_{0}(\mathbf{x})+\langle \mathbf{w}, \phi_{0}(\mathbf{x}) \rangle$
with feature map $\nabla_{\mathbf{w}}f(\mathbf{w}_{0},\mathbf{x})$ corresponding to the tangent kernel $K_{0}(\mathbf{x},\mathbf{x}') = \langle \nabla_{\mathbf{w}}f(\mathbf{w}_{0},\mathbf{x}), \nabla_{\mathbf{w}}f(\mathbf{w}_{0},\mathbf{x}') \rangle $.
GD then learns the corresponding minimum RKHS distance to the initialization $F(\mathbf{w}_{0})$ solution, i.e. $\argmin_{h} \norm{h-F(\mathbf{w}_{0})}_{K_{0}} \thickspace \mbox{s.t.} \thickspace h(X)=\mathbf{y}$. Note that we can 
avoid the $F(\mathbf{w}_{0})$ by choosing a familiy of functions verifying $F(\mathbf{w}_{0}) = 0$ (see e.g. the unbiased initialization from \cite{chizat2019lazy}). Then are we just studying an 
uninformative regime or can we really replace neural nets with linear models? In what case does this kernel regime appear? \\
\par 
Initially, the appearance of the kernel regime was shown to be linked to the width of the network : taking the width to infinity under an appropriate scaling of the weights allows to linearize the network to obtain an asymptotically equivalent (in law) Gaussian process, leading to a kernel method and thus to the kernel regime. For the corresponding result, see e.g. \cite{jacot2018neural,daniely2017sgd}.
Closer to what we want to do is \cite{chizat2019lazy}. Regardless of the width, can always reach the kernel regime 
when the scale of the initialization goes to infinity. In what follows, we will mainly consider gradient flow, i.e. 
\begin{equation}
    \frac{d\mathbf{w}}{dt} = -\nabla_{\mathbf{w}}\hat{L}(\mathbf{w})
\end{equation}
Initialize at different scales, i.e. $\mathbf{w}(0) = \alpha \mathbf{w}_{0}$ where $\alpha >0$ and $\mathbf{w}_{0}$ can be 
any $\mathbf{w}_{0}$ which can be random, and such that $F(\mathbf{w}_{0}) = 0$ (i.e. $\mathbf{w}_{0}$ maps to a null function in function space). For any 
$\alpha$, we will write the dynamics 
\begin{equation}
    \frac{d\mathbf{w}_{\alpha}}{dt} = -\nabla_{\mathbf{w}}\hat{L}(\mathbf{w}_{\alpha})
\end{equation}
The result from \cite{chizat2019lazy} then states that when $\alpha$ goes to infinity, after a appropriate rescaling of time, the entire trajectory converges to the 
kernel one :  
\begin{equation}
    \lim_{\alpha \to \infty} \sup_{t} \norm{\mathbf{w}_{\alpha}(\frac{1}{\alpha^{D-1}}t)-\mathbf{w}_{K}(t)}_{\infty} = 0
\end{equation}
where $\mathbf{w}_{K}$ represents the vector obtained by running gradient descent on the tangent kernel model : $\dot{\mathbf{w}}(K) = -\nabla_{\mathbf{w}}\hat{L}(\langle \mathbf{w}, \phi_{0}(\mathbf{x}) \rangle)$.

\subsection{A simple model : 2-layer linear diagonal network}
We would like a simple model going beyond the linear case (for which we understand the phenomenology), that can still be studied analytically and where the 
implicit bias is interesting. We are going to look at element-wise squaring of parameters, i.e. 
\begin{equation}
    f(\mathbf{w},\mathbf{x}) = \langle \mathbf{w}^{2},\mathbf{x} \rangle \quad \beta = F(\mathbf{w}^{2})
\end{equation}
where for a given vector, $^{2}$ denotes the elementwise squaring.
This is equivalent to a depth 2 diagonal linear network. Here we cannot get all linear functions this way 
because the coefficients cannot be negative. In order to allow negative coefficients, we introduce the $2d$ parameters $\mathbf{w} = \begin{bmatrix} \mathbf{w}_{+} \\ \mathbf{w}_{-} \end{bmatrix}$ and $\boldsymbol{\beta}_{\mathbf{w}} = \mathbf{w}^{2}_{+}-\mathbf{w}^{2}_{-}$. This is equivalent 
to a depth-2 diagonal linear network with $2d$ parameters on the first layer $f(\mathbf{w},\mathbf{x}) = \mathbf{w}^{\top}\mbox{diag}(\mathbf{w})\begin{bmatrix}\mathbf{x} \\ -\mathbf{x}\end{bmatrix}$, as in the following figure
\begin{figure}[h]
    \centering
    \includegraphics[scale=0.4]{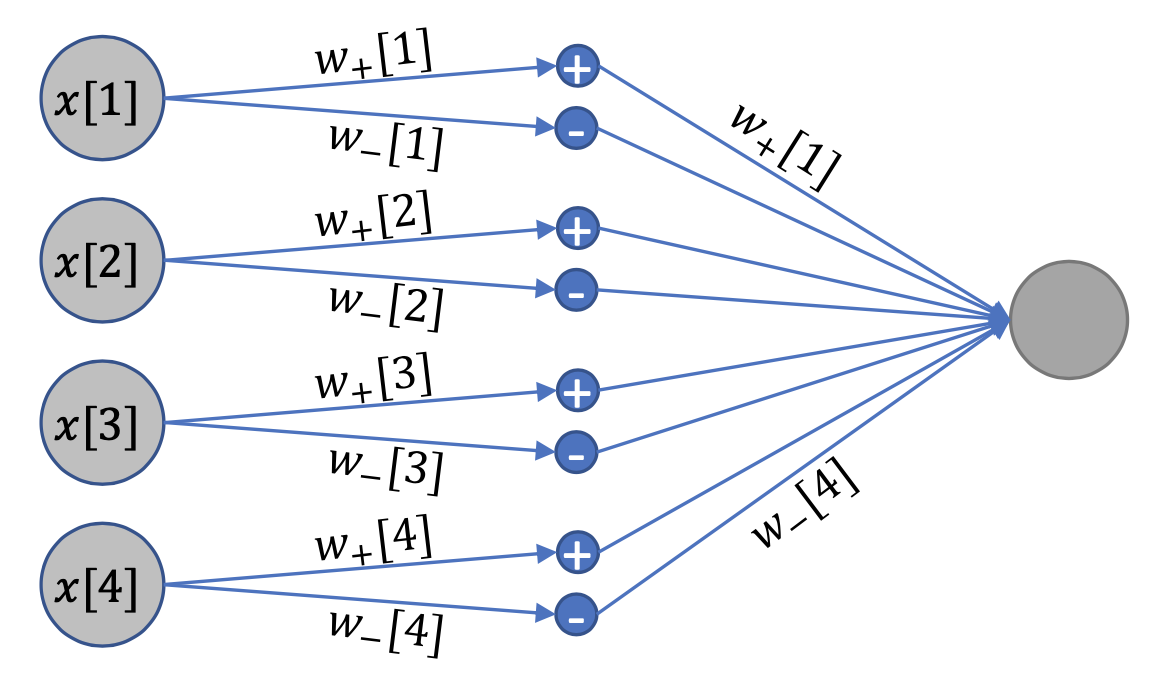}
    \caption{Depth-2 diagonal linear net with replicated an signed units in the first layer}
    \label{fig:diag_net}
\end{figure}

Negative functions can thus be represented and we can also choose the initialization such that $F(\mathbf{w}_{0}) = 0$. If $\mathbf{w}_{0} = 1_{d}$, 
then at initialization, $\boldsymbol{\beta}_{0} = 0$ whatever the scale of initialization $\alpha$, where $\mathbf{w}_{+,\alpha}(0) = \mathbf{w}_{-,\alpha}(0) = \alpha \mathbf{w}_{0}$. What's the implicit bias of doing gradient descent on this model, when considering the square loss?
\begin{equation}
    L(\hat{y},y) = \frac{1}{2}(\hat{y}-y)^{2}
\end{equation}
\subsubsection{Analytical study of GD in parameter space}
Applying the chain rule gives, denoting $\mathbf{X} \in \mathbb{R}^{n \times d}$ the design matrix,
\begin{align}
    &\dot{\mathbf{w}}_{+} = -\frac{d \boldsymbol{\beta}}{d \mathbf{w}^{+}}\frac{d L(\boldsymbol{\beta})}{d \boldsymbol{\beta}} = -2\mathbf{X}^{\top}\mathbf{r}(t)\odot\mathbf{w}_{+}(t)\\
    &\dot{\mathbf{w}}_{-} = -\frac{d \boldsymbol{\beta}}{d \mathbf{w}^{-}}\frac{d L(\boldsymbol{\beta})}{d \boldsymbol{\beta}} = 2\mathbf{X}^{\top}\mathbf{r}(t)\odot\mathbf{w}_{-}(t),
\end{align}
where we defined the residual $\mathbf{r}(t) = \mathbf{X}\boldsymbol{\beta}(t)-\mathbf{y}$ and $\odot$ denotes the elementwise product.

Assume that the residuals are known, then these are differential equations that can be solved. Integrating yields 
\begin{align}
    \label{eq:exp_w1}
    \mathbf{w}_{+}(t) &= \mathbf{w}_{+}(0)\odot\exp\{-2\mathbf{X}^\top\int_{0}^{t}\mathbf{r}(\tau)d\tau\} \\
    \mathbf{w}_{-}(t) &= \mathbf{w}_{-}(0)\odot\exp\{2\mathbf{X}^\top\int_{0}^{t}\mathbf{r}(\tau)d\tau\} \label{eq:exp_w2}
\end{align}
where $\mathbf{r}(t) = \mathbf{X}^{\top}\left(\mathbf{w}_{+}^{2}-\mathbf{w}_{-}^{2}\right)-\mathbf{y}$. Although this is just a rewriting as integral equations, 
we can get important information from this. Letting $\mathbf{S}(t) = \int_{0}^{t}\mathbf{r}(\tau)d\tau$:
\begin{align}
    \label{eq:reduc_dim1}
    \boldsymbol{\beta}(t) &= \mathbf{w}^{2}_{+}(0)\odot \exp\left(-4\mathbf{X}^{\top}\mathbf{S(t)}\right)-\mathbf{w}^{2}_{-}(0)\odot \exp\left(4\mathbf{X}^{\top}\mathbf{S(t)}\right) \\
    &= \alpha^{2}1_{d}\odot \left(\exp\left(-4\mathbf{X}^{\top}\mathbf{S(t)}\right)-\exp\left(4\mathbf{X}^{\top}\mathbf{S(t)}\right)\right) \label{eq:reduc_dim2} \\
    &= 2\alpha^{2}1_{d}\odot \mbox{sinh}\left(-4\mathbf{X}^{\top}\mathbf{S(t)}\right) \label{eq:reduc_dim3}
\end{align}
where sinh is the hyperbolic sine.
We are interested in the regime where $n << d$ i.e. where the number of samples is much lower than the dimensionality $d$, so they are many solutions to the problem $\mathbf{X}\boldsymbol{\beta} = \mathbf{y}$. However, 
Eq.\eqref{eq:reduc_dim1}-\eqref{eq:reduc_dim2} shows that we have reduced the set of solutions to a lower dimensional manifold of dimension $n$. This is similar to the study of GD on linear regression, i.e. 
the predictor is always spanned by the data. Here we observe the same thing on a non-linear model. This manifold is defined by 
\begin{align}
    \mathcal{M} &= \left\{\boldsymbol{\beta} = \alpha^{2}1_{d}\odot\left(\exp\left(-4\mathbf{X}^{\top}\mathbf{s}\right)-\exp\left(4\mathbf{X}^{\top}\mathbf{s}\right)\right)\vert \mathbf{s} \in \mathbb{R}^{n}\right\} \\
    &= \left\{\boldsymbol{\beta} = 2\alpha^{2}1_{d}\odot \mbox{sinh}\left(-4\mathbf{X}^{\top}\mathbf{s}\right)\vert \mathbf{s} \in \mathbb{R}^{n}\right\}
\end{align}
where we have $n$ additional constraints defined by the equation $\mathbf{X}\boldsymbol{\beta} = \mathbf{y}$ leading to a unique solution. Recall that 
we are not studying the convergence of gradient flow per se, rather we assume it converges and we study the corresponding fixed point to characterize its implicit bias. Let's write an equivalent optimization problem giving the same 
set of solutions. We want to find a function $Q_{\alpha}(\boldsymbol{\beta})$ that is implictly minimized by the GF dynamics such that 
\begin{align}
    \boldsymbol{\beta^{*}} &\in \min_{\boldsymbol{\beta}} Q_{\alpha}(\boldsymbol{\beta}) \\
    &\mbox{s.t.} \quad \mathbf{X}\boldsymbol{\beta} = \mathbf{y}
\end{align}
The Lagrangian formulation for this problem reads 
\begin{align}
    \min_{\boldsymbol{\beta}} \max_{\boldsymbol{\nu}} Q_{\alpha}(\boldsymbol{\beta})+\boldsymbol{\nu}^{\top}\left(\mathbf{y}-\mathbf{X}\boldsymbol{\beta}\right)
\end{align}
The KKT conditions then read 
\begin{align}
  \label{eq:KKT_mirror}
    \mathbf{X}\boldsymbol{\beta} &= \mathbf{y} \\
    \nabla_{\boldsymbol{\beta}}Q_{\alpha}(\boldsymbol{\beta}) &= \mathbf{X}^{\top}\boldsymbol{\nu}
\end{align}
Using Eq.\eqref{eq:reduc_dim3}, we may write 
\begin{align}
    \mbox{sinh}^{-1}\left(\frac{1}{2\alpha^{2}}\boldsymbol{\beta}\right) = -4\mathbf{X}^{\top}\mathbf{s} 
\end{align}
where the inverse hyperbolic sine is applied elementwise. Matching this with the optimality condition Eq.\eqref{eq:KKT_mirror} then gives 
\begin{align}
\boldsymbol{\nu} &= -4\int_{0}^{\infty}\mathbf{r}_{\alpha}(t)\left(=-4\mathbf{s}\right) \\
\nabla_{\boldsymbol{\beta}}Q_{\alpha}(\boldsymbol{\beta}) &= \mbox{sinh}^{-1}\left(\frac{1}{2\alpha^{2}}\boldsymbol{\beta}\right)
\end{align}
Integrating this expression, remembering that the gradient of an element-wise function is applied element-wise, yields 
\begin{equation}
  \label{eq:potential_1}
    Q_{\alpha}(\boldsymbol{\beta}) = \sum_{i=1}^{d}\alpha^{2}q(\frac{\beta_{i}}{\alpha^{2}})
\end{equation}
where $q(z) = \int_{0}^{z}\mbox{sinh}^{-1}(\frac{t}{2})dt = 2-\sqrt{4+z^{2}}+z\mbox{sinh}^{-1}(\frac{z}{2})$. \\
We can now study this function for different scalings of $\alpha$. Plotting $q$ gives the following figure :
\begin{figure}[h]
    \centering
    \includegraphics[scale=0.6]{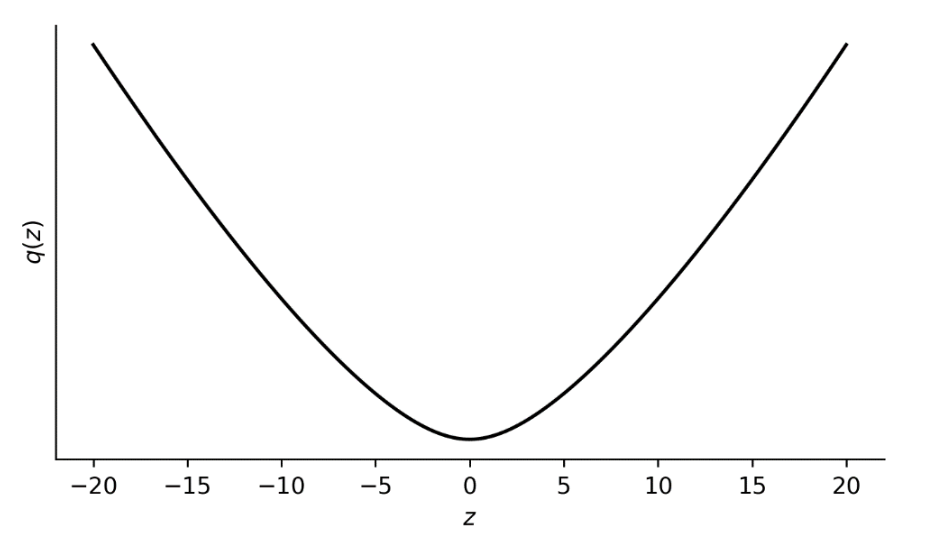}
    \caption{$q(z) = \int_{0}^{z}\mbox{sinh}^{-1}(\frac{t}{2})dt = 2-\sqrt{4+z^{2}}+z\mbox{sinh}^{-1}(\frac{z}{2})$}
\end{figure} 
\quad \\
For $\alpha \to \infty$, we may look at $q$ around zero, a second order Taylor expansion shows that $q$ is quadratic around zero. Thus for large $\alpha$, 
the regularization is effectively quadratic on each coordinate. In this model, we can show that the tangent kernel at initialization is the linear kernel
$K_{0}(\mathbf{x}, \mathbf{x}') = \langle \mathbf{x},\mathbf{x'} \rangle$ thus the solution converges to the minimum $\ell_{2}$ norm solution
\begin{equation}
    \boldsymbol{\beta}_{\alpha}(\infty)\xrightarrow[]{\alpha \to \infty}\hat{\boldsymbol{\beta}}_{L_{2}}  = \argmin_{\mathbf{X}\boldsymbol{\beta} = \mathbf{y}} \norm{\boldsymbol{\beta}}_{2} \quad \mbox{Kernel regime}
\end{equation}  
For small values of $\alpha$ however, $q$ becomes close to a $\ell_{1}$ norm and we obtain an effective $\ell_{1}$, sparsity inducing implicit regularization  
which does not correspond to a kernel regime, but a rich regime : 
\begin{equation}
    \boldsymbol{\beta}_{\alpha}(\infty)\xrightarrow[]{\alpha \to 0}\hat{\boldsymbol{\beta}}_{L_{1}}  = \argmin_{\mathbf{X}\boldsymbol{\beta} = \mathbf{y}} \norm{\boldsymbol{\beta}}_{1} \quad \mbox{Rich regime}
\end{equation}
We thus have a transition from a kernel inductive bias to a sparsity inducing inductive bias. The rich learning regime corresponds to the $\ell_{1}$ regime, 
which is one of the main benefits we are looking for in a machine learning model : learning features corresponds to selecting a small number of important features among an infinite amount of features (in the infinite size limit for neural networks).
A more precise characterization of this transition can be found in Theorem 2 from \cite{woodworth2020kernel}, which we reproduce here.
\begin{theorem}[Theorem 2 from \cite{woodworth2020kernel}]
    \label{th:scaling}
    For $0<\epsilon<d$, under the above setting,
    \begin{align*}
        &\alpha \leqslant \min\left\{\left(2(1+\epsilon)\norm{\boldsymbol{\beta}^{*}_{\ell_{1}}}_{1}\right)^{-\frac{2+\epsilon}{2\epsilon}},\exp(-d/(\epsilon)\norm{\boldsymbol{\beta}^{*}_{\ell_{1}}}_{1})\right\} \implies \norm{\boldsymbol{\beta}^{\infty}_{\alpha,1}}_{1}\leqslant (1+\epsilon)\norm{\boldsymbol{\beta}^{*}_{\ell_{1}}}_{1} \\
        &\alpha \geqslant \sqrt{2(1+\epsilon)(1+2/\epsilon)\norm{\boldsymbol{\beta}_{\ell_{2}}^{*}}}_{2} \implies \norm{\boldsymbol{\beta}^{\infty}_{\alpha,1}}_{2}^{2} \leqslant (1+\epsilon)\norm{\boldsymbol{\beta}^{*}_{\ell_{1}}}_{2}^{2}
    \end{align*}
\end{theorem}
The sparsity inducing bias is fundamentally different from the kernel regime. Consider the following example of sparse regression 
\begin{align}
    y_{i} = \langle \boldsymbol{\beta}^{*}, \mathbf{x}_{i} \rangle+\gamma \quad \mbox{where} \quad \gamma \sim \mathcal{N}(0,0.01)
\end{align}
where $d=1000; \norm{\boldsymbol{\beta}^{*}}_{0} = 5$ and we have $n=100$ samples. A kernel method cannot solve this problem : we can see this on the following 
figure 
\begin{figure}[h]
    \centering
    \includegraphics[scale=0.3]{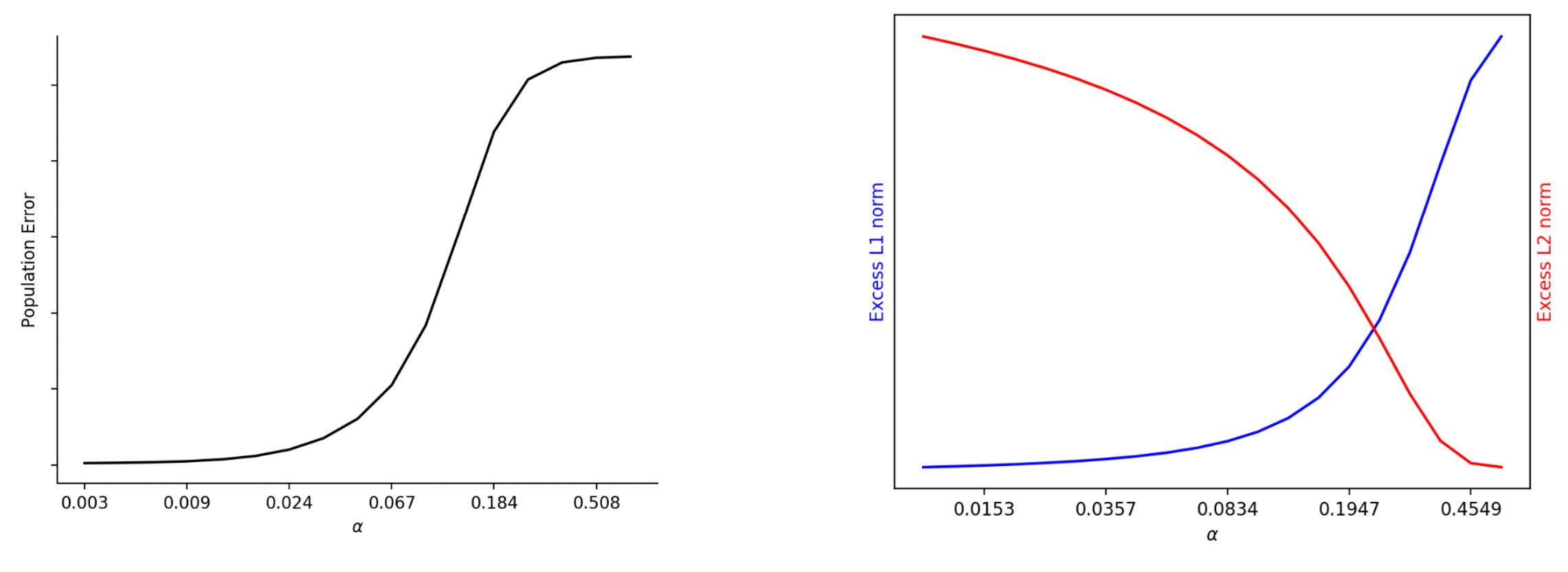}
    \caption{Generalization error of kernel and rich regime on a sparse regression problem with $n << d$}
\end{figure}
For large $\alpha$, we are in the kernel regime and the excess $\ell_{2}$ norm is small, but the population error is large. For small alpha however, 
both the excess $\ell_{1}$ norm and the population error are small. \\
Getting to the $\ell_{1}$ regime is actually quite difficult. Indeed, looking at the thresholds given by Theorem \ref{th:scaling} : $\alpha$ has to be exponentially small. What's the sample complexity as a function 
of $\alpha$, or conversely, how small does $\alpha$ has to be to reach good performance, let's say $L(\boldsymbol{\beta}_{\alpha}(\infty)) \leqslant 0.025$
\begin{figure}[h]
    \centering
    \includegraphics[scale=0.3]{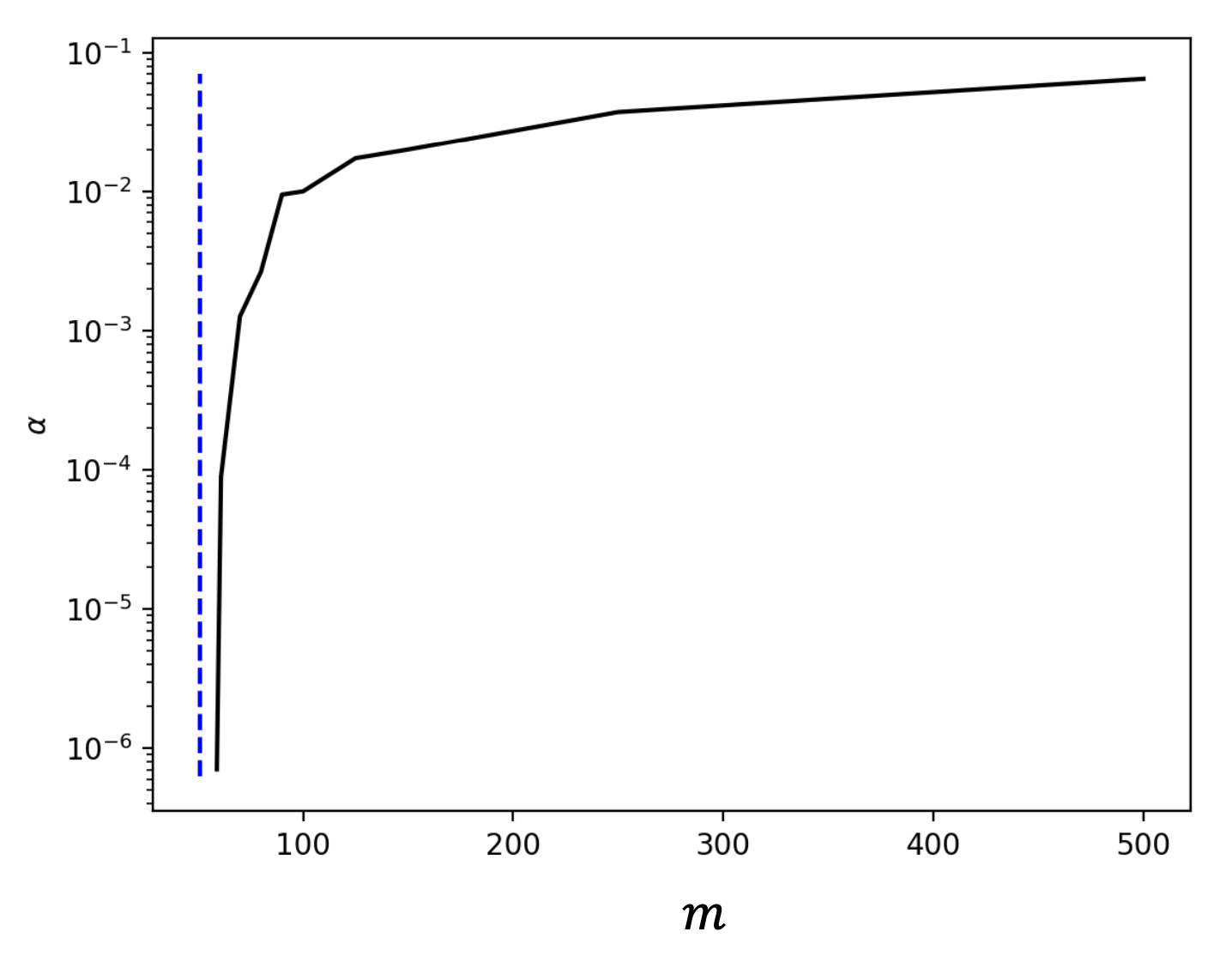}
    \caption{Threshold value of alpha varies with the number of samples}
\end{figure}
Thus, for very low number of samples, doing the exact $\ell_{1}$ is impossible (vertical asymptote), but for a reasonable 
amount of samples, we can get good performance with an approximate $\ell_{1}$. This concludes the link between the scaling of initialization $\alpha$ to the kernel and rich regimes.

\subsubsection{Studying the dynamics in function space}
Whats do the dynamics look like in function space ? Recall that the function space is parametrized by $\boldsymbol{\beta}$, thus we want to write the solution 
directly as a function of $\boldsymbol{\beta}$ instead of $\mathbf{w}$, where $\boldsymbol{\beta} = F(\mathbf{w})$. Writing the dynamics at the level of $\beta$, we reach
\begin{align}
    \dot{\boldsymbol{\beta}} = \frac{d \boldsymbol{\beta}}{d\mathbf{w}}\dot{\mathbf{w}} &= \nabla F(\mathbf{w}(t))^{\top}\dot{\mathbf{w}} \\
    &=\nabla F(\mathbf{w}(t))^{\top}\left(-\nabla_{\mathbf{w}}L(\mathbf{w}(t))\right),
\end{align}
where $\nabla F(\mathbf{w}(t)) \in \mathbb{R}^{p\times p}$ is the Jacobian of $F$. The chain rule then gives
\begin{align}
    \nabla_{\mathbf{w}}L(\mathbf{w}(t)) &= \nabla_{\mathbf{w}}L(\boldsymbol{\beta}(\mathbf{w}(t))) \\
    &=\nabla F(\mathbf{w}(t))\nabla L(\boldsymbol{\beta})
\end{align}
giving the dynamics 
\begin{equation}
    \dot{\boldsymbol{\beta}} = -\nabla F(\mathbf{w}(t))^{\top}\nabla F(\mathbf{w}(t))\nabla L(\boldsymbol{\beta}).
\end{equation}
This corresponds to the previously discussed Riemanian gradient flow (or natural gradient descent), with a metric tensor determined by $\rho = \left(\nabla F(\mathbf{w}(t))^{\top}\nabla F(\mathbf{w}(t))\right)^{-1}$, i.e. 
\begin{equation}
    \dot{\boldsymbol{\beta}} = -\rho^{-1}\nabla L(\boldsymbol{\beta}).
\end{equation}
Thus choosing a certain parametrization $F$ induces a geometry in the search done by the gradient descent, 
governed by the metric tensor $\rho$. But $\rho$ is a function of $\mathbf{w}(t)$: in the case of the model 
discussed previously, 
\begin{align}
    \nabla F(\mathbf{w}(t)) &= \begin{bmatrix}\mbox{diag}(\mathbf{w}_{+})\\ \mbox{diag}(\mathbf{w}_{-})\end{bmatrix} \in \mathbb{R}^{2d \times d} \\
    \rho(\mathbf{w}(t)) &= \mbox{diag}\left(\mathbf{w}_{+}^{2}+\mathbf{w}_{-}^{2}\right)^{-1}.
\end{align}
We would now like to rewrite this entirely as functions of $\boldsymbol{\beta}$, i.e. write 
\begin{equation}
    \dot{\boldsymbol{\beta}} = - \rho(\boldsymbol{\beta})\nabla L(\boldsymbol{\beta})
\end{equation}
This is problem dependent, and possible here. Recall the dynamics on $\mathbf{w}_{+}$ and $\mathbf{w}_{-}$ from the previous section, we have 
\begin{align}
    \frac{d}{dt}\left(\mathbf{w}_{+}\odot \mathbf{w}_{-}\right) = -2\mathbf{X}^{\top}\mathbf{r}(t)\left(\mathbf{w}_{+}\odot \mathbf{w}_{-}-\mathbf{w}_{-}\odot \mathbf{w}_{+}\right) = 0
\end{align}
We can thus evaluate this quantity at $t=0$ which gives 
\begin{equation}
    \label{eq:e1}
    \forall t, \thickspace \mathbf{w}_{+}\odot \mathbf{w}_{-} = \alpha^{2}1_{d}
\end{equation}
Note that this also appears immediately by considering Eq.\eqref{eq:exp_w1}-\eqref{eq:exp_w2} along with the fact that $\mathbf{w}_{+}(0)=\mathbf{w}_{-}(0) = \alpha 1_{d}$. Furthermore, 
by definition of $\boldsymbol{\beta}(t)$
\begin{equation}
    \label{eq:e2}
    \boldsymbol{\beta}(t) = \mathbf{w}_{+}^{2}-\mathbf{w}_{-}^{2},
\end{equation}
which leads to an element-wise quadratic equation that we can solve. Since the equations are the same for each coordinate, we drop the coordinate index. Squaring both sides of Eq.\eqref{eq:e1} and replacing $w_{+}^{2}(t)$ with $\beta(t)+w_{-}^{2}(t)$ using Eq.\eqref{eq:e2}, we obtain 
\begin{equation}
    w_{-}^{4}(t)-\beta w_{-}^{2}(t)+\alpha^{4}=0.
\end{equation}
This is a quadratic equation in $w_{-}^{2}$ whose positive solution reads
\begin{equation}
    w_{-}^{2} = \frac{-\beta+\sqrt{\beta^{2}+4\alpha^{4}}}{2}.
\end{equation}
$w_{+}^{2}$ is obtained in similar fashion, leading to 
\begin{equation}
    w_{+}^{2} = \frac{\beta+\sqrt{\beta^{2}+4\alpha^{4}}}{2}.
\end{equation}
This leads to the following expression for the metric tensor $\rho$ and the corresponding dynamics
\begin{align}
    \rho^{-1} &= \mbox{diag}\left(\sqrt{\beta^{2}+4\alpha^{4}}\right)^{-1} \\
    \dot{\boldsymbol{\beta}} &= -\mbox{diag}\left(\sqrt{\beta^{2}+4\alpha^{4}}\right)^{-1} \odot \nabla L(\boldsymbol{\beta})
\end{align}
We can recover the previously discussed phenomenology from this equation. For large $\alpha$, the $\beta^{2}$ term is negligible and we recover standard gradient flow dynamics with $\ell_{2}$ geometry in the function space. If $\alpha =0$, the scaling in front of the gradient is 
proportional to the absolute value of $\beta$. Thus higher absolute value coefficients will decay faster, which will promote sparsity.
Now, does this metric tensor correspond to a Hessian map defining a mirror descent? To check this we need to solve $\rho = \nabla^{2} \Psi$ where $\Psi$ is the potential defining the Bregman distance used for mirror descent.
In general, a metric tensor is not a Hessian map, but here it is the case, mostly thanks to the diagonal structure. We can then simply integrate each element on the diagonal twice.
Performing this double integral yields the following potential 
\begin{equation}
    \Psi(\alpha,\beta) = \alpha^{2}\sum_{i=1}^{d}\left(\frac{\beta}{2\alpha^{2}}\mbox{sinh}^{-1}\left(\frac{\beta}{2\alpha^{2}}\right)-\sqrt{4+\frac{\beta^{2}}{\alpha^{4}}}\right)
\end{equation}
Up to a constant, this is the same potential as the one implicitly being minimized by the gradient descent, as established at Eq.\ref{eq:potential_1}.
\subsubsection{Comparing explicit and implicit regularization}
Is what we have established in the previous section equivalent to explicit regularization using the $\ell_{2}$ norm in parameter space?
\begin{align}
    \boldsymbol{\beta}_{\alpha,\mathbf{w}_{0}}^{R} = F\left(\argmin_{\mathbf{w}}\norm{\mathbf{w}-\alpha\mathbf{w}_{0}}_{2}^{2} \thickspace \mbox{s.t.} \thickspace L(\mathbf{w})=0\right)\\
     = \argmin_{\boldsymbol{\beta}} R_{\alpha,\mathbf{w}_{0}}(\boldsymbol{\beta}) \thickspace \mbox{s.t.} \mathbf{X}\boldsymbol{\beta} = \mathbf{y} \\
     \mbox{where} \quad R_{\alpha,\mathbf{w}_{0}} = \min_{\mathbf{w}} \norm{\mathbf{w}-\alpha\mathbf{w}_{0}}_{2}^{2} \thickspace \mbox{s.t.} \thickspace F(\mathbf{w})=\boldsymbol{\beta}.
\end{align}
Using a similar analysis as before, we study the optimization problem in parameter space over $\boldsymbol{\beta}$ which is equivalent to the optimization problem 
in weight space over $\mathbf{w}$ where we use explicit $\ell_{2}$ regularization. For standard linear regression this is a classical result, gradient descent converges to the min 
$\ell_{2}$ norm solution (in that case $\boldsymbol{\beta} = \mathbf{w}$). When the initialization is $\mathbf{w}_{0} = \mathbf{1}$, one can determine an analytical expression for 
$R_{\alpha,\mathbf{w}_{0}}$:
\begin{equation}
    R_{\alpha,\mathbf{1}}(\boldsymbol{\beta}) = \sum_{i}r(\boldsymbol{\beta}_{i}/\alpha^{2})
\end{equation}
where $r(z)$ is the unique real root of $p_{z}(u) = u^{4}-6u^{3}+(12-2z^{2})u^{2}-(8+10z^{2})u+z^{2}+z^{4}$. The next figure shows a plot of $r(z)$ next to $q(z)$ obtained from the 
implicit bias analysis. 
\begin{figure}[h]
    \centering
    \includegraphics[scale=0.4]{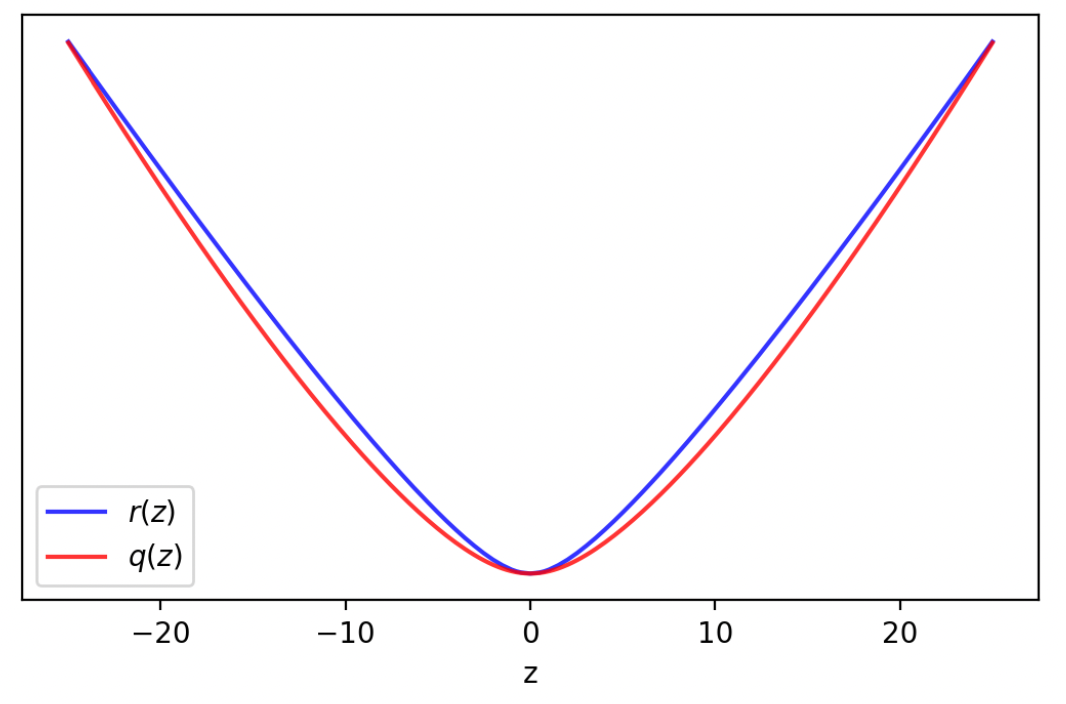}
    \caption{Comparing explicit and implicit regularization}
\end{figure}
The functions are very close to one another, even if a more refined analysis shows that the rich regime can be reached 
with a polynomial scale of $\alpha$ with $r$ instead of an exponential one with $q$. See the discussion in \cite{woodworth2020kernel} for more detail.
\subsection{The effect of width}
For now we have studied the effect of the scale of the initialization on the regime in which the dynamics operate. What is the effect of the width ?
To find out, consider now that the model we want to learn is the function
\begin{align}
f((\mathbf{U},\mathbf{V}),\mathbf{x}) = \sum_{i=1,..,d,j=1,..,k}u_{i,j}v_{i,j}x[i] = \langle \mathbf{U}\mathbf{V}^{\top},\mbox{diag}(\mathbf{x}) \rangle
\end{align}
where $\mathbf{U}, \mathbf{V} \in \mathbb{R}^{d \times k}$, and we learn the model by minimizing 
\begin{align}
    L(\mathbf{U},\mathbf{V}) = \sum_{n=1}^{N}\left(\langle \mathbf{X}_{n},\mathbf{M}_{\mathbf{U},\mathbf{V}}\rangle -y_{n}\right)^{2} = \tilde{L}(\mathbf{M}_{\mathbf{U},\mathbf{V}})
\end{align}
using gradient flow, and we defined the map $M_{\mathbf{U},\mathbf{V}} = F(\mathbf{U},\mathbf{V}) = \mathbf{U}\mathbf{V}^{\top}$ in the notations of the previous section.
This model can be considered as an extension of the linear model from above over matrix valued observations, with an additional 
width parameter $k$, i.e. a matrix factorization problem or a wide parallel linear network.
\begin{figure}[h]
    \centering
    \includegraphics[scale=0.4]{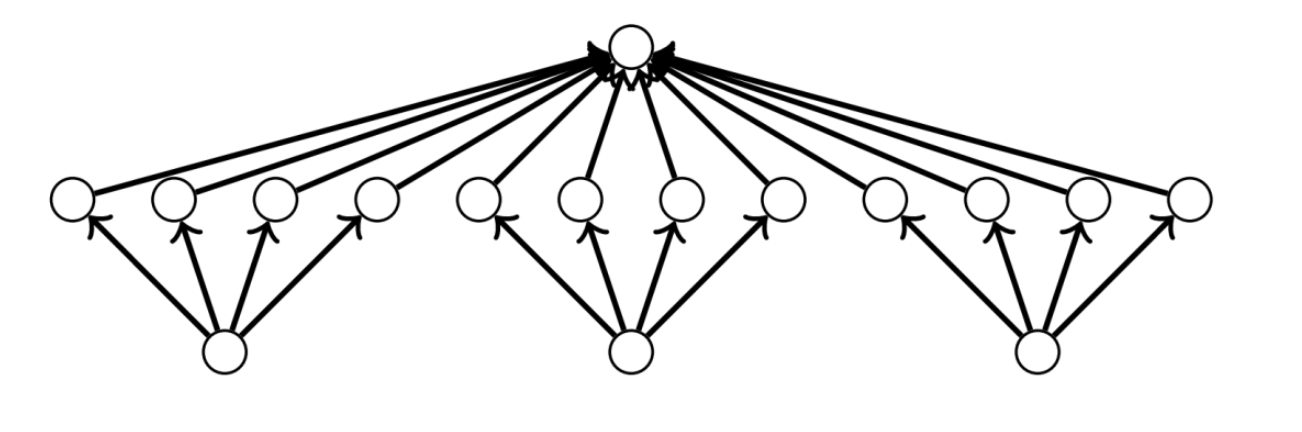}
    \caption{Wide parallel linear network}
\end{figure}
The goal is to study the combined effect of $\alpha$ and $k$ on the learning regime.
Since the number of parameters grows with the width $k$, we capture the scale of initialization with the parameter 
\begin{equation}
   \frac{1}{d}\norm{\mathbf{M}_{\mathbf{U},\mathbf{V}}}_{F}
\end{equation}
We will see that $\mathbf{M}_{\mathbf{U},\mathbf{V}}$ can be in the kernel regime even if $\sigma$ goes to $0$, depending on the relative scaling with $k$.
In the symmetric case where 
$\mathbf{M}_{\mathbf{W}} = \mathbf{W}\mathbf{W}^{\top}$, the gradient flow reads
\begin{equation}
    \dot{\mathbf{M}}_{\mathbf{W}(t)} = \nabla \tilde{L}(\mathbf{M}_{\mathbf{W}(t)})\mathbf{M}_{\mathbf{W}(t)}+\mathbf{M}_{\mathbf{W}(t)}\nabla\tilde{L}(\mathbf{M}_{\mathbf{W}(t)})
\end{equation}
thus the entire dynamics is described by $\mathbf{M}_{\mathbf{W}\mathbf{W}^{\top}}$. In the asymmetric case $\mathbf{M}_{\bU,\bV}$ this is not true. We may then 
consider the following \emph{lifted} problem defined by 
\begin{equation}
    \bar{\bM}_{\bU,\bV} = \begin{bmatrix}\bU\bU^{\top}&\bM_{\bU,\bV} \\ \bM_{\bU,\bV} & \bV \bV^{\top} \end{bmatrix}
\end{equation}
and the corresponding lifted datapoints $\bar{\bX}_{n} = \frac{1}{2}\begin{bmatrix}0& \bX_{n} \\ \bX_{n}^{\top} & 0 \end{bmatrix}$, where we consider that the datapoints are 
matrices in $\mathbb{R}^{d \times d}$, not necessarily diagonal. The implemented function is 
\begin{equation}
    \bar{f}((\bU,\bV),\bar{\bX}) = \langle \bar{\bM}_{\bU,\bV},\bar{\bX} \rangle
\end{equation}
the output of which is the same as the original model but now $\bar{\bM}_{\bU,\bV}$ is the relevant matrix to study the problem.
Assume that $\bU(0),\bV(0)$ are initialized with $\mathcal{N}(0,\sigma^{2} = \frac{\alpha^{2}}{\sqrt{k}})$. This way $\mbox{Var}\left[\mbox{diag}(\bU\bV^{\top})[i]\right] = \alpha^{2}$. In the case where the 
measurements commute, the following theorem is proven in \cite{woodworth2020kernel} (we note that the definition of the scaling parameters are different in \cite{woodworth2020kernel}, but ultimately the statements are equivalent)
 \begin{theorem}
    Let $k \to \infty$, $\alpha(k) \to 0$ and $\mu := \lim_{k \to \infty} \alpha^{4}\sqrt{k} = \sigma(k)\sqrt{k}$ and suppose that 
    $\mathbf{X}_{1},...,\mathbf{X}_{n}$ commute. If $\mathbf{M}_{\bU,\bV}(t)$ converges to a zero error solution $\bM_{\bU,\bV}^{*}$, then 
    \begin{equation}
        \bM_{\bU,\bV}^{*} = \argmin_{\mathbf{M}} Q_{\mu}(\mbox{spectrum}(\bM)) \thickspace \mbox{s.t.} \thickspace \mathbf{L}(\bM) = 0
    \end{equation}
    where $Q_{\mu}$ is the same function as before, now applied to the spectrum of $\mathbf{M}$.
 \end{theorem}
 We see that the parameter governing whether or not the function $q$ behaves like a square or an absolute value is $\mu$, which involves both 
 the scale of initialization and the width of the problem :  
\begin{itemize}
    \item if $\alpha  = o(1/k^{1/4})$, i.e. $\sigma = o(1/\sqrt{k})$, we have an $\ell_{1}$ implicit bias and rich regime, 
    \item if $\alpha = O(1/k^{1/4})$, i.e. $\sigma = O(1/\sqrt{k})$, we have an $\ell_{2}$ implicit bias and kernel regime,
    \item the scaling $\sqrt{k}\alpha^{2} \to 0$ leads to the kernel regime, even if $\norm{\boldsymbol{\beta}(0)} \simeq \alpha^{2} \to 0$
\end{itemize}
 \begin{figure}[h]
    \centering
    \includegraphics[scale=0.4]{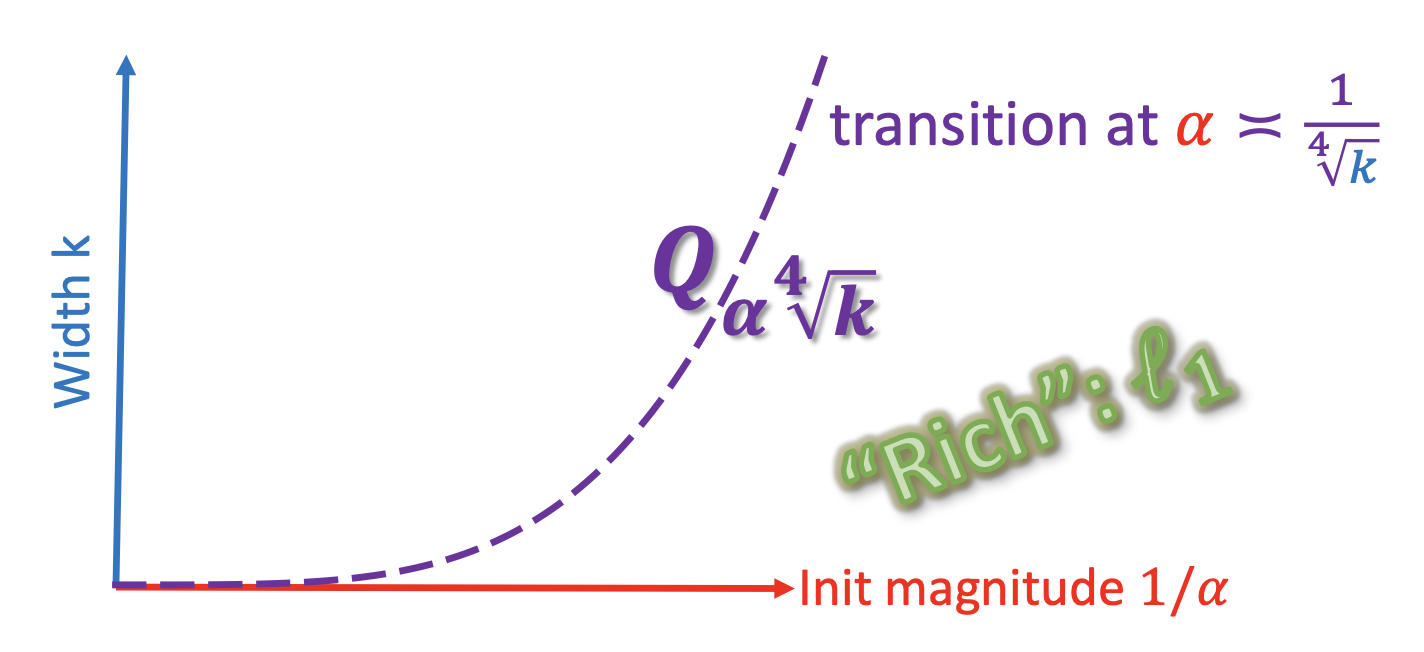}
    \caption{Rich and kernel regime in the matrix factorization problem}
\end{figure}

\subsection{Deep diagonal networks}
We now turn to the study of the effect of depth on the learning regime. To do so, we consider a deep variant of the 
model introduced above, a depth $D$ diagonal linear network :
\begin{equation}
    \boldsymbol{\beta}(t) = \bw_{+}(t)^{D}-\bw_{-}(t)^{D} \quad \mbox{and} \quad f_{D}(\bw,\bx) = \langle \bw_{+}(t)^{D}-\bw_{-}(t)^{D},\bx \rangle
\end{equation}
\begin{figure}[h]
    \centering
    \includegraphics[scale=0.4]{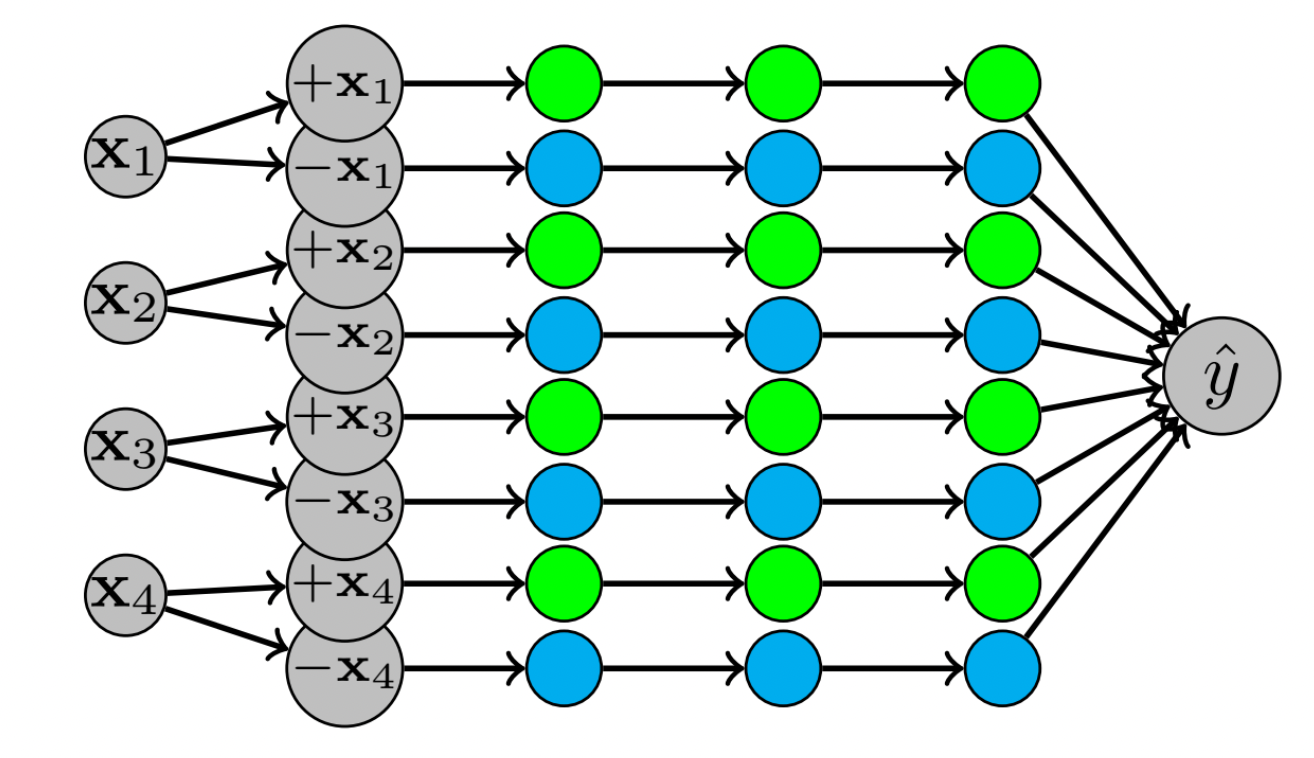}
    \caption{Deep diagonal linear network}
\end{figure}
We assume that gradient flow is initialized with $\bw_{+}(0) = \bw_{-}(0) = \alpha \mathbf{1}$, and define the residual at each time step $\mathbf{r}(t) = \bX\boldsymbol{\beta}(t)-\by$. Writing gradient flow on this model, with the square loss, reads 
\begin{align}
    \dot{\bw}_{+} = -\frac{dL}{d\bw_{+}} = -D\bX^{\top}\mathbf{r}(t)\odot \bw_{+}^{D-1} \\
    \dot{\bw}_{-} = -\frac{dL}{d\bw_{-}} = D\bX^{\top}\mathbf{r}(t)\odot \bw_{+}^{D-1} 
\end{align}
which in turn integrates to 
\begin{align}
    \bw_{+} = \left(\alpha^{2-D}+D(D-2)\bX^{\top}\int_{0}^{t}\mathbf{r}(t)dt\right)^{-\frac{1}{D-2}} \\
    \bw_{-} = -\left(\alpha^{2-D}+D(D-2)\bX^{\top}\int_{0}^{t}\mathbf{r}(t)dt\right)^{-\frac{1}{D-2}},
\end{align}
and 
\begin{align}
    \beeta(t) = \alpha^{D}\left(1+\alpha^{D-2}D(D-2)\bX^{\top}\int_{0}^{t}\mathbf{r}(t)dt\right)^{-\frac{D}{D-2}} \\
    \hspace{2cm}-\alpha^{D}\left(1+\alpha^{D-2}D(D-2)\bX^{\top}\int_{0}^{t}\mathbf{r}(t)dt\right)^{-\frac{D}{D-2}}
\end{align}
Letting $\mathbf{s} = \int_{0}^{\infty}\mathbf{r}(\tau)d\tau$ be the integrated residual and assuming a zero error solution is achieved (global convergence of the gradient method), we may write 
\begin{equation}
    \boldsymbol{\beta}(\infty) = \alpha^{D}h_{D}\left(\mathbf{X}^{\top}\mathbf{s}\right) \thickspace \mbox{and} \thickspace \bX\beeta(\infty) = \by
\end{equation}
where $h_{D} = \alpha^{D}\left(1+\alpha^{D-2}D(D-2)z\right)^{-\frac{D}{D-2}}-\alpha^{D}\left(1+\alpha^{D-2}D(D-2)z\right)^{-\frac{D}{D-2}}$ (we note that the chosen scaling is slightly different in \cite{woodworth2020kernel}, but ultimately the statements are equivalent). Recall that we are searching for an equivalent problem of the form 
\begin{align}
    \beeta^{*} \in \inf_{\beeta} Q(\beeta) \thickspace \mbox{s.t.} \thickspace \bX\beeta = \by,
\end{align}
with the corresponding Lagrangian
\begin{equation}
    \mathcal{L}(\beeta,\boldsymbol{\nu}) = Q(\beeta)+\bnu^{\top}\left(\bX\beeta-\by\right),
\end{equation}
for which the KKT optimality conditions read 
\begin{equation}
    \bX\beeta =\by \quad \mbox{and} \quad \nabla Q(\beeta^{*}) = \bX^{\top}\boldsymbol{\nu}.
\end{equation}
We can then match $\mathbf{s}$ with $\bnu$ and identify the potential $Q$ by matching its gradient with the inverse of $h_{D}$. To do so, define $q_{D} = \int h^{-1}_{D}$
and $Q_{D}(\beeta) = \sum_{i}q_{D}\left(\frac{\beta[i]}{\alpha^{D}}\right)$. It is proven in \cite{woodworth2020kernel} that 
\begin{equation}
    \forall t \quad \norm{\bX^{\top}\int_{0}^{t}r(\tau)d\tau}_{\infty} \leqslant \frac{\alpha^{2-D}}{D(D-2)}
\end{equation}
so the domain of $h_{D}$ is the interval $[-1,1]$ upon which it is monotonically increasing, ensuring the existence of the inverse mapping $h^{-1}_{D}$. Then, for all depth 
$D\geqslant 2$, this equivalent cost induces a rich implicit bias for $\alpha \to 0$, i.e. 
\begin{align}
    \lim_{\alpha \to 0} \beeta^{\infty}_{\alpha,D} = \beeta_{\ell_{1}}^{*} \\
    \lim_{\alpha \to \infty} \beeta^{\infty}_{\alpha,D} = \beeta_{\ell_{2}}^{*}
\end{align}
Although the same behaviour is observed as for the $D=2$ case, there are actually two main differences. The first one 
is that, for $D>2$, explicit regularization does not lead to a sparse bias. Indeed
\begin{equation}
    R_{\alpha}(\beeta) = \min_{\beeta = \bw_{+}^{D}-\bw_{-}^{D}} \norm{\bw-\alpha\mathbf{1}}_{2}^{2}
\end{equation}
leads to 
\begin{equation}
    R_{\alpha}(\beeta) \xrightarrow[]{\alpha \to 0} \norm{\beeta}_{2/D},
\end{equation}
i.e. the $2/D$ quasi-norm, which leads to less sparse solution than the $\ell_{1}$ norm for $D=2$. The second difference concerns the 
intermediate regime, meaning how fast does the scaling at initialization go to zero for the sparsity inducing bias to kick in. We have 
seen above that, for $D=2$, an exponentially small scale in $\alpha$ is required to enter the rich regime. As soon as $D=3$ however, 
only a polynomially decreasing scale in $\alpha$ is required, and the deeper the network the faster we can reach the rich regime when decreasing $\alpha$. This is 
illustrated by plotting the shape of $q_{D}$ for different values of $D$
\begin{figure}[h]
    \centering
    \includegraphics[scale=0.5]{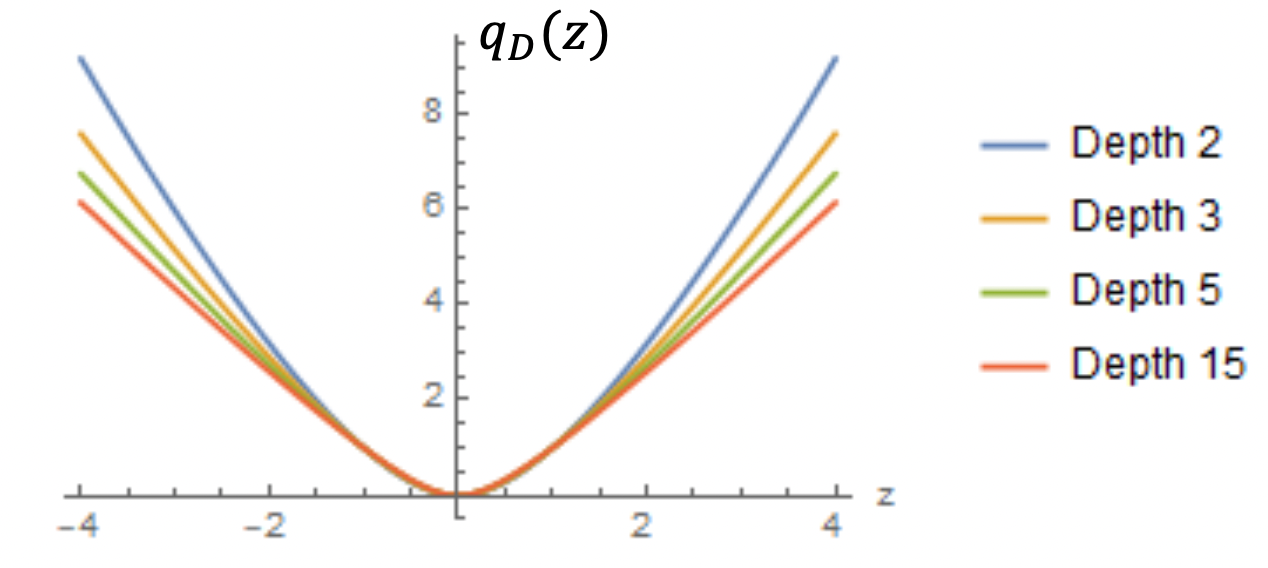}
    \caption{Implicit cost function for deeper networks}
\end{figure}
\subsection{Beyond linear models}
Recall our setup, minimizing a loss function defined over a dataset with gradient descent. The predictor function $h_{\mathbf{w}}(\mathbf{x}) = f(\mathbf{w},\mathbf{x})$ is parametrized 
by $\mathbf{w} \in \mathbb{R}^{p}$. So far we have focused on linear models taking the form 
\begin{equation}
    h_{\mathbf{w}}(\mathbf{x}) = \langle \boldsymbol{\beta}_{\mathbf{w}},\mathbf{x} \rangle
\end{equation}
where $\boldsymbol{\beta}_{\mathbf{w}} = F(\mathbf{w}) \in \mathbb{R}^{d}$.
Now consider the generic case where no linearity assumption is made on the predictor. The function 
$F$ is now defined as a mapping from $\mathbb{R}^{d}$ to $\mathbb{R}^{\mathcal{X}}$. We may write the dynamics 
on $h_{\mathbf{w}}(\mathbf{x})$ in similar fashion as before using functional derivatives :
\begin{equation}
    \dot{h}_{\mathbf{w}}(\mathbf{x}) = - \nabla F^{\top} \nabla F \nabla_{h}L(h)
\end{equation}
where $\nabla_{h}L(h)$ is now an element of $\mathbb{R}^{\mathcal{X}}$ and $\nabla F^{\top} \nabla F$ is a linear map 
from $\mathbb{R}^{\mathcal{X}}$ to $\mathbb{R}^{\mathcal{X}}$, with a kernel taking the form 
\begin{equation}
    \rho^{-1}(\mathbf{x},\mathbf{x}') = \langle \nabla_{\mathbf{w}}f(\mathbf{w},\mathbf{x}), \nabla_{\mathbf{w}}f(\mathbf{w},\mathbf{x}') \rangle
\end{equation}
Thus the dynamics in parameter space is a gradient flow according to the metric tensor defined by the tangent kernel at each time step. In the kernel regime, 
this metric tensor is fixed and remains the same as the one at initialization throughout the dynamics, whereas in the generic case it changes at each time step.
\newpage

\section{Lecture 6: Implicit bias with linear functionals and the logistic loss}

Recall that we are trying to understand how the choice of optimization geometry (i.e., what our preferred metric is in local updates) affects where the optimization will lead us. In the previous lecture, we also spoke about geometry of parameters space (usually just Euclidean geometry), but really what mattered was the geometry in function space. Recall that the relationship between parameter and function space is:
\begin{align}
&h_w(x) = f(w, x)\, \text{ or } h_w(x) = \langle \beta_w, x \rangle\, \notag \\
&\text{where } \beta_w = F(w)\,.
\end{align}
We can describe what the geometry of our model looks like in function space based on the Jacobian of the model $F$\,. Now we can do gradient flow with this inverse metric tensor:
$$
\dot{\beta} = - (\nabla F^T \nabla F) \nabla_\beta L_S(\beta)\,.
$$

We also discussed how this relates to explicit $\ell_2$ regularization and showed a setting where it was quite different. The metric tensor really depends on where we are in parameter space, which isn't conducive to studying the dynamics on the function, but in some cases, this can be circumvented.

% ``I'm lecturing, you're gonna stick with $p$ as the number of parameters."

We can write down the entries of $\rho^{-1}\,:$
$$
\rho^{-1}(x, x') = \langle \nabla_w f(w, x), \nabla_w f(w, x') \rangle\,.
$$
This is the tangent kernel at the position $w\,.$ We are conditioning the dynamics on the position $w$ at the given time. In the kernel regime, the location at which the Jacobian is evaluated doesn't change significantly and so the same kernel matrix governs the dynamics at every step / at all times. 

The simplest example in which we can see non-trivial behavior is this squared parameterization model: $f(w, x) = \langle \beta_w, x \rangle$ with $\beta_w = F(w) = w_+^2 - w_-^2\,,$ or in the deeper case: $\beta_w = F(w) = |w_+|^D - |w_-|^D\,.$ We talked about initializing at $w(0) = \alpha \vec 1\,,$ which gives $\beta(0) = 0\,.$ We saw that for $D=2\,,$ for large $\alpha\,,$ we got the kernel regime, and when we take $\alpha$ to 0, we get this $\ell_1$ regularization. Even when $D \ge 2\,,$ when we took $\alpha \rightarrow 0\,,$ we still get $\ell_1$ regularization, despite explicit $\ell_2$ norm regularization on the parameters not giving $\ell_1$ norm regularization in function space.

In this lecture, we will move away from the least-squares setting and look at logistic regression. We will see that there is an implicit bias of optimization in the classification setting, as well.

% We now turn to the study of the implicit bias of gradient descent when moving away from the square loss.
\subsection{Problem setting and equivalent reformulation}
Consider a binary classification problem 
where we minimize the logistic loss over a data set using gradient descent 
\begin{equation}
    L_{S}(\mathbf{w}) = \frac{1}{n}\sum_{i=1}^{n}\mbox{loss}\left(\langle \mathbf{w},\mathbf{x}_{i} \rangle, y_{i}\right)
\end{equation}
where $\mathbf{w} \in \mathbb{R}^{d}$, the $y_{i}$ are binary and $\mbox{loss}(\hat{y}_{i},y_{i}) = \mbox{log}(1+\mbox{exp}(-\hat{y}_{i}y_{i}))$, and we minimize this loss with 
\begin{equation}
    \mathbf{w}_{k+1} = \mathbf{w}_{k}-\eta\nabla L_{S}(\mathbf{w}_{k})
\end{equation}
In the overparametrized case, i.e. $d>n$, the data is separable and we may minimize the loss by considering any separating hyperplane and 
taking its norm to infinity (the loss does not have a finite minimizer in this case). Since the optimal solution diverges, we focus on the direction of the optimal solution :
\begin{equation}
    \lim_{k \to \infty} \frac{\mathbf{w}_{k}}{\norm{\mathbf{w}_{k}}_{2}}
\end{equation}
which converges to the max margin separator, i.e. the furthest away (in Euclidian distance) to all the points. This can be 
equivalently rewritten as a convex opitmization problem under inequality constraints 
\begin{align}
    \bw^{*} \in \argmin_{\bw \in \mathbb{R}^{d}} \norm{\mathbf{w}}_{2} \\
    \mbox{s.t.} \quad y_{i}\langle \bw, \bx_{i} \rangle \geqslant 1
\end{align}
The Lagrangian, denoted $\mathcal{L}$, for this problem reads 
\begin{equation}
    \mathcal{L}(\bw,\bnu) = \frac{1}{2}\norm{\bw}_{2}^{2}+\sum_{i=1}^{n}\nu_{i}\left(1-y_{i}\langle \bw_{i},\bx_{i} \rangle \right)
\end{equation}
where $\bnu \succeq 0$. Primal feasibility then requires 
\begin{equation}
    y_{i}\langle \bw_{i},\bx_{i} \rangle \geqslant 1 \quad \mbox{and} \quad \nabla_{\bw} \mathcal{L} = 0
\end{equation}
which implies $\mathbf{w} = \sum_{i}\nu_{i}\bx_{i}$, meaning the separating hyperplane is supported by the data vectors. Complementary slackness then 
indicates that the only active coefficients $\nu_{i}$ that are non-zero are those associated with datapoints verifying $y_{i}\langle \bw,\bx_{i} \rangle = 1$, i.e. where the constraint is active. For any $\mathbf{x}_{i}$ verifying $y_{i}\langle \bw_{i},\bx_{i} \rangle > 1$, the corresponding Lagrange multiplier $\nu_{i}$ will be zero.
\subsection{Gradient flow dynamics}
What does GF look like on this problem? Since we are interested in the interpolating regime, assume that the dynamics converge to 
a small error $\epsilon$. In this regime, we may approximate the logistic loss by its right hand side tail, i.e. 
\begin{equation}
    \forall i \quad \mbox{log}(1+\mbox{exp}(-\hat{y}_{i}y_{i})) \simeq \mbox{exp}^{-y_{i}\hat{y}_{i}}
\end{equation}
The gradient may then be approximated by
\begin{equation}
    -\nabla_{\bw}L_{S}(\bw) = \frac{1}{n}\sum_{i=1}^{n}e^{-\langle \bw_{i},\bx_{i} \rangle}\bx_{i}.
\end{equation}
The formal proof of what follows can be found in \cite{lyu2019gradient,moroshko2020implicit}. Intuitively, 
GF finds a separating direction that will not change much after a certain number of iteration, and increase its norm. We may then write 
\begin{equation}
    \bw_{k} = \bw_{\infty}g(k)+\rho(k)
\end{equation}
where $\rho(k) = o(1)$ (a Theorem from [Soudry Hoffer Srebro 18] actually shows that $g(k) = log(k)$). Replacing in the expression of the gradient leads to 
\begin{equation}
    \nabla_{\bw} L_{S}(\bw) = \frac{1}{n}\sum_{i=1}^{n}e^{-\langle \bw,\bx_{i} \rangle} \bx_{i} \simeq \frac{1}{n} \sum_{i=1}^{n}e^{g(k)\langle \bw_{\infty},\bx_{i} \rangle -O(1)}\bx_{i},
\end{equation}
which is a linear combination of the data points and gives an explicit expression for the $\nu_{i}$ coefficients. Now denote 
by $\gamma$ the margin $\gamma = \min_{i} \langle \bw_{\infty},\bx_{i} \rangle >0$, and define the normalized separator 
\begin{equation}
    \hat{\bw}_{\infty} = \frac{\bw_{\infty}}{\gamma}
\end{equation}
whose norm will remain finite.
Primal feasibility is verified by construction $\langle \bw, \bx_{i} \rangle \geqslant 1$ is satisfied by construction, and we have seen 
that the zero gradient condition on $\bw$ prescribing it as a linear combination of the data points is also verified. We also see that for large
values of $\langle \bw_{\infty}, \bx_{i} \rangle$ the corresponding coefficient $\nu_{i}$ will decrease very fast as $g(k) \to \infty$, leaving the 
main contribution to the lowest values which are the $\bx_{i}$ for which $\langle \bw_{\infty}, \bx_{i} \rangle = 1$. Thus we recover the complementary slackness condition.

% \section{KAVYA (slide 27 onward), 46:32}

\subsection{Comparing the squared, logistic and exponential loss}

For squared loss, we go to minimum distance from initialization, so the initialization is still important. For the logistic loss, however, the initialization doesn't matter at all -- we always go to the max margin solution. This makes sense because given that we diverge anyway, i.e., we go infinitely far from the starting point, it cannot matter where we start. Any finite initialization from far enough away looks like the origin. This is a significant factor that steers differences between squared loss and logistic loss. We could repeat this analysis and show that when you minimize the logistic loss, for any homogenous model, you always go to the minimum $\ell_2$ norm solution. Any network with fixed depth and homogeneous activation. We can show this with basically the same proof. 

\begin{theorem}[\cite{pmlr-v97-nacson19a, lyu2019gradient}]
% (Nacson Gunasekar Lee S Soudry 2019), (Lyu Li 2019)]
If $L_S(w) \rightarrow 0$ and the step size is small enough to ensure convergence in direction, then:
$$
w_\infty \propto \text{ first order stationary point of } \arg \min \norm{w}_2 \text{ s.t. } \forall\, i \, y_i f(w, x_i) \ge 1\,.
$$
\end{theorem}
The first order stationary points of this objective are exactly those that satisfy the KKT conditions. For convex problems, this implies optimality. Here, we have to be a bit more careful. The objective is convex but $f$ is not in general a convex function of $w$. 
This suggests that the implicit bias is defined by $R_F(h) = \arg \min_{F(w) = h} \norm{w}_2\,,$ i.e., which function $h$ is representable with the smallest possible norm in parameter space. These two problems have the same global minimum but this does not mean that a first-order stationary point of the first problem is a first-order stationary point of the second problem. Relating the two remains largely open.
%\hl{55:28 after answering Boaz question}

% \hl{check with C\'{e}dric whether there is overlap here; how to structure.}

%\hl{56:40 -- Haim question}

For squared loss, as the magnitude of initialization goes to infinity, the entire trajectory converges to the kernel trajectory. In particular, then, the implicit bias is the implicit bias of the kernel (i.e., the limit points are the same). For the logistic loss, we can get a similar statement about the infinite scale leading to the kernel regime, but only for finite time. That is, if we fix the amount of time for which we optimize and then take the scale to infinity, then the kernel regime will still be observed. These results are presented formally in \cite{chizat2019lazy}. For the logistic loss, if we change the order of limits for scale of initialization and time of optimization, we reach a first order stationary point of $\arg \min \norm{w}_2$ such that the margin condition is satisfied (formal statement in \cite{lyu2019gradient} and \cite{pmlr-v97-nacson19a}).

% under some conditions , ... \hl{kernel regime}. For logistic loss, however, for any finite time, within that finite time, if I take $\alpha \rightarrow \infty\,,$ then I will remain in the kernel regime. 

Thus, for the logistic loss, the regime we are in is no longer just dependent on $\alpha$, the initialization scale,
%. At any time $T$, the training loss is $\epsilon_T$. We will parameterize in terms of $\epsilon_T\,.$ 
% If we send $\alpha$ to $\infty$ then $\epsilon_T \rightarrow 0$, we are in the kernel regime. If \hl{copy statements from slide, 59:50} 
% The transition now is no longer just based on scale 
but rather on the combination of scale and training loss, which evolves in time. For each optimization accuracy, we can ask at what scale we would enter the kernel regime. The transition occurs at $\epsilon \sim \exp{(-\alpha^2)}$. On the boundary, (under several strong assumptions), the behavior follows the $Q_\alpha$ function with parameter $\frac{\alpha}{\sqrt{\log (1/\epsilon)}}\,.$
Thus, the transition depends on the ordering of $\epsilon, \alpha$ limits. While it is true that we have an asymptotic result, it only kicks in when $\epsilon = 10^{-1700}\,,$ so we do not get $\ell_1$ regularization in practice. 

% To formalize and summarize, \hl{see slide 30} 

\begin{figure}[h]
\label{fig:logistic-qalpha-smalleps}
    \centering
    \includegraphics[scale=0.5]{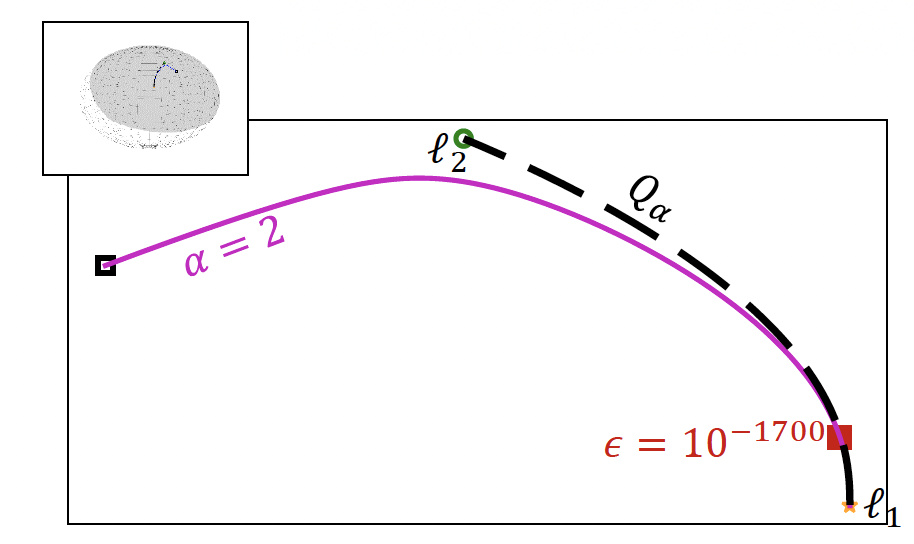}
    \caption{This figure depicts how the predictor behaves with finite but not large initialization scale ($\alpha = 2$). We know that in the long run, we will reach the minimum $\ell_1$ norm solution. When the optimization accuracy is finite, we will traverse the whole $Q_\alpha$ path. The dashed black line is the path of minimum $Q_\alpha$ margin solution: solution with margin 1 that minimizes $Q_\alpha$ for corresponding $\alpha$. As $\alpha$ increases, it seems we converge to this path. (Note that we don't know how to prove this but it seems to hold empirically.) Importantly, while it is true asymptotically in $\epsilon\,$ that minimizing the regularizer in function space corresponds to minimizing the norm of the weights, the value of $\epsilon$ at which it starts to hold is completely impractical. }
\end{figure}

% \hl{1:06:13}
When we consider deeper models, things only get worse: the asymptotic regime is even harder to reach. It is not well-understood yet what is happening here when $\alpha$ is small. Finally, the width of the network also makes a difference in effective initialization scale. 

% Becomes clearer with greater depth. Example: depth 3 takes us to min $\ell_{2/3}$ quasinorm solution. SOSP. 

\paragraph{Other Control Parameters} Other control parameters (aside from initialization scale) include how early we stop training, the shape of the initialization (i.e., relative scale of the parameters), the step size, and the stochasticity. The latter has been studied through the lens of batch size and label noise.

\subsection{Matrix Factorization Setting and Commutativity}
%\hl{same model as previous lecture; 1:13:04}
We are still going to look at a linear model in $\beta$, but now $\beta$ is going to be a matrix. This is a standard least squares objective in $\beta$. The main difference is that now, instead of factorizing $\beta$ element-wise, we are going to employ the factorization $\beta = UV^T$. This formulation captures many things: matrix completion, matrix reconstruction from linear measurements, and even multi-task learning. In what follows, 
we will denote $\hat{L}$ the empirical loss that is being minimized.
 
% \hl{copy problem from slide, and copy the rest from slides, too....}

Let us study the implicit bias of gradient flow on the factorization in function space, i.e., in matrix space:
\begin{align}
\dot{U}(t) &= - \nabla_U \hat{L}(UV^T) \\
\dot{V}(t) &= - \nabla_V \hat{L}(UV^T) \\
\Rightarrow \dot{\beta} &= - (\nabla F^T \nabla F) \nabla \hat{L}(\beta) = -(UU^T \nabla \hat{L}(\beta) + \nabla \hat{\beta}VV^T)\,.
\end{align}

Observe that this is a linear transformation of the gradient. We would be interested in writing the transformation in terms of just $\beta$, and this is not possible here, as it turns out. Instead, we introduce an augmented variable: 
$$
W = \begin{bmatrix}
U \\
V
\end{bmatrix} \quad \quad
\tilde{\beta} \coloneqq  WW^T = \begin{bmatrix}
UU^T & UV^T \\
VU^T & VV^T
 \end{bmatrix} = \begin{bmatrix}
UU^T & \beta \\
\beta^T & VV^T
 \end{bmatrix}
$$

This is still a matrix factorization problem in $W$, except it is now a symmetric matrix factorization problem. This corresponds to a minimization problem over positive semidefinite matrices. Namely:
$$
\min_{\tilde{\beta} \succeq 0} \hat{L}(\tilde{\beta}) = \norm{\tilde{\mathcal{X}}(\tilde{\beta}) - y}^2_2\,.
$$
for appropriately defined $\tilde{\mathcal{X}}$. The resulting dynamics are: $\dot{\tilde{\beta}} = - ( \tilde{\beta} \nabla \hat{L}(\tilde{\beta}) + \nabla \hat{L}(\tilde{\beta}) \tilde{\beta}^T)\,.$

What we've shown here is that whenever we have a {\em non-symmetric} matrix factorization problem, it is a special case of a higher-dimensional {\em symmetric} matrix factorization problem. In some sense, the real problem we should be looking at {\em is} this one, since the non-symmetric one hides information about the geometry.  

The local geometry of the space in which we search is given by left and right multiplication by $\tilde\beta\,.$ Now, we want to see if we can identify this as a Hessian map. Do the dynamics stay in a low-dimensional manifold?
% \hl{get algebra from slides.}

As before, where we are depends on the integral of the residual so far. If matrices commute, we can solve the differential equation explicitly with a matrix exponential. The metric tensor is a Hessian map, which directly implies that the dynamcics occur in a low-dimension manifold. In this case, the result is very robust: it doesn't depend on the residuals, nor the loss under which the residuals are computed. It is not robust to step size: a large step along the tangent space could lead to exiting the manifold now that it is curved. If the problem is non-commutative, the order in which we multiply on the left and right matters. $\beta(t)\,$ is then a ``time-ordered exponential,'' and we cannot ignore the ordering of the residuals. Even with just two data points, we can navigate the entire space instead of a low-dimension manifold. An analogy is parallel parking, where with the forward/backward and left/right controls we can somehow move straight to the right.
%\hl{parallel parking analogy}

This non-commutative case is the case we are in in general. This is a well-defined question but one for which the solution is not known.

\section{Perspectives}
% \hl{1:29:37}
We have seen three notions of implicit bias : one related to statistical guarantees of a model, benign overfitting and finally implicit bias related to optimization. The statistical aspect of implicit bias 
is well understood and falls within the framework of statistical learning theory. Leaving benign overfitting aside and focusing on optimization, we have demonstrated a general approach allowing to quantitatively study 
the implicit bias of a given descent algorithm for a chosen architecture. In particular, by identifying the potential that is implicitly minimized for a given 
problem, we are able to determine conditions to obtain an implicit bias akin to a kernel method or a model that is able to do feature selection or feature learning. While this approach is informative 
for the simple models we have considered, determining whether or not they generalize to more realistic neural network architectures remains unclear. Furthermore, it is possible that other forms of implicit bias 
exist, and that they are ultimately repsonsible for the empirical success of deep neural networks. In other words, we do not know if the implicit minimization of an effective potential is sufficient 
to explain most cases.

\quad \\
\quad \\
\quad \\
\textbf{Acknowledgements} \quad \\
\quad \\
These are notes from the lecture of Nathan Srebro given at the summer school \textquotedblleft Statistical Physics \& Machine Learning \textquotedblright ,
that took place in Les Houches School of Physics in France from 4th to 29th July 2022. The school was organized by Florent Krzakala and Lenka Zdeborov\'a from EPFL.
Recordings of the lectures are available on the school's website \url{https://leshouches2022.github.io/}.
\newpage

\bibliographystyle{plain}
\bibliography{references}

\begin{thebibliography}{10}

\bibitem{amari2012differential}
Shun-ichi Amari.
\newblock {\em Differential-geometrical methods in statistics}, volume~28.
\newblock Springer Science \& Business Media, 2012.

\bibitem{arora2009computational}
Sanjeev Arora and Boaz Barak.
\newblock {\em Computational complexity: a modern approach}.
\newblock Cambridge University Press, 2009.

\bibitem{barron1993universal}
Andrew~R Barron.
\newblock Universal approximation bounds for superpositions of a sigmoidal
  function.
\newblock {\em IEEE Transactions on Information theory}, 39(3):930--945, 1993.

\bibitem{bartlett1996valid}
Peter Bartlett.
\newblock For valid generalization the size of the weights is more important
  than the size of the network.
\newblock {\em Advances in neural information processing systems}, 9, 1996.

\bibitem{bartlett2019nearly}
Peter~L Bartlett, Nick Harvey, Christopher Liaw, and Abbas Mehrabian.
\newblock Nearly-tight vc-dimension and pseudodimension bounds for piecewise
  linear neural networks.
\newblock {\em Journal of Machine Learning Research}, 20(63):1--17, 2019.

\bibitem{bartlett2020benign}
Peter~L Bartlett, Philip~M Long, G{\'a}bor Lugosi, and Alexander Tsigler.
\newblock Benign overfitting in linear regression.
\newblock {\em Proceedings of the National Academy of Sciences},
  117(48):30063--30070, 2020.

\bibitem{bauschke2017-relativesmooth}
Heinz~H. Bauschke, J\'{e}r\^{o}me Bolte, and Marc Teboulle.
\newblock A descent lemma beyond lipschitz gradient continuity: First-order
  methods revisited and applications.
\newblock {\em Mathematics of Operations Research}, 42(2):330--348, 2017.

\bibitem{double_descent_belkin_2019}
Mikhail Belkin, Daniel Hsu, Siyuan Ma, and Soumik Mandal.
\newblock Reconciling modern machine-learning practice and the classical
  bias–variance trade-off.
\newblock In {\em Proceedings of the National Academy of Sciences}, 2019.

\bibitem{kernel_learning_belkin_2018}
Mikhail Belkin, Siyuan Ma, and Soumik Mandal.
\newblock To understand deep learning we need to understand kernel learning.
\newblock In {\em International Conference on Machine Learning}, 2018.

\bibitem{bengio2013representation}
Yoshua Bengio, Aaron Courville, and Pascal Vincent.
\newblock Representation learning: A review and new perspectives.
\newblock {\em IEEE transactions on pattern analysis and machine intelligence},
  35(8):1798--1828, 2013.

\bibitem{berg1984harmonic}
Christian Berg, Jens Peter~Reus Christensen, and Paul Ressel.
\newblock {\em Harmonic analysis on semigroups}, volume 100.
\newblock Springer-Verlag New York, 1984.

\bibitem{blum1988training}
Avrim Blum and Ronald Rivest.
\newblock Training a 3-node neural network is np-complete.
\newblock {\em Advances in neural information processing systems}, 1, 1988.

\bibitem{bousquet2003introduction}
Olivier Bousquet, St{\'e}phane Boucheron, and G{\'a}bor Lugosi.
\newblock Introduction to statistical learning theory.
\newblock In {\em Summer school on machine learning}, pages 169--207. Springer,
  2003.

\bibitem{bubeck2015convex}
S{\'e}bastien Bubeck et~al.
\newblock Convex optimization: Algorithms and complexity.
\newblock {\em Foundations and Trends{\textregistered} in Machine Learning},
  8(3-4):231--357, 2015.

\bibitem{implicit_bias_gd_wide_two_layer_2020}
Lenaic Chizat and Francis Bach.
\newblock Implicit bias of gradient descent for wide two-layer neural networks
  trained with the logistic loss.
\newblock In {\em Conference on Learning Theory}, 2020.

\bibitem{chizat2019lazy}
Lenaic Chizat, Edouard Oyallon, and Francis Bach.
\newblock On lazy training in differentiable programming.
\newblock {\em Advances in Neural Information Processing Systems}, 32, 2019.

\bibitem{cybenko1989approximation}
George Cybenko.
\newblock Approximation by superpositions of a sigmoidal function.
\newblock {\em Mathematics of control, signals and systems}, 2(4):303--314,
  1989.

\bibitem{daniely2017sgd}
Amit Daniely.
\newblock Sgd learns the conjugate kernel class of the network.
\newblock {\em Advances in Neural Information Processing Systems}, 30, 2017.

\bibitem{daniely2014average}
Amit Daniely, Nati Linial, and Shai Shalev-Shwartz.
\newblock From average case complexity to improper learning complexity.
\newblock In {\em Proceedings of the forty-sixth annual ACM symposium on Theory
  of computing}, pages 441--448, 2014.

\bibitem{furst1984parity}
Merrick Furst, James~B Saxe, and Michael Sipser.
\newblock Parity, circuits, and the polynomial-time hierarchy.
\newblock {\em Mathematical systems theory}, 17(1):13--27, 1984.

\bibitem{gaynier1995sinusoidal}
RJ~Gaynier and Tom Downs.
\newblock Sinusoidal and monotonic transfer functions: Implications for vc
  dimension.
\newblock {\em Neural networks}, 8(6):901--904, 1995.

\bibitem{implicit_bias_gd_conv_2018}
Suriya Gunasekar, Jason Lee, Daniel Soudry, and Nathan Srebro.
\newblock Implicit bias of gradient descent on linear convolutional networks.
\newblock In {\em Advances in Neural Information Processing Systems}, 2018.

\bibitem{gwbns_2017}
Suriya Gunasekar, Blake Woodworth, Srinadh Bhojanapalli, Behnam Neyshabur, and
  Nathan Srebro.
\newblock Implicit regularization in matrix factorization.
\newblock In {\em Advances in Neural Information Processing Systems}, 2017.

\bibitem{hastie2022surprises}
Trevor Hastie, Andrea Montanari, Saharon Rosset, and Ryan~J Tibshirani.
\newblock Surprises in high-dimensional ridgeless least squares interpolation.
\newblock {\em Annals of statistics}, 50(2):949, 2022.

\bibitem{hornik1989multilayer}
Kurt Hornik, Maxwell Stinchcombe, and Halbert White.
\newblock Multilayer feedforward networks are universal approximators.
\newblock {\em Neural networks}, 2(5):359--366, 1989.

\bibitem{jacot2018neural}
Arthur Jacot, Franck Gabriel, and Cl{\'e}ment Hongler.
\newblock Neural tangent kernel: Convergence and generalization in neural
  networks.
\newblock {\em Advances in neural information processing systems}, 31, 2018.

\bibitem{kearns1994cryptographic}
Michael Kearns and Leslie Valiant.
\newblock Cryptographic limitations on learning boolean formulae and finite
  automata.
\newblock {\em Journal of the ACM (JACM)}, 41(1):67--95, 1994.

\bibitem{lecun2015deep}
Yann LeCun, Yoshua Bengio, and Geoffrey Hinton.
\newblock Deep learning.
\newblock {\em nature}, 521(7553):436--444, 2015.

\bibitem{overp_matrix_sensing_li_ma_zhang_2018}
Yuanzhi Li, Tengyu Ma, and Hongyang Zhang.
\newblock Algorithmic regularization in over-parameterized matrix sensing and
  neural networks with quadratic activations.
\newblock In {\em Conference On Learning Theory}, 2018.

\bibitem{implicit_bias_gd_matrix_factorization_2021}
Zhiyuan Li, Yuping Luo, and Kaifeng Lyu.
\newblock Towards resolving the implicit bias of gradient descent for matrix
  factorization: Greedy low-rank learning.
\newblock In {\em International Conference on Learning Representations}, 2021.

\bibitem{lu2018-relativesmooth}
Haihao Lu, Robert~M. Freund, and Yurii Nesterov.
\newblock Relatively smooth convex optimization by first-order methods, and
  applications.
\newblock {\em SIAM Journal on Optimization}, 28(1):333--354, 2018.

\bibitem{lyu2019gradient}
Kaifeng Lyu and Jian Li.
\newblock Gradient descent maximizes the margin of homogeneous neural networks.
\newblock {\em arXiv preprint arXiv:1906.05890}, 2019.

\bibitem{mcculloch1943logical}
Warren~S McCulloch and Walter Pitts.
\newblock A logical calculus of the ideas immanent in nervous activity.
\newblock {\em The bulletin of mathematical biophysics}, 5:115--133, 1943.

\bibitem{mohri2018foundations}
Mehryar Mohri, Afshin Rostamizadeh, and Ameet Talwalkar.
\newblock {\em Foundations of machine learning}.
\newblock MIT press, 2018.

\bibitem{moroshko2020implicit}
Edward Moroshko, Blake~E Woodworth, Suriya Gunasekar, Jason~D Lee, Nati Srebro,
  and Daniel Soudry.
\newblock Implicit bias in deep linear classification: Initialization scale vs
  training accuracy.
\newblock {\em Advances in neural information processing systems},
  33:22182--22193, 2020.

\bibitem{pmlr-v97-nacson19a}
Mor~Shpigel Nacson, Suriya Gunasekar, Jason Lee, Nathan Srebro, and Daniel
  Soudry.
\newblock Lexicographic and depth-sensitive margins in homogeneous and
  non-homogeneous deep models.
\newblock In Kamalika Chaudhuri and Ruslan Salakhutdinov, editors, {\em
  Proceedings of the 36th International Conference on Machine Learning},
  volume~97 of {\em Proceedings of Machine Learning Research}, pages
  4683--4692. PMLR, 09--15 Jun 2019.

\bibitem{nemirovsky1979problem}
AS~Nemirovsky and DB~Yudin.
\newblock Problem complexity and optimization method efficiency.
\newblock {\em M.: Nauka}, 1979.

\bibitem{nesterov2018lectures}
Yurii Nesterov et~al.
\newblock {\em Lectures on convex optimization}, volume 137.
\newblock Springer, 2018.

\bibitem{path_sgd_2015}
Behnam Neyshabur, Ruslan Salakhutdinov, and Nathan Srebro.
\newblock Path-sgd: Path-normalized optimization in deep neural networks.
\newblock In {\em Advances in Neural Information Processing Systems}, 2015.

\bibitem{function_space_bounded_norm_multivariate_2020}
Greg Ongie, Rebecca Willett, Daniel Soudry, and Nathan Srebro.
\newblock A function space view of bounded norm infinite width relu nets: The
  multivariate case.
\newblock In {\em International Conference on Learning Representations}, 2020.

\bibitem{pinkus1999approximation}
Allan Pinkus.
\newblock Approximation theory of the mlp model in neural networks.
\newblock {\em Acta numerica}, 8:143--195, 1999.

\bibitem{rahimi2007random}
Ali Rahimi and Benjamin Recht.
\newblock Random features for large-scale kernel machines.
\newblock {\em Advances in neural information processing systems}, 20, 2007.

\bibitem{rudi2017generalization}
Alessandro Rudi and Lorenzo Rosasco.
\newblock Generalization properties of learning with random features.
\newblock {\em Advances in neural information processing systems}, 30, 2017.

\bibitem{inf_width_in_function_space_2019}
Pedro Savarese, Itay Evron, Daniel Soudry, and Nathan Srebro.
\newblock How do infinite width bounded norm networks look in function space?
\newblock In {\em Conference On Learning Theory}, 2019.

\bibitem{shalev2014understanding}
Shai Shalev-Shwartz and Shai Ben-David.
\newblock {\em Understanding machine learning: From theory to algorithms}.
\newblock Cambridge university press, 2014.

\bibitem{shalev2010learnability}
Shai Shalev-Shwartz, Ohad Shamir, Nathan Srebro, and Karthik Sridharan.
\newblock Learnability, stability and uniform convergence.
\newblock {\em The Journal of Machine Learning Research}, 11:2635--2670, 2010.

\bibitem{soudry2018implicit}
Daniel Soudry, Elad Hoffer, Mor~Shpigel Nacson, Suriya Gunasekar, and Nathan
  Srebro.
\newblock The implicit bias of gradient descent on separable data.
\newblock {\em Journal of Machine Learning Research}, 19(70):1--57, 2018.

\bibitem{turing1939systems}
Alan~Mathison Turing.
\newblock Systems of logic based on ordinals.
\newblock {\em Proceedings of the London Mathematical Society, Series 2},
  45:161--228, 1939.

\bibitem{vapnik1999nature}
Vladimir Vapnik.
\newblock {\em The nature of statistical learning theory}.
\newblock Springer science \& business media, 1999.

\bibitem{vapnik1974theory}
Vladimir Vapnik and Alexey Chervonenkis.
\newblock Theory of pattern recognition, 1974.

\bibitem{marginal_value_adaptive_grad_2017}
Ashia~C. Wilson, Rebecca Roelofs, Mitchell Stern, Nathan Srebro, and Benjamin
  Recht.
\newblock The marginal value of adaptive gradient methods in machine learning.
\newblock In {\em Advances in Neural Information Processing Systems}, 2017.

\bibitem{wolpert1997no}
David~H Wolpert and William~G Macready.
\newblock No free lunch theorems for optimization.
\newblock {\em IEEE transactions on evolutionary computation}, 1(1):67--82,
  1997.

\bibitem{woodworth2020kernel}
Blake Woodworth, Suriya Gunasekar, Jason~D Lee, Edward Moroshko, Pedro
  Savarese, Itay Golan, Daniel Soudry, and Nathan Srebro.
\newblock Kernel and rich regimes in overparametrized models.
\newblock In {\em Conference on Learning Theory}, pages 3635--3673. PMLR, 2020.

\bibitem{zhang2021understanding}
Chiyuan Zhang, Samy Bengio, Moritz Hardt, Benjamin Recht, and Oriol Vinyals.
\newblock Understanding deep learning (still) requires rethinking
  generalization.
\newblock {\em Communications of the ACM}, 64(3):107--115, 2021.

\end{thebibliography}
%
% Uncomment for keywords
%\vspace{2pc}
%\noindent{\it Keywords}: XXXXXX, YYYYYYYY, ZZZZZZZZZ
%
% Uncomment for Submitted to journal title message
%\submitto{\JPA}
%
% Uncomment if a separate title page is required
%\maketitle
% 
% For two-column output uncomment the next line and choose [10pt] rather than [12pt] in the \documentclass declaration
%\ioptwocol
%

\end{document}